\documentclass[table]{ai2style/ai2}

\usepackage{microtype}
\usepackage{amsmath, amssymb, amsfonts}
\usepackage{graphicx}
\usepackage{booktabs, tabularx, longtable}
\usepackage{multirow}
\usepackage{makecell}
\usepackage{wrapfig}
\usepackage{float}
\usepackage{siunitx}
\usepackage{csquotes}
\usepackage{xspace}
\usepackage{adjustbox}
\usepackage{threeparttable}

\usepackage{tikz}
\usetikzlibrary{arrows.meta,positioning,calc,backgrounds,shadows.blur}
\usepackage{pgfplots}
\pgfplotsset{compat=1.18}
\usepgfplotslibrary{groupplots}

\usepackage{listings}
\usepackage[most]{tcolorbox}

\usepackage{url}

\usepackage[absolute]{textpos}

\definecolor{darkblue}{rgb}{0,0,0.5}
\hypersetup{
  citecolor=darkblue,
  linkcolor=darkblue,
  urlcolor=darkblue
}

\definecolor{ai2pink}{HTML}{f0529c}
\definecolor{ai2lightpink}{HTML}{fbecf3}
\definecolor{olmoDarkBlue}{HTML}{012e59}
\definecolor{olmoBlue}{HTML}{265ed4}
\definecolor{olmoTeal}{HTML}{00d5ff}
\definecolor{olmoOrange}{HTML}{ff9100}
\definecolor{lightgrey}{HTML}{fafcfc}
\definecolor{midgrey}{HTML}{e6eded}

\DeclareUnicodeCharacter{2212}{\ensuremath{-}}
\newcommand{\norm}[1]{\left\lVert#1\right\rVert}

\newtcolorbox{prompt}[1]{colback=gray!5,colframe=ai2pink!80,fonttitle=\bfseries,title={#1}}

\definecolor{LearnerMain}{RGB}{30,102,200}
\definecolor{LearnerLite}{RGB}{120,175,255}
\definecolor{ActorMain}{RGB}{200,57,43}
\definecolor{ActorLite}{RGB}{244,170,160}
\definecolor{QueueMain}{RGB}{34,121,60}

\tikzset{
  every node/.style={font=\sffamily},
  panel/.style={
    rounded corners=2mm,
    minimum width=58mm, minimum height=36mm,
    inner sep=6mm, align=center, text=white,
    blur shadow={shadow blur steps=4, shadow opacity=.45}
  },
  title/.style={font=\bfseries\Large, text=white},
  flow/.style={-{Latex[length=3mm,width=2.3mm]}, line width=.9pt},
  lbl/.style={font=\normalsize, fill=white, fill opacity=.85, text opacity=1, inner sep=1pt},
  flowlbl/.style={
    midway, sloped,
    inner sep=1.5pt,
    font=\sffamily\Large,
    fill=white, fill opacity=.98, text opacity=1,
  }
}

\usepackage{nicematrix}
\newcolumntype{L}[1]{>{\raggedright\let\newline\\\arraybackslash\hspace{0pt}}m{#1}}
\newcolumntype{C}[1]{>{\centering\let\newline\\\arraybackslash\hspace{0pt}}m{#1}}
\newcolumntype{R}[1]{>{\raggedleft\let\newline\\\arraybackslash\hspace{0pt}}m{#1}}
\newcolumntype{P}[1]{>{\centering\let\newline\\\arraybackslash\columncolor{ai2lightpink}}m{#1}}
\newcolumntype{Q}[1]{>{\centering\let\newline\\\arraybackslash}m{#1}}
\newcolumntype{H}{>{\setbox0=\hbox\bgroup}c<{\egroup}@{}}

\newcommand{\huggingface}{\raisebox{-1.5pt}{\includegraphics[height=1.05em]{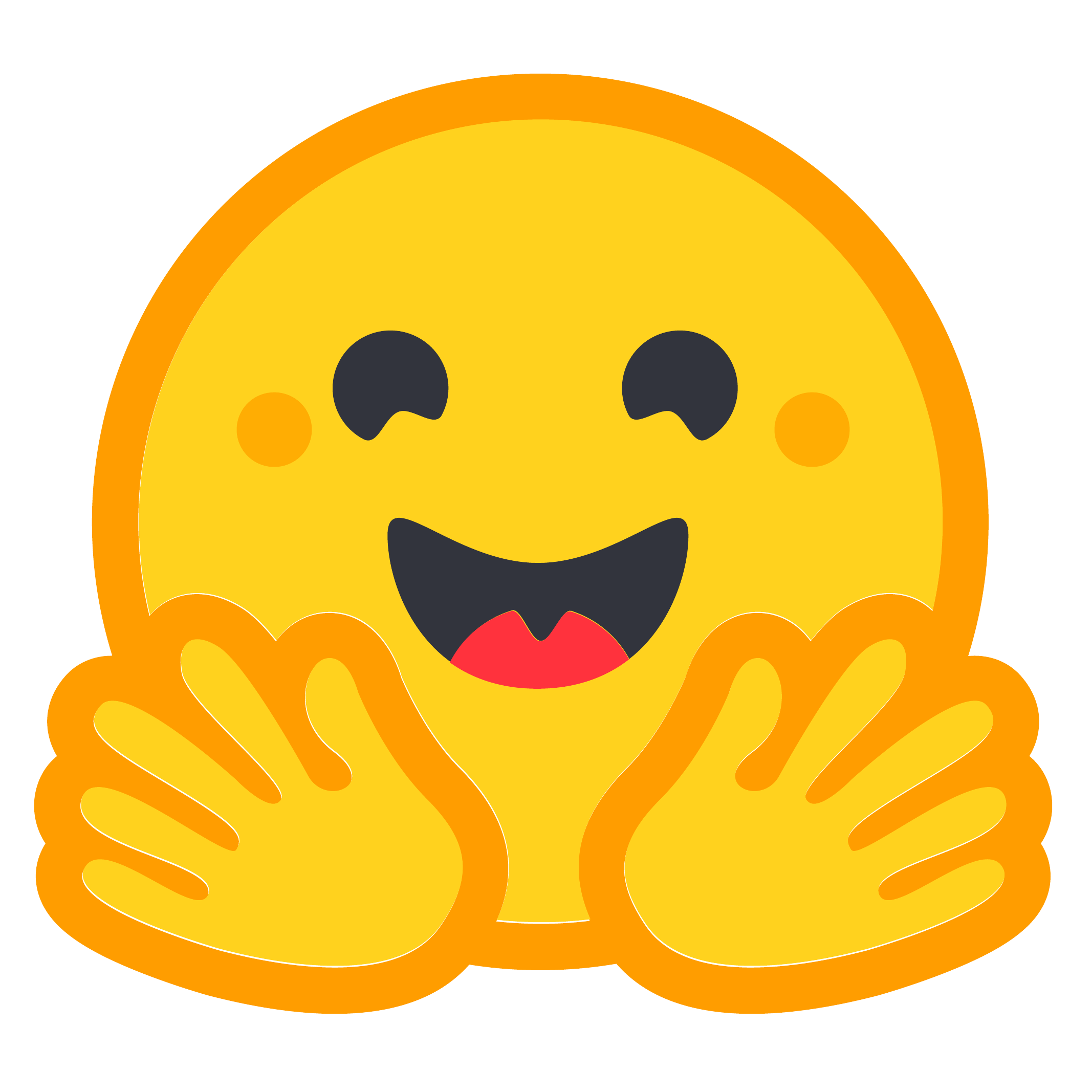}}\xspace}

\newcommand{\emailLogo}{\raisebox{-1.5pt}{\includegraphics[height=1.05em]{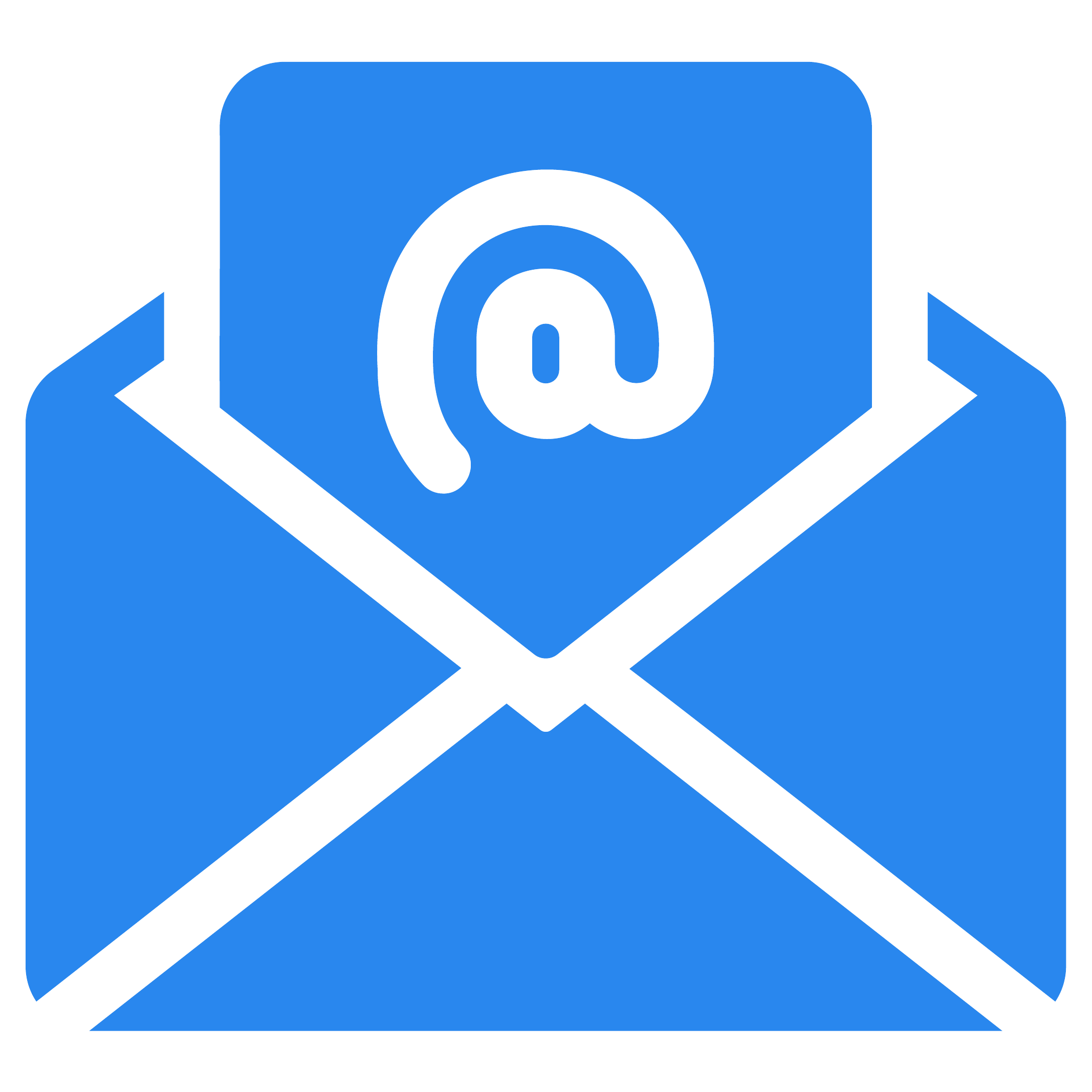}}\xspace}
\newcommand{\github}{\raisebox{-1.5pt}{\includegraphics[height=1.05em]{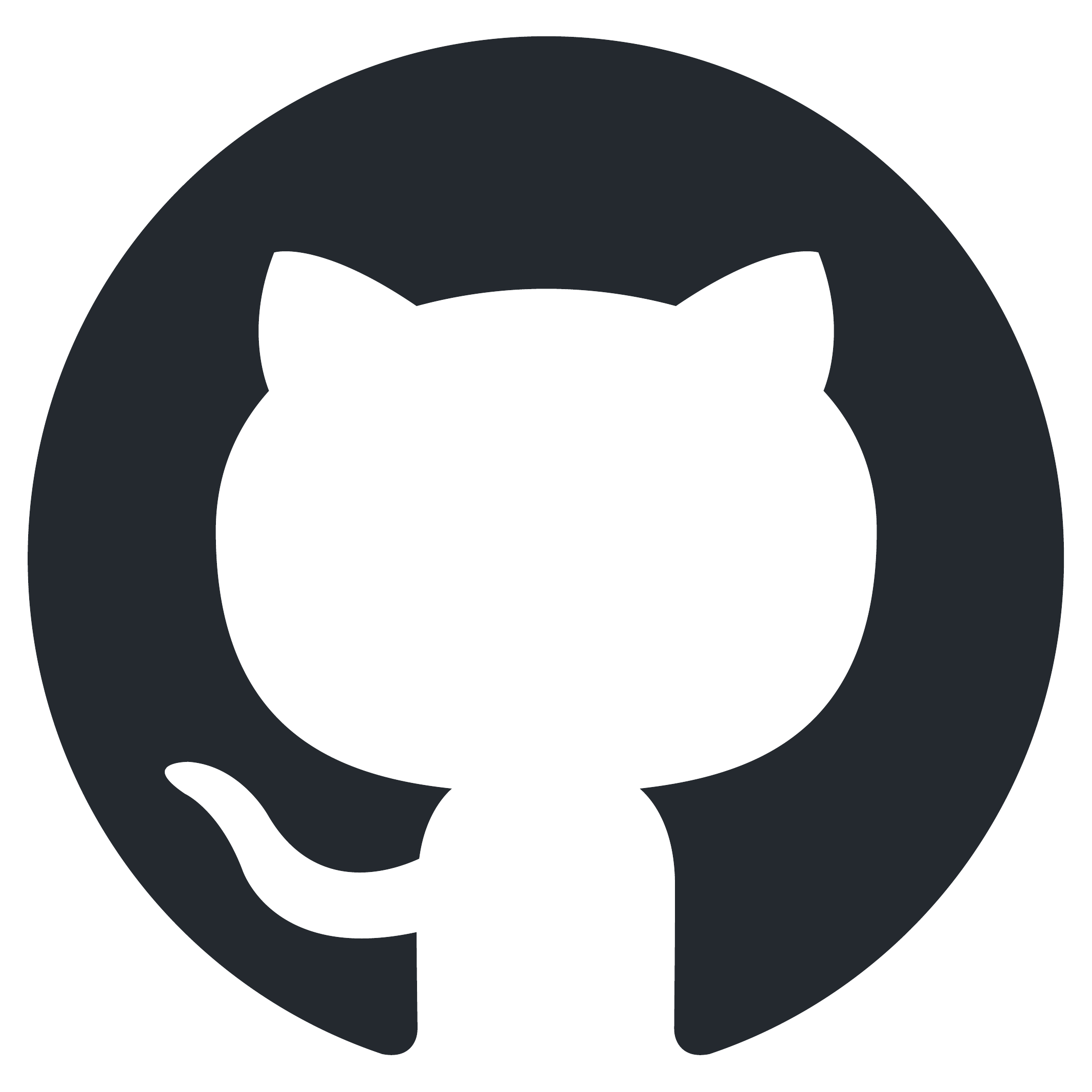}}\xspace}
\newcommand{\wandb}{\raisebox{-1.5pt}{\includegraphics[height=1.05em]{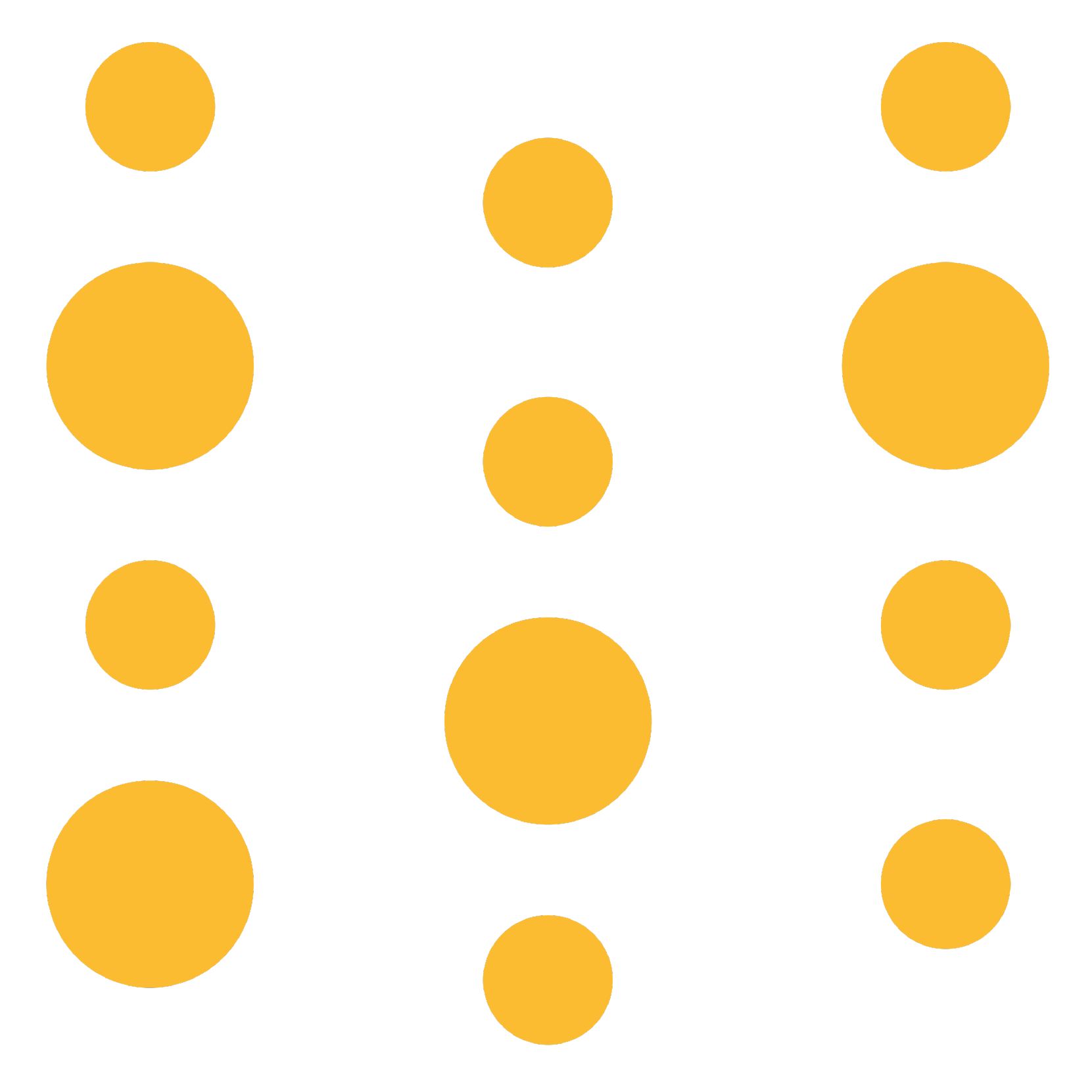}}\xspace}

\newcommand{\allenAiAff}{\raisebox{.28em}{\hspace{.02em}\scalebox{0.7}{\textbf{1}}}}
\newcommand{\uwAff}{\raisebox{.28em}{\hspace{.02em}\scalebox{0.7}{\textbf{2}}}}
\newcommand{\lambdaAff}{\raisebox{.28em}{\hspace{.02em}\scalebox{0.7}{\textbf{3}}}}
\newcommand{\ethAff}{\raisebox{.28em}{\hspace{.02em}\scalebox{0.7}{\textbf{4}}}}
\newcommand{\cambridgeAff}{\raisebox{.28em}{\hspace{.02em}\scalebox{0.7}{\textbf{5}}}}
\newcommand{\commaAff}{\raisebox{.28em}{\hspace{.02em}\scalebox{0.7}{\textbf{,}\hspace{0.1em}}}}

\newcommand{\coreContrib}{\raisebox{.28em}{\hspace{.05em}\includegraphics[height=.45em]{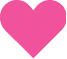}}\hspace{0.1em}}

\usepackage{amsmath}
\usepackage{amssymb}
\usepackage{amsthm}
\usepackage{thm-restate}
\usepackage{mathtools}
\usepackage{bbm}
\usepackage{paralist}
\usepackage{wrapfig}
\usepackage{cancel}

\newcommand{\olmothreeeval}{\textsc{OlmoBaseEval}\xspace}

\newcommand\model{Olmo Hybrid\xspace}

\definecolor{midgrey}{HTML}{e6eded}
\definecolor{ai2midwhite}{HTML}{f2e5d9}
\definecolor{ai2offwhite}{HTML}{fbf4ee}
\definecolor{lightgrey}{HTML}{fafcfc}

\definecolor{transformercolor}{HTML}{105257}
\definecolor{swacolor}{RGB}{190,215,235}
\definecolor{deltanetcolor}{HTML}{5ec9e1}
\definecolor{mambacolor}{RGB}{130,185,130}
\definecolor{hybridcolor}{HTML}{de5d9a}

\DeclarePairedDelimiter\abs{\lvert}{\rvert}
\DeclarePairedDelimiter\ceil{\lceil}{\rceil}
\DeclarePairedDelimiter\floor{\lfloor}{\rfloor}

\makeatletter
\let\oldabs\abs
\def\abs{\@ifstar{\oldabs}{\oldabs*}}
\let\oldceil\ceil
\def\ceil{\@ifstar{\oldceil}{\oldceil*}}
\let\oldfloor\floor
\def\floor{\@ifstar{\oldfloor}{\oldfloor*}}
\makeatother

\newcommand\NC{\mathsf{NC}}

\newcommand\PNC{\mathsf{PNC}}
\newcommand\AC{\mathsf{AC}}
\newcommand\TC{\mathsf{TC}}

\newcommand\NL{\mathsf{NL}}
\renewcommand\P{\mathsf{P}}
\newcommand\NP{\mathsf{NP}}

\usepackage{tikz}

\newtheoremstyle{ai2thm}%
  {6pt}{6pt}
  {\itshape}
  {}
  {\slightlysmaller\sffamily\bfseries}
  {}
  {0.5em}
  {%
    \thmname{#1}\thmnumber{ #2}\thmnote{: \sffamily\bfseries#3 \\}%
  }

\newtheoremstyle{ai2def}%
  {6pt}{6pt}
  {\normalfont}
  {}
  {\slightlysmaller\sffamily\bfseries}
  {}
  {0.5em}
  {%
    \thmname{#1}\thmnumber{ #2}\thmnote{: \sffamily\bfseries#3 \\}%
  }

\theoremstyle{ai2thm}
\newtheorem{theorem}{Theorem}
\newtheorem{lemma}{Lemma}

\newtheorem{corollary}{Corollary}[theorem]
\newtheorem{lemcorollary}{Corollary}[lemma]
\theoremstyle{ai2def}
\newtheorem{definition}{Definition}
\newtheorem{assumption}{Assumption}

\Crefname{corollary}{Corollary}{Corollaries}
\crefname{corollary}{corollary}{corollaries}
\Crefname{lemcorollary}{Corollary}{Corollaries}
\crefname{lemcorollary}{corollary}{corollaries}
\Crefname{assumption}{Assumption}{Assumptions}
\crefname{assumption}{assumption}{assumptions}

\tcbset{
    hybridBoxStyle/.style={
        colback=ai2lightpink,
        colframe=black
    }
}

\let\oldtheorem\theorem
\let\endoldtheorem\endtheorem
\renewenvironment{theorem}[1][]{%
  \begin{tcolorbox}[hybridBoxStyle]
  \oldtheorem[#1]
}{%
  \endoldtheorem
  \end{tcolorbox}
}

\let\olddefinition\definition
\let\endolddefinition\enddefinition
\renewenvironment{definition}[1][]{%
  \begin{tcolorbox}[hybridBoxStyle]
  \olddefinition[#1]
}{%
  \endolddefinition
  \end{tcolorbox}
}

\let\oldassumption\assumption
\let\endoldassumption\endassumption
\renewenvironment{assumption}[1][]{%
  \begin{tcolorbox}[hybridBoxStyle]
  \oldassumption[#1]
}{%
  \endoldassumption
  \end{tcolorbox}
}

\let\oldcorollary\corollary
\let\endoldcorollary\endcorollary
\renewenvironment{corollary}[1][]{%
  \begin{tcolorbox}[hybridBoxStyle]
  \oldcorollary[#1]
}{%
  \endoldcorollary
  \end{tcolorbox}
}

\newcommand\FO{\mathsf{FO}}

\newcommand{\trueloss}{\tilde{L}}

\usepackage{listings}
\usepackage{caption}
\lstdefinestyle{pythonstyle}{
    language=Python,
    basicstyle=\ttfamily\small,
    keywordstyle=\color{blue}\bfseries,
    stringstyle=\color{teal},
    commentstyle=\color{gray}\itshape,
    backgroundcolor=\color{gray!10},
    frame=single,
    framerule=0.5pt,
    rulecolor=\color{gray!50},
    showstringspaces=false,
    breaklines=true,
    postbreak=\mbox{\textcolor{red}{$\hookrightarrow$}\space},
}
\newcommand\WM[1]{\textsf{\textcolor{orange}{[WM: #1]}}}
\newcommand\AS[1]{\textsf{\textcolor{brown}{[AS: #1]}}}
\newcommand\anej[1]{\textsf{\textcolor{lime!80!black}{[Anej: #1]}}}
\newcommand\hanna[1]{\textsf{\textcolor{red}{[hanna: #1]}}}
\newcommand\nato[1]{\textsf{\textcolor{purple}{[nathan: #1]}}}
\newcommand\nascomment[1]{\textsf{\textcolor{blue}{[Noah: #1]}}}
\newcommand\david[1]{\textsf{\textcolor{magenta}{[David: #1]}}}
\newcommand\kyle[1]{\textsf{\textcolor{green}{[Kyle: #1]}}}
\newcommand\yanhong[1]{\textsf{\textcolor{cyan}{[Yanhong: #1]}}}
\newcommand\tylerr[1]{\textsf{\textcolor{purple}{[Tylerr: #1]}}}

\newcommand{\hidedc}{%
    \renewcommand\WM[1]{}%
    \renewcommand\AS[1]{}%
    \renewcommand\anej[1]{}%
    \renewcommand\hanna[1]{}%
    \renewcommand\nato[1]{}%
    \renewcommand\nascomment[1]{}%
    \renewcommand\david[1]{}%
    \renewcommand\kyle[1]{}%
    \renewcommand\yanhong[1]{}%
    \renewcommand\tylerr[1]{}%
}

\hidedc  

\title{Olmo Hybrid: From Theory to Practice and Back}

\authorOne[\allenAiAff]{William Merrill\coreContrib}
\authorOne[\allenAiAff]{Yanhong Li\coreContrib}
\authorOne[\allenAiAff]{Tyler Romero\coreContrib}
\authorOne[\allenAiAff\commaAff\ethAff]{Anej Svete\coreContrib}
\authorOne[\lambdaAff]{Caia Costello\coreContrib}
\authorOne[\allenAiAff]{Pradeep Dasigi}
\authorOne[\allenAiAff]{Dirk Groeneveld}
\authorOne[\allenAiAff]{David Heineman}
\authorOne[\allenAiAff]{Bailey Kuehl}
\authorOne[\allenAiAff]{Nathan Lambert}
\authorOne[\lambdaAff]{Chuan Li}
\authorOne[\allenAiAff\commaAff\uwAff]{Kyle Lo}
\authorOne[\allenAiAff]{Saumya Malik}
\authorOne[\lambdaAff]{DJ Matusz}
\authorOne[\allenAiAff\commaAff\cambridgeAff]{Benjamin Minixhofer}
\authorOne[\allenAiAff\commaAff\uwAff]{Jacob Morrison}
\authorOne[\allenAiAff]{Luca Soldaini}
\authorOne[\allenAiAff]{Finbarr Timbers}
\authorOne[\allenAiAff]{Pete Walsh}
\authorOne[\allenAiAff\commaAff\uwAff]{Noah A.~Smith}
\authorOne[\allenAiAff\commaAff\uwAff]{Hannaneh Hajishirzi}
\authorOne[\allenAiAff]{Ashish Sabharwal\coreContrib}

\vspace{-20pt}
\affiliation[\allenAiAff]{Allen~Institute~for~AI}
\affiliation[\uwAff]{University~of~Washington}
\affiliation[\lambdaAff]{Lambda~ML~Team}  
\affiliationTwo[\ethAff]{ETH~Z\"urich}
\affiliationTwo[\cambridgeAff]{University~of~Cambridge}
\vspace{-6pt}

\contribution[]{
\emailLogo \texttt{willm@allenai.org} \quad\quad
\coreContrib{marks core contributors}.
See author contributions \hyperref[sec:contrib]{here}.
}

\abstract{
    Recent work has demonstrated the potential of non-transformer language models, especially linear recurrent neural networks (RNNs) and hybrid models that mix recurrence and attention.
    Yet there is no consensus on whether the potential benefits of these new architectures justify the risk and effort of scaling them up.
    To address this, we provide evidence for the advantages of hybrid models over pure transformers on several fronts.
    First, theoretically, we show that hybrid models do not merely inherit the expressivity of transformers and linear RNNs, but can express tasks \emph{beyond} both,
    such as code execution.
    Putting this theory to practice,
    we train \textbf{\model}, a 7B-parameter model largely comparable to Olmo 3 7B but with the sliding window layers replaced by Gated DeltaNet layers.
    We show that \model outperforms Olmo 3 across standard pretraining and mid-training evaluations, demonstrating the benefit of hybrid models in a controlled, large-scale setting.
    We find that the hybrid model scales significantly more efficiently than the transformer, explaining its higher performance.
    However, its unclear why greater expressivity on specific formal problems should result in better scaling or superior performance on downstream tasks unrelated to those problems.
    To explain this apparent gap, we return to theory and argue why increased expressivity should translate to better scaling efficiency, completing the loop.
    Overall, our results suggest that hybrid models mixing attention and recurrent layers are a powerful extension to the language modeling paradigm: not merely to reduce memory during inference, but as a fundamental way to obtain more expressive models that scale better during pretraining.
}

\metadata[\quad\huggingface Collection:]{
\href{https://huggingface.co/collections/allenai/olmo-hybrid}{\texttt{Olmo-Hybrid-7B}}
}

\metadata[\vspace{.5em}\quad\wandb Training Logs:]{
  \href{https://wandb.ai/ai2-llm/Olmo-Hybrid-7B/reports/Olmo-3-7B-vs-Olmo-Hybrid--VmlldzoxNjA5NDIyNg}{\texttt{Olmo-3-7B-vs-Olmo-Hybrid}}
}

\metadata[\vspace{.5em}\quad\github Training Code:]{
  \href{https://github.com/allenai/OLMo-core}{\texttt{OLMo-core}}~\sans{(\href{https://github.com/allenai/OLMo-core/blob/main/src/scripts/official/OLMo-hybrid/OLMo-hybrid-7B-pretrain.py}{pretrain}, \href{https://github.com/allenai/OLMo-core/blob/main/src/scripts/official/OLMo-hybrid/OLMo-hybrid-7B-midtrain.py}{mid-train}, \href{https://github.com/allenai/OLMo-core/blob/main/src/scripts/official/OLMo-hybrid/OLMo-hybrid-7B-long-context.py}{long context}, \href{https://github.com/allenai/OLMo-core/blob/main/src/scripts/official/OLMo-hybrid/OLMo-hybrid-7B-sft-think.py}{sft-think}, \href{https://github.com/allenai/OLMo-core/blob/main/src/scripts/official/OLMo-hybrid/OLMo-hybrid-7B-sft-instruct.py}{sft-instruct})}
}

\begin{document}
\maketitle

\newpage
\setcounter{tocdepth}{2}
\tableofcontents
\newpage

\section{Introduction}


Transformers have become the backbone architecture for language models \citep{Radford2018ImprovingLU}, after initially eclipsing RNNs in machine translation \citep{vaswani2017attention}.
%
Recently, theoretical and empirical work has characterized the limitations of transformers~\citep{strobl2024formal} and the potential for alternative recurrent architectures to overcome them~\citep{merrill2024illusion,grazzi2025unlocking}.
Transformers lack the expressive power to robustly represent state tracking tasks, which require sequential computation \citep{merrill-sabharwal-2023-parallelism,chiang2025transformers}.
While the parallelism of transformers is important for efficient training, a longstanding issue at inference time is their quadratic scaling with context length.
Modern RNNs strike a better balance: their compute scales \emph{linearly} with context length, and, unlike classical RNNs \citep{elman1990finding,hochreiter1997lstm}, they can be parallelized through linear gating \citep{bradbury2017qrnn,katharopoulos2020transformers,gu2022efficiently,gu2024mamba}.
Moreover, these linear RNNs can also express state tracking tasks, giving an expressivity advantage over transformers \citep{merrill2024illusion,grazzi2025unlocking,peng2025rwkv}.

However, switching from transformers to RNNs is not a free lunch, as RNNs (even linear ones) struggle with copying and recall tasks due to their bounded state \citep{arora2024zoology,jelassi2024repeat}.
This has led to the exploration of \emph{hybrid models} that mix attention and linear RNN layers in order to leverage the benefits of both architectural primitives.
Recently, hybrid models have been trained at scales up to 9B active parameters and 36T tokens (e.g., Mamba-2-Hybrid, \citealp{waleffe2024empiricalstudymamba}; Samba, \citealp{ren2025samba}; Nemotron-H, \citealp{nvidia2025nemotronh}; Qwen3-Next, \citealp{qwen3next2025}; Kimi Linear, \citealp{kimi2025linear}; and Qwen 3.5, \citealp{qwen35blog}) with encouraging results.
Yet there is not yet consensus on whether the advantages of hybrid models justify the cost and risk of switching to a fundamentally different architecture, in part because a controlled comparison between transformer and modern hybrid LMs is lacking at large scale.

To rectify this gap, we introduce \textbf{\model}, a family of model artifacts comparable to Olmo 3 except that the architecture interleaves Gated DeltaNet (GDN, \citealp{yang2025gdn,grazzi2025unlocking}) layers with attention layers at a 3:1 ratio in place of the sliding-window-attention layers used in Olmo 3.
We train \model 7B on up to 6T tokens, finding large improvements in token efficiency (and thus also compute efficiency): \model matches Olmo 3 7B on MMLU~\citep{hendryckstest2021} with 49\% fewer training tokens.
After pretraining, this token efficiency translates to improvements on MMLU and other evaluations, with gains persisting after mid-training; the final \model base checkpoint outperforms Olmo 3 across all aggregated domains of \olmothreeeval.
Beyond these standard base model evaluations, \model shows large gains in long-context ability, with a 14.1\% improvement on RULER 64k over Olmo 3.
Because \model is closely comparable to Olmo 3 apart from its hybrid architecture, we take these results as strong evidence in favor of hybrid models over pure transformers.

Beyond large-scale experiments, we also present theoretical results and fully controlled scaling studies to explain these performance gains.
Past theoretical work has shown attention and recurrence have complementary strengths \citep{merrill2024illusion,grazzi2025unlocking}. Mixing them is, thus, a natural way to reap the benefits of both primitives.
We extend this with novel theory showing that hybrid models are more powerful than the sum of their parts: there are formal problems related to code evaluation that neither transformers nor GDN can express on their own, but which hybrid models can represent theoretically and learn empirically.
This greater expressivity does not immediately imply that hybrid models should be better LMs.
We therefore run controlled scaling studies comparing hybrid models to transformers and show that they indeed attain better token efficiency, consistent with our observations from the \model pretraining run.
It is not clear \emph{prima facie} why greater expressivity on specific formal problems should improve scaling efficiency on benchmarks like MMLU.
To fill this gap, we develop a formal explanation of how increasing an architecture's expressivity can improve data and compute efficiency for loss and downstream tasks---even on tasks unrelated to new problems expressible by the model.

Taken together, our results suggest that hybrid models dominate transformers both theoretically, in their balance of expressivity and parallelism, and empirically, in terms of benchmark performance and long-context abilities.
We believe these findings position hybrid models for wider adoption and call on the research community to pursue further architecture research.

\section{\model Overview} \label{sec:main}

\subsection{Architecture}

Our architecture matches that of Olmo 3 7B (see \citealp{olmo2025olmo3} for details), except that 75\% of layers use GDN heads in place of attention heads. The layers alternate so that 3 GDN layers are followed by 1 multi-head attention layer.
In particular, each GDN head uses GDN \citep{yang2025gdn} extended with negative eigenvalues \citep{grazzi2025unlocking}:

\begin{definition}[GDN with Negative Eigenvalues \normalfont\slightlylarger{\citep{schlag2021lineartransformers,yang2025gdn,grazzi2025unlocking}}] \label{def:gdn}
    Let $d$ be the head dimension.
    For each token $t$, we are given $\mathbf q_t, \mathbf k_t \in \mathbb R^d$, $\mathbf v_t \in \mathbb R^{2d}$, and $\alpha_t, \beta_t \in (0, 1)$, with $\norm{\mathbf k_t} = 1$.
    The initial state $\mathbf S_0$ is the 0 matrix. The state $\mathbf S_t \in \mathbb R^{2d \times d}$ and head output $\mathbf y_t \in \mathbb R^{2d}$ are updated via
    \begin{align*}
        \mathbf S_t &= \mathbf S_{t-1} \alpha_t (\mathbf I - 2 \beta_t \mathbf k_t \mathbf k_t^\top)  + \mathbf v_t \mathbf k_t^\top \\
        \mathbf y_t &= \mathbf S_t \mathbf q_t .
    \end{align*}
\end{definition}

GDN extends the earlier DeltaNet \citep{schlag2021lineartransformers,yang2024parallelizing} by adding the decay factor $\alpha_t \in (0, 1)$ on the state update.
The negative eigenvalue extension replaces $\beta_t$ from the original GDN with $2 \beta_t$.
Despite its simplicity, this change has important consequences for the expressive power of GDN \citep{grazzi2025unlocking}, and we therefore adopt it---in contrast to the GDN architecture used by Qwen-3-Next \citep{qwen3next2025}, Kimi Linear \citep{kimi2025linear}, and Qwen 3.5 \citep{qwen35blog}.

A GDN head fits seamlessly into the overall transformer architecture from Olmo 3 since we can use the standard query, key, and value projections as inputs and treat $\mathbf y_t$ as the per-head output.
It is straightforward to construct a hybrid model by replacing entire attention layers with GDN layers---substituting each attention head with a GDN head.
In particular, we replace the sliding-window attention (SWA) layers (75\% of layers) from Olmo 3 in this way, resulting in a 3:1 hybridization ratio.
In GDN, the value vector is typically twice the length of that in a transformer, and $\beta_t$ requires a few extra parameters to compute.
Thus, we slightly adjust the shape of the model to match Olmo 3 in parameter count and training throughput (see \Cref{sec:training-overview}).

\paragraph{Why GDN?}
As we discuss in \Cref{sec:expressivity}, GDN with negative eigenvalues, unlike linear attention~\citep{katharopoulos2020transformers} or Mamba~\citep{gu2024mamba}, can express state tracking problems that transformers cannot.
Moreover, in \Cref{sec:expressivity}, we prove that hybrid models with attention and GDN layers can express more than purely attention or purely GDN models.
Complementing these expressivity results, we show empirically in \Cref{sec:scaling-laws} that hybrid models with GDN exhibit favorable scaling properties relative to transformers, and provide a theoretical explanation for why their additional expressive power should translate into more efficient scaling (\Cref{sec:expressivity-to-scaling}).

Another important advantage of \model's GDN layers over Olmo 3's SWA layers is the reduced per-layer inference state size, as shown in \Cref{tab:state-size}.

\begin{table}[h]
    \centering
    \small
    \caption{Per-layer inference state size comparison under the Olmo 3 configuration stored in fp16.}
    \label{tab:state-size}
    \begin{tabular}{@{}lrrr@{}}
        \toprule
        Layer Type & Elements & FP16 Size & vs.\ GDN \\
        \midrule
        Multi-Head Attention ($32$K seq, $32$ KV heads, $d_h{=}128$) & $268.4$M & $512$ MiB & $485\times$ \\
        Grouped-Query Attention ($32$K seq, $8$ KV heads, $d_h{=}128$) & $67.1$M & $128$ MiB & $121\times$ \\
        Grouped-Query SWA ($4096$ window, $8$ KV heads, $d_h{=}128$) & $8.39$M & $16.0$ MiB & $15.2\times$ \\
        \model{} GDN ($30$ heads, $d_k{=}96$, $d_v{=}192$) & $0.55$M & $1.05$ MiB & --- \\
        \bottomrule
    \end{tabular}
\end{table}

\subsection{Training Overview} \label{sec:training-overview}

\begin{figure}[!t]
    \centering
    \includegraphics[width=1.0\linewidth]{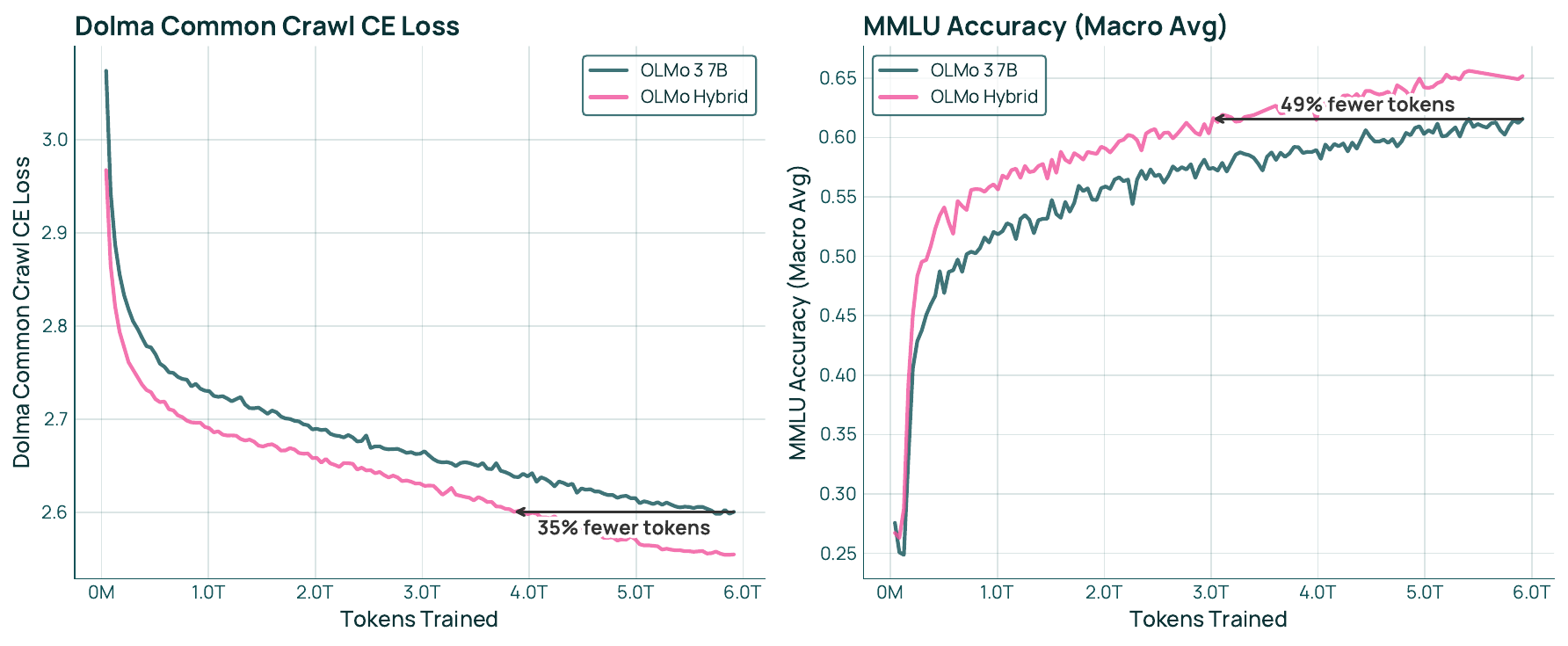}
    \caption{\model 7B is more efficient than OLMo 3 7B during pretraining, reaching the same Common Crawl loss in 35\% fewer tokens and the same MMLU accuracy using 49\% fewer tokens (and thus also 35\% and 49\% fewer FLOPs, respectively).
    }
    \label{fig:pretraining-perf}
\end{figure}

\paragraph{Pretraining.}
Overall, \model aims to use the same hyperparameters and training configuration as Olmo 3 to maximize comparability, though we made a few small changes where warranted.
In particular, since GDN has more parameters than a transformer with the same hyperparameters, we reduced the number of heads in \model from 32 to 30 and set the key and query head dimension to 96 and the value head dimension to 192.
This yielded an \model 7B model with 7.0B parameters (compared to 6.8B for Olmo 3) that closely matched---in fact, slightly outperformed---Olmo 3 7B in early training throughput benchmarks.
We will refer to this model as \textbf{\model 7B}.
We also made a few other minor changes to the hyperparameters and training configuration from Olmo 3; see \Cref{sec:pretraining} for more details.

\paragraph{Mid-Training.}
Mid-training largely follows the procedure from Olmo 3 \citep{olmo2025olmo3}: adaptation on the mid-training data used for Olmo 3 32B (light filtering applied to the mid-training mixture for Olmo 3 7B) followed by long-context extension.
We used a doubled mid-training batch size based on a refined understanding of the interplay between learning rate and batch size since training Olmo 3 \citep{merrill2025critical}. As for Olmo 3 32B, we conduct two separate mid-training runs on different 100B token subsets of Dolma 3 Dolmino Mix and merge the final checkpoints.
For long-context extension, we experimented with the YaRN methodology \citep{peng2024yarn} used for Olmo 3, as well as dropping RoPE entirely (DroPE, \citealp{gelberg2025extendingcontextpretrainedllms}). DroPE removes RoPE after mid-training, eliminating the extrapolation bottleneck imposed by rotation frequencies learned at shorter context lengths. We hypothesize that DroPE is well suited to hybrid architectures because the GDN layers already provide implicit positional encoding through their recurrent structure, reducing the model's reliance on RoPE. While we chose DroPE for the released \model, we evaluate \model with both methods for comparability (see \Cref{tab:lc_evals_ruler}). We run long-context extension on 100B tokens of Dolma 3 Longmino Mix.

As a proof of concept, we also apply the Olmo 3 post-training recipe to Olmo Hybrid in \Cref{sec:posttraining}.

\subsection{\model Results}

\begin{table}[t!]
\setlength\tabcolsep{2pt} 
\renewcommand{\arraystretch}{0.9}
{\fontsize{8}{8}\selectfont
\begin{center}
\begin{footnotesize}
\begin{tabular}{lH|ccccc|cccc}
\toprule
    &\multicolumn{6}{c}{\textbf{\texttt{Base Aggregate Scores}}} & \multicolumn{4}{c}{\textbf{\texttt{Held-out Scores}}} \\
    {\textbf{\fontsize{8}{8}\selectfont~Model}} &
    {\textbf{\fontsize{8}{8}\selectfont~$\#$ Toks}} & 
    {\textbf{\fontsize{8}{8}\selectfont~Math}} & 
    {\textbf{\fontsize{8}{8}\selectfont~Code}} & 
    {$\textbf{\fontsize{8}{8}\selectfont~MC}_\textbf{\fontsize{6}{6}\selectfont~STEM}$} & 
    {$\textbf{\fontsize{8}{8}\selectfont~MC}_\textbf{\fontsize{6}{6}\selectfont~Non-STEM}$} & 
    {\textbf{\fontsize{8}{8}\selectfont~GenQA}} & 
    {\textbf{\fontsize{8}{8}\selectfont~LBPP}} &
    {\textbf{\fontsize{8}{8}\selectfont~BBH}} &
    {\textbf{\fontsize{8}{8}\selectfont~MMLU Pro}} & 
    {\textbf{\fontsize{8}{8}\selectfont~DM Math}}
    \\
\midrule


\rowcolor{lightgrey}    Olmo 3 7B Pretrain & -- & 23.3 & 19.6 & 64.0 & 71.9 & 68.5 & 6.3 & 48.6 & 31.2 & 15.7 \\
\rowcolor{lightgrey}    \model{} 7B Pretrain & -- & 27.0 & 17.1 & 67.4 & 75.5 & 66.8 & 2.9 & 52.2 & 35.7 & 13.2 \\
\rowcolor{midgrey}      Olmo 3 7B Midtrain & -- & 59.8 & 32.1 & 67.3 & 78.2 & 71.3 & 17.8 & 65.9 & 36.6 & 23.8 \\
\rowcolor{midgrey}      \model{} 7B Midtrain & -- & 61.2 & 32.9 & 70.9 & 81.3 & 73.6 & 17.7 & 68.6 & 43.1 & 23.7 \\
\rowcolor{ai2lightpink}    Olmo 3 7B LC & -- & 54.6 & 30.9 & 66.2 & 78.2 & 72.5 & 17.7 & 64.0 & 37.2 & 23.6 \\
\rowcolor{ai2lightpink}    \model{} 7B LC & -- & 55.1 & 32.4 & 70.0 & 80.4 & 72.9 & 16.8 & 65.2 & 41.7 & 23.4 \\

\bottomrule
\end{tabular}
\end{footnotesize}
\vspace{2mm}
\caption{
Base model evaluations across the training pipeline for \model{} vs.~Olmo 3. After mid-training, \model{} outperforms Olmo 3 7B across every evaluation domain.
On evaluations held out from Olmo 3 development (right), \model{} outperforms Olmo 3 on MMLU Pro and BBH but loses slightly on LBPP and DM Math.
Overall, the results suggest that \model{} 7B generally outperforms Olmo 3 7B.
}
\label{tab:base_evals_overview}
\end{center}
}
\end{table}

\paragraph{Baselines.}
To evaluate the performance of hybrid models, we focus on the largely clean comparison between Olmo 3 and \model at a large scale of 7B parameters and 6T tokens.
Later, in \Cref{sec:scaling-laws}, we complement this with extensive fully controlled experiments between transformers and hybrid models at smaller scale, as well as, in \Cref{sec:full-comparison}, comparisons against other released (but not perfectly comparable) transformers, linear RNNs, and hybrid models.

\paragraph{Pretraining Efficiency.} As shown in \Cref{fig:pretraining-perf}, \model is more compute- and data-efficient than Olmo 3, reaching the same Common Crawl loss and MMLU accuracy as Olmo 3 7B with significantly fewer training tokens.
By the end of the 6T-token training run, this increased efficiency of \model translates to gains on both metrics compared to Olmo 3.
This pretraining efficiency is a compelling fundamental advantage of hybrid models over transformers.
Additional pretraining efficiency curves across 6 downstream benchmarks are shown in Appendix~\Cref{fig:pretraining-perf-extended}.

\paragraph{Standard Benchmarks.} In \Cref{tab:base_evals_overview}, we compare \model and Olmo 3 7B across different stages of pretraining and mid-training on domain-specific and held-out evaluation splits from \olmothreeeval.
The final pretrained \model checkpoint outperforms Olmo 3 7B on math (+4.3\%), STEM MC (+3.4\%), and non-STEM MC (+4.4\%), with smaller degradations in code (-2.5\%) and general QA (-2.3\%).
However, after mid-training, \model outperforms Olmo 3 7B across every evaluation domain, and these performance gains persist after long-context extension.
Additionally, we compare Olmo 3 and \model on evaluations held out from the Olmo 3 development process.
On these held-out evaluations, we find gains on MMLU Pro~\citep{wang2024mmlupro} (+4.5\%) and BBH~\citep{suzgun2022bbh} (+1.2\%) but degradations on LBPP~\citep{matton2024lbpp} (-1.1\%) and DM Math~\citep{saxton2019analysing} (-0.2\%).
Taken together, the domain-specific and held-out evaluations point to strong performance of the \model 7B base model relative to Olmo 3 7B.

\begin{table}[t!]
\setlength\tabcolsep{2pt}
\renewcommand{\arraystretch}{0.9}
{\fontsize{8}{8}\selectfont
\begin{center}
\begin{footnotesize}
\begin{tabular}{lc|ccccc}
\toprule
    & &\multicolumn{5}{c}{\textbf{\texttt{RULER Scores}}} \\
    {\textbf{\fontsize{8}{8}\selectfont~Model}} &
    {\textbf{\fontsize{8}{8}\selectfont~LC Method}} &
    {\textbf{\fontsize{8}{8}\selectfont~4k}} &
    {\textbf{\fontsize{8}{8}\selectfont~8k}} &
    {\textbf{\fontsize{8}{8}\selectfont~16k}} &
    {\textbf{\fontsize{8}{8}\selectfont~32k}} &
    {\textbf{\fontsize{8}{8}\selectfont~64k}}
    \\
\midrule


    Olmo 3 7B & YaRN & 95.8 & 89.3 & 83.2 & 78.9 & 70.9 \\
    \model{} 7B & YaRN & 92.8 & 91.3 & 90.0 & 84.7 & 76.9 \\
    \model{} 7B & DroPE & 92.2 & 89.8 & 88.4 & 86.2 & 85.0 \\

\bottomrule
\end{tabular}
\end{footnotesize}
\vspace{2mm}
\caption{
    Results for \model{} after long-context tension with YaRN and DroPE compared against Olmo 3. Both methods outperform Olmo 3 7B at long sequence lengths, with the gains from DroPE being particularly notable.
}
\label{tab:lc_evals_ruler}
\end{center}
}
\end{table}

\begin{figure}[!t]
    \centering
    \includegraphics[width=0.6\linewidth]{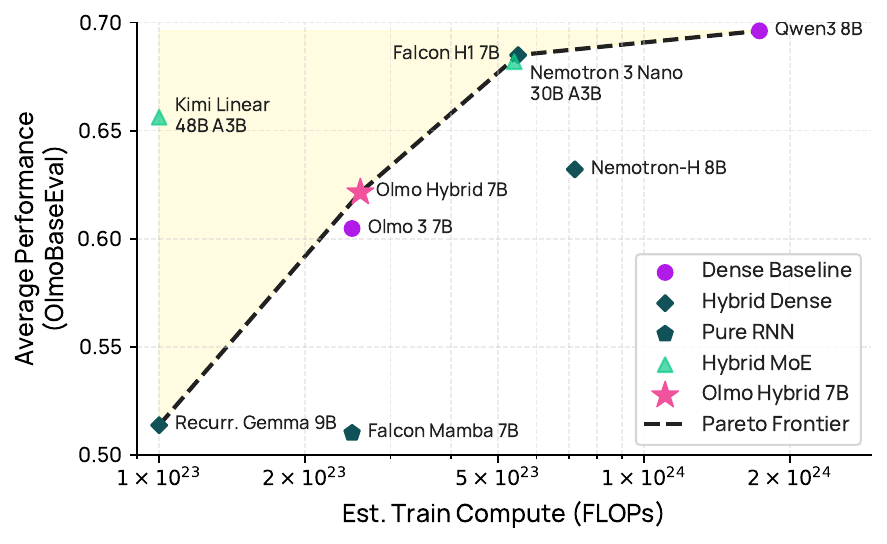}
    \caption{
    Compute-performance tradeoff of open-weight hybrid, RNN, and transformer base models on the average of \olmothreeeval task suites.
    \model 7B is on the Pareto frontier of open-weight dense models.
    Training compute was estimated using the $C=6ND$ heuristic from \citet{kaplan2020scalinglaws}, with reported token and parameter counts.
    Per-benchmark results are reported in Table \ref{tab:base-eval}.
    While theoretical compute for MoE-based SSMs is low, effective training efficiency can be roughly 50--80\% of a dense model \citep{rajbhandari2022deepspeed}, so we draw our frontier over dense models only.
    For this plot, we show only models obtaining >50\% average performance on \olmothreeeval.
    }
    \label{fig:base-performance-frontier}
\end{figure}

\paragraph{Long-Context Capabilities.} Additionally, after long-context extension \citep{olmo2025olmo3}, we find that \model shows substantial gains on RULER \citep{hsieh2024ruler}, a standard long-context benchmark, over Olmo 3.
We adapt \model for long context using both the YaRN methodology \citep{peng2024yarn} from Olmo 3 \citep{olmo2025olmo3} and the more recent DroPE method \citep{gelberg2025extendingcontextpretrainedllms}.
As shown in \Cref{tab:lc_evals_ruler}, \model 7B outperforms Olmo 3 7B at long RULER lengths with both YaRN and DroPE.
We attribute these gains on long-context tasks to the presence of linear RNN layers, in line with existing findings in the literature \citep{yang2025path,ren2025samba}.

\paragraph{Comparison to Open-Weight Models.}
Beyond the controlled comparison between \model and Olmo 3, we also benchmark \model against other open-weight models of similar parameter count.
As shown in \Cref{fig:base-performance-frontier}, \model is on the Pareto frontier for performance on \olmothreeeval as a function of training compute among open-weight dense models.
We detail baseline architectures and present full results in \S\ref{sec:full-comparison}.

\section{Expressive Power of Hybrid Models} \label{sec:expressivity}

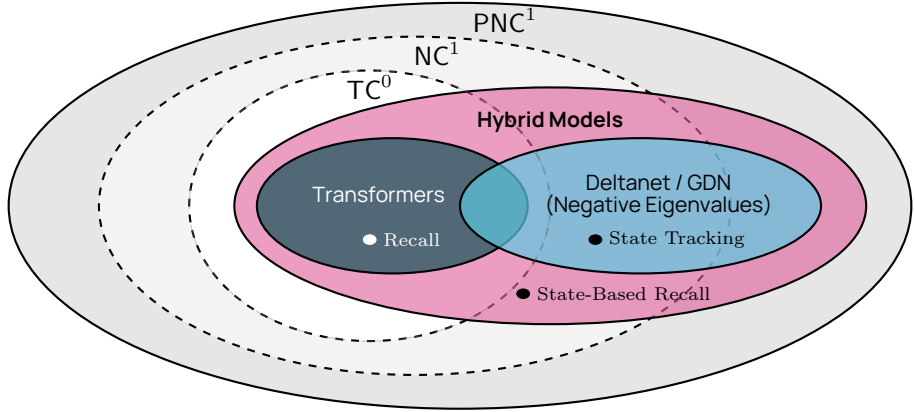
\begin{figure}[t]
    \centering
    \begin{tikzpicture}[xscale=1.2,yscale=0.9]
\definecolor{pnc1color}{RGB}{230,230,230}
\definecolor{nc1color}{RGB}{243,243,243}
\definecolor{tc0color}{RGB}{255,255,255}

\fill[pnc1color] (1,0) ellipse (5cm and 3cm);
\draw[thick] (1,0) ellipse (5cm and 3cm);
\node at (1.5, 2.75) {$\PNC^1$};

\fill[nc1color] (0.5,0) ellipse (3.5cm and 2.5cm);
\draw[thick,dashed] (0.5,0) ellipse (3.5cm and 2.5cm);
\node at (0.75,2.25) {$\NC^1$};

\fill[tc0color] (0,0) ellipse (2cm and 2cm);
\draw[thick,dashed] (0,0) ellipse (2cm and 2cm);
\node at (0, 1.75) {$\TC^0$};

\fill[hybridcolor,opacity=0.6] (2,0) ellipse (3.5cm and 1.75cm);
\draw[thick] (2,0) ellipse (3.5cm and 1.75cm);
\node at (2, 1.25) {\textbf{\footnotesize Hybrid Models}};

\fill[transformercolor,opacity=0.7] (0.25, 0) ellipse (1.5cm and 1cm);
\draw[thick] (0.25, 0) ellipse (1.5cm and 1cm);
\node at (0.1, 0.17) {\textcolor{white}{\footnotesize Transformers}};


\fill[deltanetcolor,opacity=0.7] (3, 0) ellipse (2cm and 1cm);
\draw[thick] (3, 0) ellipse (2cm and 1cm);
\node at (3.2, 0.35) {\footnotesize Deltanet / GDN};
\node at (3.2, 0) {\footnotesize (Negative Eigenvalues)};

\draw[gray, dash pattern=on 8pt off 6pt, opacity=0.45] (0,0) ellipse (2cm and 2cm);

\filldraw[white] (0, -0.5) circle (2pt);
\node[anchor=west, font=\scriptsize] at (0.05, -0.5) {\textcolor{white}{Recall}};

\filldraw[black] (2.5, -0.5) circle (2pt);
\node[anchor=west, font=\scriptsize] at (2.55, -0.5) {State Tracking};

\filldraw[black] (1.7, -1.3) circle (2pt);
\node[anchor=west, font=\scriptsize] at (1.75, -1.3) {State-Based Recall};

\end{tikzpicture}
    \caption{Expressive power of transformers, linear RNNs, and hybrid models relative to circuit complexity classes.
    Dashed lines represent unproven but \emph{conjectured} separations between classes (e.g., $\TC^0 \neq \NC^1$).
    Notably, transformers can express recall (\Cref{sec:transformer-expressivity}), and DeltaNet (or GDN) with negative eigenvalues  can express state tracking (\Cref{sec:rnn-expressivity}) \citep{grazzi2025unlocking}.
    Hybridizing gives both capabilities, and we prove that it also unlocks state-based recall, a problem that neither model can express on its own (\Cref{sec:hybrid-expressivity}).
    With the addition of padding tokens, transformers can express exactly the class $\TC^0$, whereas hybrid models can capture all of $\NC^1$, which enables solving boolean formula evaluation (\Cref{sec:expressivity-exact}).
    }
    \label{fig:hierarchy}
\end{figure}

This section considers the benefits of hybrid models over transformers for \emph{expressive power}, i.e., the class of computational tasks that each architecture can represent. We show that hybrid models can express tasks beyond both transformers and GDN.
Later, we return to explaining how this theoretical expressivity advantage could translate to better compute and data efficiency during pretraining compared to transformers~(\Cref{sec:scaling-laws}).

While no comprehensive theory of deep learning models exists, theory is advanced enough to provide a useful conceptual framework for understanding the practical limitations of architectures and how to address them.
In particular, our motivation for hybrid models takes \emph{expressivity} as a guiding principle:

\begin{tcolorbox}[hybridBoxStyle]
\textbf{Expressivity Thesis} \citep[Section 5.1]{merrill2025thesis} \\
A deep learning architecture should be made as expressive as possible---being able to represent as many naturalistic problems as possible theoretically---while remaining trainable and scalable for pretraining.
\end{tcolorbox}

In other words, this perspective calls for designing a highly \emph{flexible} model whose hypothesis class contains as many subtasks as possible that could conceivably be reflected in naturalistic data, while remaining trainable at scale~\citep[cf.~the Bitter Lesson;][]{sutton2019bitter}.
Rather than constraining the hypothesis class, we trust the optimizer to find the right fit to the data.
This perspective contrasts with classical machine learning wisdom that expressivity and inductive bias are at odds (the bias-variance tradeoff; \citealp{Mitchell1980NeedForBiases,Vapnik1991RiskMinimization,GemanBienenstockDoursat1992BiasVariance}). However, recent work suggests that deep learning methods have an implicit bias toward simplicity that prevents highly expressive models from overfitting \citep{wilson2025position}.
With this in mind, the expressivity thesis outlined above asserts that we should focus on making our architecture maximally expressive without worrying about the impact on inductive bias, as long as the architecture satisfies standard trainability constraints, i.e., the network must be differentiable and signals must propagate across layers during the forward and backward passes \citep[cf.][]{yang2024spectralcondition,dey2025completep}.

Beyond trainability, the other major constraint that competes with expressivity is scalability.
In general, there is a fundamental tradeoff between expressive power and the degree to which the model can process data efficiently during training.
For example, if we want our architecture to exactly express $\NP$-complete problems, it would not be possible to process text with it efficiently (assuming $\P \neq \NP$).
Even if we only require the model to express all tasks in $\P$ at training time, the model still could not be \emph{parallelized} effectively over long sequences \cite[assuming $\NC \neq \P$;][]{greenlaw1991compendium}, which would preclude scaling training to massive amounts of data.
Thus, our goal in architecture design is not to blindly increase expressivity as much as possible, but to do so while preserving parallelism.
A good architecture is one that pushes the frontier of the \textbf{expressivity-parallelism tradeoff} \citep{merrill-sabharwal-2023-parallelism,liu2025serialscaling} as much as possible within fundamental, complexity-theoretic limits.


Existing work has revealed that transformers do not sit on the frontier between expressivity and parallelism: more expressive models exist that are essentially just as parallelizable in practice.
In particular, hybrid models can achieve more expressivity while maintaining similar levels of scalability, which we justify both with prior work and new theoretical results.
Our core results are summarized in \Cref{fig:hierarchy}.
Linear RNNs and transformers have complementary strengths: while linear RNNs like GDN can express state tracking problems beyond the capabilities of transformers, transformers surpass linear RNNs at recall.
Each architecture can be trained at scale, but each also has its own expressivity limitations, motivating \emph{hybrid} models that inherit the best properties from each.
Moreover, we prove new theoretical results establishing that hybrid models have expressivity advantages \emph{beyond} both pure transformers and pure linear RNNs.

\subsection{Background: Expressive Power of Transformers} \label{sec:transformer-expressivity}

As the go-to LM architecture, transformers have received significant theoretical treatment investigating what problems they can and cannot solve \citep{strobl2024formal}.
Despite their empirical success, it has become clear that, when used as next-token predictors, transformers can only express a restricted set of parallelizable computational problems---formally, those that lie within the complexity class $\TC^0$ \citep{merrill-sabharwal-2023-parallelism,chiang2025transformers}.
Under standard complexity conjectures, this implies transformers of fixed depth cannot express \emph{inherently sequential} computational problems over arbitrary sequence lengths.
This includes many problems like graph connectivity ($\NL$-complete) and computing reward in a Markov decision process ($\P$-complete), but, in particular, it includes the fundamental problem of \emph{state tracking}.

\paragraph{Transformers are Limited at State Tracking.} A particularly prominent example of a problem not expressible by transformers (assuming $\TC^0 \neq \NC^1$) is \emph{state tracking} \citep[Section 3]{liu2023transformers,merrill2024illusion}. State tracking is the task of mapping a world state $x \in X$ and sequence of actions $\delta_1, \ldots, \delta_n \in \Delta$ to the updated world state after applying the actions sequentially, e.g., $(\delta_n \circ \ldots \circ \delta_1)(x) \in X$, where both $X, \Delta$ are finite sets and composition over $\Delta$ is associative.\footnote{State tracking can be described in many equivalent ways, e.g., the word problem in a finite monoid, recognizing a regular language, or simulating a deterministic finite automaton \citep[cf.][Section 3]{merrill2024illusion}.}
One natural example of state tracking is chess, where $X$ is the set of all board states and $\Delta$ is the set of all moves.
Another example is the ``shell game'': the problem of composing swaps (transpositions) over five objects.
The complexity of state tracking depends on the algebraic structure of the transition monoid $\Delta$: in the hardest case (capturing the shell game and certain notations of chess; \citealp{merrill2024illusion}), state tracking is $\NC^1$-complete.
It follows that these instances of hard state tracking cannot be expressed by fixed-depth transformers assuming the complexity conjecture $\TC^0 \neq \NC^1$, even though variants of these problems could conceivably manifest as subproblems of next-token prediction over natural language data.
Intuitively, the problem arises from the fact that attention cannot aggregate information over the sequence of state updates in a way that is fully sensitive to their order (each attention head could look to the last token and apply an individual update, but then the depth of the network would grow with the sequence length).

\subsection{Background: Expressive Power of Linear RNNs} \label{sec:rnn-expressivity}

In contrast to transformers, classical nonlinear RNNs can naturally express state tracking by using their hidden vector as a representation of the intermediate state and applying each transition one at a time \citep[e.g.,][Theorem 3.1]{merrill2019sequential}.
A fundamental question is whether linear RNNs, with their linear recurrent update, can still express sequential state updates in a similar way.


\begin{figure}[t!]
\centering
\begin{minipage}{0.48\textwidth}
\lstset{style=pythonstyle}
\begin{lstlisting}
a, b, c, d, e = range(5)
a, c = c, e
...  # n lines
assert a == _  # 0 to 4
\end{lstlisting}
\end{minipage}
\quad
\begin{minipage}{0.48\textwidth}
\lstset{style=pythonstyle}
\begin{lstlisting}
bits = [0, 1, 0, 0, ...]  # m bits
a = 36

assert bits[a] == _  # 0 or 1
\end{lstlisting}
\end{minipage}
\caption{Code evaluation contexts where predicting the next token requires solving state tracking (left) and recall (right). As the number of lines $n$ grows, fixed-depth transformers cannot represent state tracking, assuming $\TC^0 \neq \NC^1$ \citep{merrill2024illusion}. As the number of bits $m$ grows, RNNs with sub-linear precision cannot solve recall because of their bounded state \citep{arora2024based,jelassi2024repeat}. Hybrid models can represent both problems.}
\label{fig:state-tracking-recall}
\end{figure}

\begin{figure}[t!]
\centering
\lstset{style=pythonstyle}
\begin{lstlisting}
bits = [0, 1, 0, 0, ...]  # m bits
a, b, c, d, e = 36, 23, 12, 2, 56  # 0 to m - 1

a, c = c, e
...  # n lines

assert bits[a] == _  # 0 or 1
\end{lstlisting}
\caption{A code evaluation context where predicting the next token requires solving state-based recall.
As $n$ increases, the task becomes inexpressible by transformers because the variable states cannot be tracked (assuming $\TC^0 \neq \NC^1$).
As $m$ grows, the task becomes inexpressible by RNNs because recall into the bit array requires more memory than their bounded state can hold.
There exists a simple hybrid model that can solve the task robustly for any value of $n$ and $m$.}
\label{fig:pointer-based-recall}
\end{figure}

\paragraph{Linear RNNs Can Track State.}
It turns out that some (but not all) linear RNNs, despite being efficiently parallelizable, can express sequential state tracking.
\citet{merrill2024illusion} showed how early linear RNN architectures like S4~\citep{gu2022efficiently} and Mamba are limited in complexity to $\TC^0$ like transformers, and thus cannot express state tracking under standard conjectures.
However, \citet{merrill2024illusion} also showed how, by making the transition matrix \emph{non-diagonal} and \emph{time-dependent}, it is possible to construct a linear RNN that can express $\NC^1$-complete state tracking problems.
Along these lines, recent linear RNN architectures such as DeltaNet \citep{yang2024parallelizing,grazzi2025unlocking}, RWKV-7 \citep{peng2025rwkv}, and PD-SSM \citep{terzic2025structured} have incorporated more complex transition matrix parameterizations that unlock additional expressive power for state tracking.
In the case of DeltaNet or GDN, extending the transition matrix to have negative eigenvalues (\Cref{def:gdn}) is important for state tracking \citep{grazzi2025unlocking}.
Intuitively, negative eigenvalues are important because they allow DeltaNet to express swap-like dynamics that alternate between two states, as shown in \Cref{fig:neg-eigenvalues}.
On empirical state tracking evaluations (such as the $A_5$ word problem), these more expressive parameterizations lead to clear improvements over transformers.
In other related work, \citet{sarrof2024expressive} show advantages of linear RNNs over transformers on star-free regular languages, a simpler type of state tracking compared to $A_5$, and \citet{merrill2026lrnns} show that DeltaNet and other linear RNNs like RWKV-7 are upper bounded by $\PNC^1$ and can solve $\PNC^1$-complete problems.

\paragraph{Linear RNNs are Limited by Recall.}
While linear RNNs have an expressivity advantage over transformers on state tracking tasks, this advantage comes at a cost: they are limited on recall-heavy tasks due to their bounded state.
In particular, the fixed-size hidden state of linear or nonlinear RNNs means they struggle on tasks involving copying or recalling tokens from the context \citep{jelassi2024repeat,arora2024based}.
In-context recall mechanisms have been suggested to be important for language modeling \citep{olsson2022icl}, and, indeed, \citet{arora2024zoology} argue that most loss degradation of linear RNNs relative to transformers can be attributed to in-context recall.
Similarly, \citet{akyurek2024incontext} argue that transformers outperform pure RNNs at in-context learning, likely due to recall abilities.
Additionally, in-context recall is important in constructions for Turing completeness when models are augmented with chain of thought \citep{merrill2024cot,wen2025rnns}.
Thus, for several reasons, it is important for RNN-based models to incorporate some mechanism for recall, providing a strong motivation for hybrid transformer-RNN architectures, which excel at recall-based tasks like copying in practice \citep{waleffe2024empiricalstudymamba}.

\begin{figure}[t!]
    \centering
    \begin{subfigure}[b]{0.48\textwidth}
        \centering
        \[
        \mathbf I - \mathbf k_t \mathbf k_t^\top =
        \mathbf I -
        \frac{1}{2} \begin{pmatrix}
            1 & -1 \\
            -1 & 1
        \end{pmatrix}
        = \frac{1}{2} \begin{pmatrix}
            1 & 1 \\
            1 & 1
        \end{pmatrix}
        \]
        \caption{Without the negative eigenvalue extension (cf.~\Cref{def:gdn}), we get eigenvalues $\lambda_1 = 1$ and $\lambda_2 = 0$.}
        \label{fig:pos-example}
    \end{subfigure}
    \hfill
    \begin{subfigure}[b]{0.48\textwidth}
        \centering
        \[
        \mathbf I - 2 \mathbf k_t \mathbf k_t^\top =
        \mathbf I -
        \begin{pmatrix}
            1 & -1 \\
            -1 & 1
        \end{pmatrix}
        = \begin{pmatrix}
            0 & 1 \\
            1 & 0
        \end{pmatrix}
        \]
        \caption{With the extension, GDN implements a ``swap'' operator with eigenvalues $\lambda_1 = 1$ and $\lambda_2 = -1$.}
        \label{fig:neg-example}
    \end{subfigure}
    \caption{
    Let $\mathbf k_t^\top = (1, -1) / \sqrt{2}$; crucially, $\norm {\mathbf k_t} = 1$.
    Multiplying the low-rank term in the GDN update by 2 allows the transition matrix to have a negative eigenvalue, meaning it can implement dynamics that alternate between states \citep{grazzi2025unlocking}.
    This is useful for expressing state tracking tasks like parity and the $A_5$ word problem.
    }
    \label{fig:neg-eigenvalues}
\end{figure}

\subsection{Hybrid Models are More Than the Sum of Their Parts} \label{sec:hybrid-expressivity}

\Cref{sec:rnn-expressivity} demonstrates that linear RNNs like DeltaNet have a key expressivity advantage over transformers on state tracking tasks.
On the other hand, they are limited on recall-heavy tasks due to their bounded state size.
From this perspective, hybrid models that mix layer types offer a natural way to build a model that can express both state tracking and recall while remaining scalable \citep{yang2025path,mohri2026rationaltransductors}.
Going beyond this, we now show that hybrid models are, in fact, \emph{more expressive} than either transformers or linear RNNs in isolation, under standard complexity conjectures.
In particular, we will first show the expressivity advantage of hybrid models on the minimal synthetic task of state-based recall:

\begin{definition}[State-Based Recall \normalfont\slightlylarger{(\Cref{fig:pointer-based-recall})}] \label{def:pointer-based-recall}
    The input is a string $x, p, \pi$, where $x \in \{0, 1\}^n$ is a bitstring, $p_1, \ldots, p_5 \in \mathbb [1, n]$ are ``pointers'' into the bitstring, and $\pi_1, \ldots, \pi_n$ is a sequence of transpositions (swap operations) over $[1, 5]$.
    Define the permuted pointer values $q$ as $q = (\pi_n \circ \pi_{n-1} \circ \ldots \circ \pi_1) (p)$.
    The output is $x_{q_1}$.
\end{definition}

Intuitively, state-based recall is designed to require \emph{composing} state tracking and recall, rather than just one capability.
State-based recall can be naturally instantiated as a subtask of \emph{evaluating code}, which requires both tracking the state of variables and using their values for memory accesses (e.g., indexing into a list).
\Cref{fig:pointer-based-recall} shows an example of how state-based recall might appear within code language modeling.
This gives an intuition for how the additional expressivity of hybrid models could be relevant for tasks involving code evaluation.
We now show formally that, since it requires the composition of state tracking and recall, state-based recall is solvable by hybrid models, but not pure transformers or GDN models:\footnote{Our formal models of transformers and RNNs operate under the common \emph{log-precision arithmetic} assumption \citep{merrill-sabharwal-2023-parallelism,merrill2024illusion}, where all internal arithmetic is allowed to use $O(\log n)$ bits for input sequences of length $n$. All our inexpressibility results easily carry over to bounded-precision models. \Cref{thm:pointer-recall} holds also for poly-precision transformers~\citep{chiang2025transformers} and can be extended to sub-linear-precision RNNs.}

\begin{restatable}[State-Based Recall Separation]{theorem}{pointerRecall} \label{thm:pointer-recall}
    There exists a hybrid model (GDN with negative eigenvalues + averaging-hard attention) that solves state-based recall, with just one alternation between layer types, in either order.
    In contrast, no transformer or
    RNN can express this problem, assuming $\TC^0 \neq \NC^1$ for transformers.
\end{restatable}

\begin{proof}[Proof Sketch]
    A hybrid model can solve state-based recall in two ways.
    One way is to first use GDN layers to compose the transpositions over pointers (an example of $\NC^1$-complete state tracking) and then retrieve the value at $p_1$ using attention.
    The other way is to first use attention heads to retrieve the values of $p_1, \ldots, p_5$ and then compose transpositions over the values.
    Thus, hybrid models with a single alternation of layer types in either order can express the task.
    On the other hand, full transformers and GDN each lack the ability to express one of the pieces, and thus cannot express state-based recall.
    We defer a rigorous proof of these negative results to \Cref{sec:hybrid-expressivity-proofs}.
\end{proof}

Since state-based recall is expressible with GDN before attention or vice versa,
it follows that alternating layers in either order also unlocks additional expressivity for hybrid models relative to full attention or full GDN.
It is an open question whether multiple alternations between layer types buy more expressivity than having just one alternation.


\subsection{Expressive Power of Padded Hybrid Models} \label{sec:expressivity-exact}

The interaction between attention and GDN layers naturally allows hybrid models to solve state-based recall, providing a minimal example of an expressivity gain over pure transformers and linear RNNs (under standard conjectures).
Moreover, we now show that it also provides more general expressivity benefits for hybrid models: in the presence of padding tokens, hybrid models can express \emph{every} problem in $\NC^1$, which includes many $\NC^1$-complete problems that are beyond the capabilities of padded transformers (which remain in $\TC^0$).
In particular, our results imply that, unlike padded transformers, padded hybrid models can evaluate boolean formulas, even though, at first glance, it is non-obvious why the interaction of attention and recurrence should enable this.

Before presenting our results, we briefly introduce \emph{padding tokens} and their relevance to theoretical analysis of expressivity.
Padding tokens \citep{goyal2024think,pfau2024lets} are simply blank tokens ($\square$) appended to a model's input context: i.e., a model with $n^c$ padding takes as input $w \square^{\abs{w}^c} \in \Sigma^*$ rather than just $w \in \Sigma^*$.
Padding tokens are interesting for theoretical analysis because adding padding tokens leads architectures to subsume (or exactly correspond to) natural complexity classes in expressivity.
Without padding tokens, showing that an entire circuit complexity class neatly lower bounds an architecture is not always possible, in large part because circuit complexity classes like $\TC^0$ and $\NC^1$ allow arbitrary polynomial size computation, whereas models leverage a computation graph of size $O(n^c)$ for some fixed $c$.
Adding padding tokens can close the gap between poly-size circuit classes and neural sequence models by allowing them to expand the size of the model's computation without adding new parameters \citep{li2024chain,merrill2025exact,london2025pausetokens}.

It is already established that, with padding tokens, the exact class of problems expressible by transformers has a natural characterization.
In particular, while averaging-hard-attention transformers without padding tokens are bounded within $\TC^0$ \citep{merrill-sabharwal-2023-parallelism}, with polynomial padding, the languages expressible by such transformers are \emph{exactly} $\TC^0$ \citep{merrill2025exact}:

\begin{theorem}[Padded Transformers \normalfont\slightlylarger{\citep{merrill2025exact}}] \label{thm:padded-ahat}
    With polynomial padding tokens, fixed-depth transformers with averaging-hard attention recognize exactly $\FO$-uniform $\TC^0$.
\end{theorem}

We give a comparable result for hybrid models with polynomial padding, showing they capture \emph{all} of $\NC^1$, a complexity class thought to be larger than $\TC^0$:

\begin{restatable}[Padded Hybrid Models]{theorem}{paddedHybrid} \label{thm:padded-hybrid}
    With polynomial padding tokens, fixed-depth hybrid models (averaging-hard attention + GDN with negative eigenvalues) can recognize any language in $\FO$-uniform $\NC^1$.
\end{restatable}

\begin{proof}[Proof Sketch]
    We use the surprising classical result that any problem in $\NC^1$ can be reduced via a first-order ($\FO$) formula to composing transpositions over 5 elements \citep{barrington1986bounded}.
    Padded attention layers can express an arbitrary $\FO$ reduction, and GDN can implement transposition composition.
    Thus, this decomposition allows solving any problem in $\NC^1$ with a padded hybrid model.
    Full proof in \Cref{sec:hybrid-expressivity-proofs}.
\end{proof}

In contrast to \Cref{thm:pointer-recall}, which establishes an expressivity advantage for hybrid models on a specific problem, \Cref{thm:padded-hybrid} establishes that padded hybrid models subsume an entire complexity class ($\NC^1$) that, under standard complexity conjectures, is more powerful than the class padded transformers correspond to ($\TC^0$).
Under the conjecture that $\TC^0 \neq \NC^1$, any $\NC^1$-complete problem can be expressed by a padded hybrid model but not by a padded transformer.
One interesting problem in this regime is \textbf{boolean formula evaluation}.
Let $\phi$ be a boolean formula over $\langle \{0, 1\}, \vee, \wedge \rangle$.
The evaluation problem takes as input $\phi$ serialized in Polish notation, and
the output is the value of $\phi$.
Since formula evaluation is $\NC^1$-complete \citep{buss1987alogtime}, we obtain:

\begin{corollary}[Boolean Formula Evaluation Separation] \label{cor:formula-eval}
    For some $c$, there exists a hybrid model (averaging-hard attention + GDN with negative eigenvalues) that solves boolean formula evaluation with $n^c$ padding tokens.
    On the other hand, assuming $\TC^0 \neq \NC^1$, no transformer or
    RNN can solve boolean formula evaluation even with $n^c$ padding tokens, for any $c$.
\end{corollary}

In contrast to state-based recall, it is not obvious at first glance that boolean formula evaluation should be expressible via the interaction of attention and recurrence.
However, \Cref{cor:formula-eval} shows that it can be, at least for padded models, and moreover that, under standard complexity conjectures, it could not be solved with just one of these primitives.
\Cref{cor:formula-eval} can be restated to fold the padding into the problem definition itself: \emph{unpadded} hybrid models can solve \emph{padded} formula evaluation, but neither transformers nor linear RNNs can.
Finally, the only property of boolean formula evaluation leveraged to obtain \Cref{cor:formula-eval} is that it is $\NC^1$-complete. Thus, a similar result follows for any $\NC^1$-complete problem.
It is an open question whether a similar boolean formula evaluation result might be obtained for unpadded hybrid models.

\begin{figure}[t!]
    \centering
    \includegraphics[width=\linewidth]{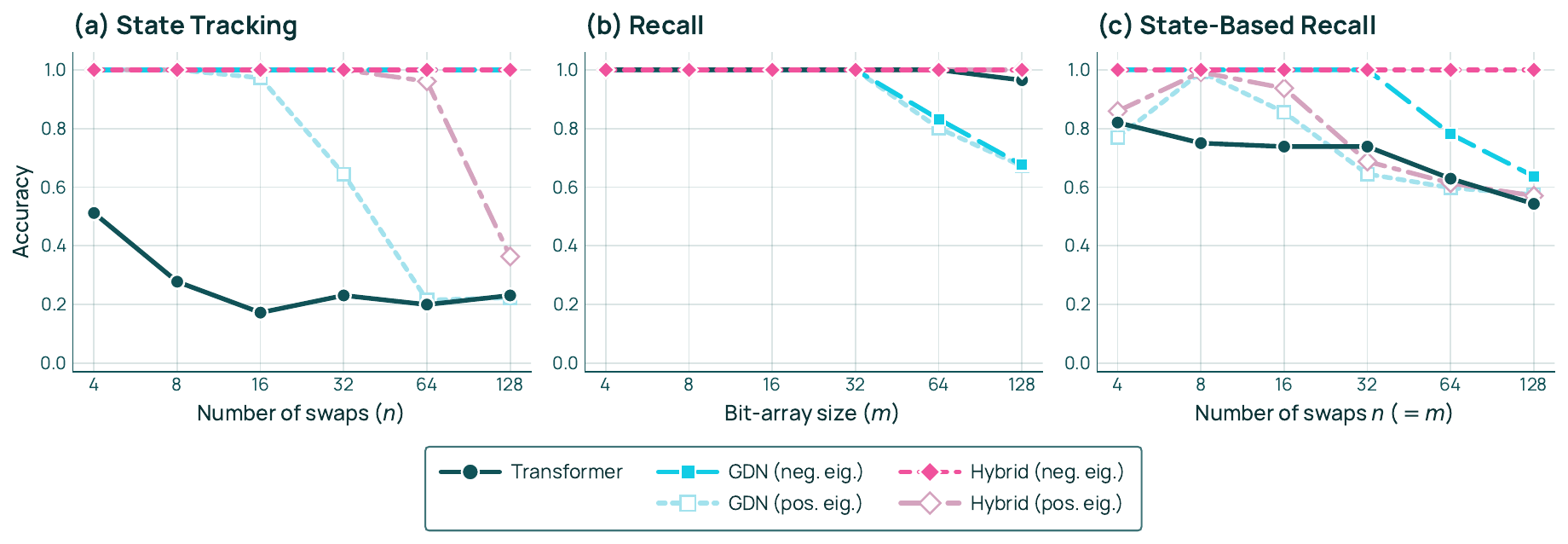}
    \caption{Synthetic task results.
    \textbf{(a)} State tracking accuracy vs.\ number of updates $n$: linear RNN and hybrid remain near-perfect, while the transformer falls quickly.
    \textbf{(b)} Recall accuracy vs.\ bit-array size $m$: transformer and hybrid maintain perfect recall, while the linear RNN degrades for large $m$.
    \textbf{(c)} State-based recall accuracy vs.\ number of swaps $n$ ($m=n$): the hybrid remains robust, while both the transformer and linear RNN degrade for large $m, n$.}
    \label{fig:synth-results}
\end{figure}

\subsection{Synthetic Evaluations} \label{sec:synthetic-evals}

To empirically validate the expressivity tradeoffs from
\Cref{sec:transformer-expressivity,sec:rnn-expressivity,sec:hybrid-expressivity},
we train three primary causal LMs---a standard transformer, a linear RNN (GDN with negative eigenvalues), and a hybrid model mixing GDN blocks with full attention---on synthetic state tracking, recall, and state-based recall tasks.
We additionally evaluate no-negative-eigenvalue ablations of the linear RNN and hybrid model.
Each task is generated online from short code-like templates (cf.\ \Cref{fig:state-tracking-recall,fig:pointer-based-recall}) and framed as a next-token prediction task where the final token requires evaluating the program correctly.
We report next-token accuracy on the final answer token and vary task difficulty by increasing the number of required state updates and/or the size of the stored bit-array.
For the state-tracking and state-based recall tasks, we additionally train on execution traces that include intermediate state-reveal checks in the form of assert statements, and we include these tokens in the language-modeling loss following \citet{siems2026learningstatetrackingcodeusing}.
Complete experimental details (model hyperparameters, optimization, curricula, and ablation runs) are provided in \Cref{app:synthetic-evals}.

\paragraph{State Tracking.}
To evaluate state tracking, we use the code evaluation task in \Cref{fig:state-tracking-recall}, where the context applies a sequence of state updates and the model must predict the value of a queried component of the final state.
As shown in \Cref{fig:synth-results}, the GDN-based linear RNN and the hybrid remain near-perfect across all tested lengths, while the transformer drops rapidly as the update sequence length increases, consistent with the theoretical limitation of fixed-depth attention on state tracking (\Cref{sec:transformer-expressivity}).
The no-negative-eigenvalue ablations substantially weaken this behavior: the linear RNN without negative eigenvalues falls to $0.21484$ and $0.22266$ at $n=64$ and $128$, and the corresponding hybrid falls to $0.36328$ at $n=128$.
In this implementation, negative eigenvalues appear important for robust recurrent state tracking, consistent with the findings from \citet{grazzi2025unlocking}.

\paragraph{Recall.}
To evaluate pure recall, we use the code evaluation task illustrated in \Cref{fig:state-tracking-recall}, where predicting \texttt{bits[idx]} requires retrieving a single entry from a long list of bit assignments.
As shown in \Cref{fig:synth-results}, the transformer remains near-perfect and the hybrid remains perfect across the tested context sizes, while the linear RNN degrades as the number of bits grows.
The no-negative-eigenvalue ablation leaves the hybrid unchanged and only slightly worsens the linear RNN on the hardest settings ($0.80078$ at $m=64$ and $0.67188$ at $m=128$, compared with $0.83203$ and $0.67578$ with negative eigenvalues).
This matches the theoretical expectation that linear RNNs should struggle with recall due to their bounded-size state (\Cref{sec:rnn-expressivity}).

\paragraph{State-Based Recall.}
To test the composition of state tracking and recall, we use the state-based recall task in \Cref{fig:pointer-based-recall}, where state updates transform a set of pointers that are subsequently used to index into a bit array.
Theoretically, \Cref{thm:pointer-recall} predicts hybrid models can express state-based recall while transformers and RNNs cannot under standard complexity conjectures; we now test empirically whether a separation is observed when these models are trained on the task.
As shown in \Cref{fig:synth-results}, both the transformer and the pure GDN-based linear RNN degrade as difficulty increases, while the hybrid with negative eigenvalues remains essentially perfect across all tested difficulties.
By contrast, the no-negative-eigenvalue hybrid loses this advantage, dropping to $0.57031$ at $n=128$; the corresponding no-negative-eigenvalue linear RNN is similarly weak.
This suggests that, in practice, the hybrid's compositional advantage depends on the recurrent block's access to negative eigenvalues.

Overall, these results support the theoretical picture while refining the architectural story: transformers excel at recall but struggle with hard state tracking, GDN-based linear RNNs show the opposite pattern, and hybrids with negative eigenvalues inherit both capabilities and remain strong on their composition. Removing negative eigenvalues has little effect on pure recall, but substantially harms state tracking and the composed state-based recall task, in line with theoretical predictions \citep{grazzi2025unlocking}.

\subsection{Discussion: Pushing the Expressivity-Parallelism Frontier}

Transformers and linear RNNs have complementary strengths from an expressivity perspective.
We have shown that hybrid models do not simply inherit the strengths of each architecture; they go beyond both transformers and linear RNNs by expressing tasks that neither architecture alone can.
Notably, this expressivity advantage of hybrid models comes despite them retaining theoretical parallelizability to a similar degree as transformers \citep{merrill2026lrnns}.
In other words, hybrid architectures extract strictly more from the level of parallelism at which transformers and linear RNNs operate, in a similar vein to the idea of ``leaving less money on the table'' and getting more out of the same resources in machine learning~\citep{ligett2026:talk-leaving-less-money-behind}.
Thus, hybrid models push the expressivity-parallelism frontier for language modeling architectures beyond transformers (cf.~beginning of~\Cref{sec:expressivity}), yielding fundamentally more expressive models that remain similarly scalable.


\section{Scaling Behavior of Hybrid Models} \label{sec:scaling}

Having established the increased theoretical expressivity of hybrid models over transformers in \Cref{sec:expressivity}, we now turn to evaluating hybrid LMs in practice.
A central question is whether the theoretical expressivity guarantees for hybrid models translate to better empirical performance as a function of the parameter count and data invested into a language model.

First, by fitting scaling laws for hybrid models and transformers in a carefully controlled setting (\Cref{sec:scaling-laws}), we establish that hybrid models achieve better scaling efficiency on loss-based pretraining metrics compared to transformers---particularly in scaling with data quantity---consistent with our findings from the \model pretraining run.
Next, drawing on existing theoretical explanations for scaling laws, we present a theoretical argument that increasing an LM's expressive power should improve its scaling efficiency (\Cref{sec:expressivity-to-scaling}).
Informally, building on explanations of scaling laws in terms of the multi-task nature of language modeling \citep{michaud2023quantization}, we argue that increased expressivity can improve scaling because it allows a model to acquire a larger proportion of the discrete computational tasks reflected in the pretraining data.
This provides a conceptual explanation for the improved scaling efficiency of hybrid models that we take as a compelling hypothesis for future work to test and develop further.


\subsection{Scaling Laws for Hybrid Models} \label{sec:scaling-laws}

This section presents derived empirical scaling laws for transformers, pure GDN models, and hybrid models that form the backbone of \model.
Empirical scaling laws enable a principled comparison of different architectures, evaluating not only performance at specific scales but also projecting to larger scales.
Our results show that hybrid models are both more data-efficient and more compute-efficient than transformers across scales.

\paragraph{Overview.}
We fit Chinchilla-style scaling laws \citep{hoffmann2022chinchilla} to transformers, pure GDN models, and hybrid GDN--transformer models.
We find that the hybrid model has a meaningfully lower data coefficient $B$, while scaling exponents are statistically indistinguishable across architectures.
This aligns with our theoretical analysis in \Cref{sec:expressivity-to-scaling} that increased expressivity should improve scaling efficiency by reducing the data coefficient, which captures the fixed-factor improvement in the data required to reach a target loss.
This translates to projected token savings of ${\sim}1.3$--$1.9\times$ at model sizes from 1B to 70B parameters.

\paragraph{Scaling Laws Formulation. } Scaling laws describe the behavior of the language modeling loss as a smooth power law with model size and data budget \citep{kaplan2020scalinglaws}.
\citet{hoffmann2022chinchilla} investigate scaling laws with the parametric form
\begin{equation} \label{eq:scaling-law}
    L(N, D) = E + \frac{A}{N^\alpha} + \frac{B}{D^\beta},
\end{equation}
where the number of parameters $N$ and number of training tokens $D$ are the independent variables.
The quantity $E$ is the irreducible loss, i.e., the loss that would be attained with infinite resources.
The other fit parameters govern how efficiently loss is reduced when $N, D$ are scaled.
The coefficients $A$ and $B$ capture fixed-factor improvements in the resources required to reach a target loss: reducing either by a factor $k$ with fixed $E$ implies that the same target loss can be reached with a $k$-fold reduction in resources.
Finally, the exponents $\alpha$ and $\beta$ govern the rate at which loss improves with scale, though generally these exponents are not changed much by architecture choices.
Comparing these parameters between two architectures provides a principled way to quantify which is fundamentally more compute- or data-efficient.

We now fit scaling laws to evaluate how efficiently hybrid models scale relative to other architectures.
We evaluate three architectures---a pure transformer, a pure GDN, and a hybrid GDN architecture with every fourth layer being a full transformer block---across scales from $60\text{M}$ to $760\text{M}$ parameters, fitting scaling laws to estimate their coefficients and project performance at larger compute budgets.
We follow an \textbf{isoparams data collection strategy} \citep{hoffmann2022chinchilla}: for each model size $N \in \{60\text{M}, 100\text{M}, 190\text{M}, 370\text{M}, 600\text{M}, 760\text{M}, 1\text{B}\}$, we train a model at $0.5\times, 1\times, 2\times, 4\times$, and $8\times$ Chinchilla-optimal tokens (20 tokens per parameter).
Na{\"i}vely, this would require launching \emph{five} separate runs per architecture and parameter budget.
We avoid this by using a WSD-S (warmup--stable--decay with periodic resets) learning rate schedule \citep{hu2024minicpm,wen2025understanding}, which is \textbf{token-agnostic}: the learning rate at step $t$ does not depend on the total number of tokens $T$.
This allows us to \emph{reuse} intermediate checkpoints from a single long run to collect loss measurements at different data budgets $D$.
More precisely, we \emph{decay} the learning rate for $5\%$ of the training tokens to 0 at each of the five Chinchilla factors to obtain the trained checkpoint for that Chinchilla factor.
Decaying the learning rate at each of the five Chinchilla factors and evaluating the resulting checkpoints yields five $(N, D, L)$ triples from a single training run per architecture and parameter count.

We run this procedure for three models:
\begin{enumerate}
    \item The transformer baseline based on the Olmo 3 model,
    \item A fully linear RNN (GDN), and
    \item A hybrid model with a $3:1$ GDN-to-transformer layer ratio---the architecture that forms the basis of \model.
\end{enumerate}
We train all architectures under \textbf{matched conditions}: identical training data (Olmo 3 32B mix), optimizer settings, batch size scaling, and evaluation.
The only difference between runs is the architecture itself.

\paragraph{Model Specification.}
Rather than matching parameter counts exactly, we match the architectural blueprint---the number of layers, heads, and the hidden dimension---of each model to that of Olmo 3 at each scale.
This results in models with different numbers of parameters.
We account for these differences by focusing on performance as a function of the training \emph{FLOPs} and fit the scaling laws with respect to the exact number of parameters.
The detailed configurations are presented in \Cref{tab:ablation-configs}.


\paragraph{Fitting and Using Scaling Laws.}
Given the collected $(N, D, L)$ data triples (35 per architecture; 5 for each of the 7 scales), we fit parametric scaling laws following Approach 3 of \citet{hoffmann2022chinchilla}: we directly fit $L(N, D) = E + A / N^\alpha + B / D^\beta$ by minimizing the Huber loss of $\log L$ between predictions and data.
We model the \emph{validation} loss, computed as the average cross-entropy across 11 held-out evaluation domains: C4~\citep{raffel2019t5}, Dolma Books, Dolma Common Crawl, Dolma pes2o, Dolma Reddit, Dolma Stack, Dolma Wiki~\citep{soldaini2024dolma}, ICE~\citep{greenbaum1996ice}, M2D2 S2ORC~\citep{reid2022m2d2,lo2020s2orc}, Pile~\citep{gao2020pile}, and WikiText-103~\citep{merity2016wikitext}.
Averaging across domains provides a less noisy and more representative signal than any single validation set.
Following \citet{hoffmann2022chinchilla}, we use Huber loss to reduce sensitivity to outliers from training instabilities.
The fit is performed jointly over all $(N, D, L)$ triples for each architecture separately.
To assess uncertainty in our scaling law estimates, we additionally compute 95\% confidence intervals via bootstrap resampling (1{,}000 iterations).

\paragraph{Free and Fixed-Exponent Fits.}
We fit scaling laws in two ways.
In the \emph{unconstrained} fit, all five parameters ($E$, $A$, $\alpha$, $B$, $\beta$) are free; this provides the most flexible fit that explains the scaling laws but yields wide bootstrap confidence intervals, making per-coefficient comparisons unreliable.
We therefore also present a \emph{fixed-exponent} fit in which we fix $\alpha$ and $\beta$ across all architectures and fit only $E$, $A$, and $B$.
Fixing the exponents concentrates the remaining variance onto the efficiency coefficients, enabling more statistically robust comparisons between architectures.
Full details on the data used for fitting, the optimization procedure, and the construction of all figures and tables in this section are provided in \Cref{sec:fig-table-details}.

\begin{figure*}[t!]
\centering
\includegraphics[width=\textwidth]{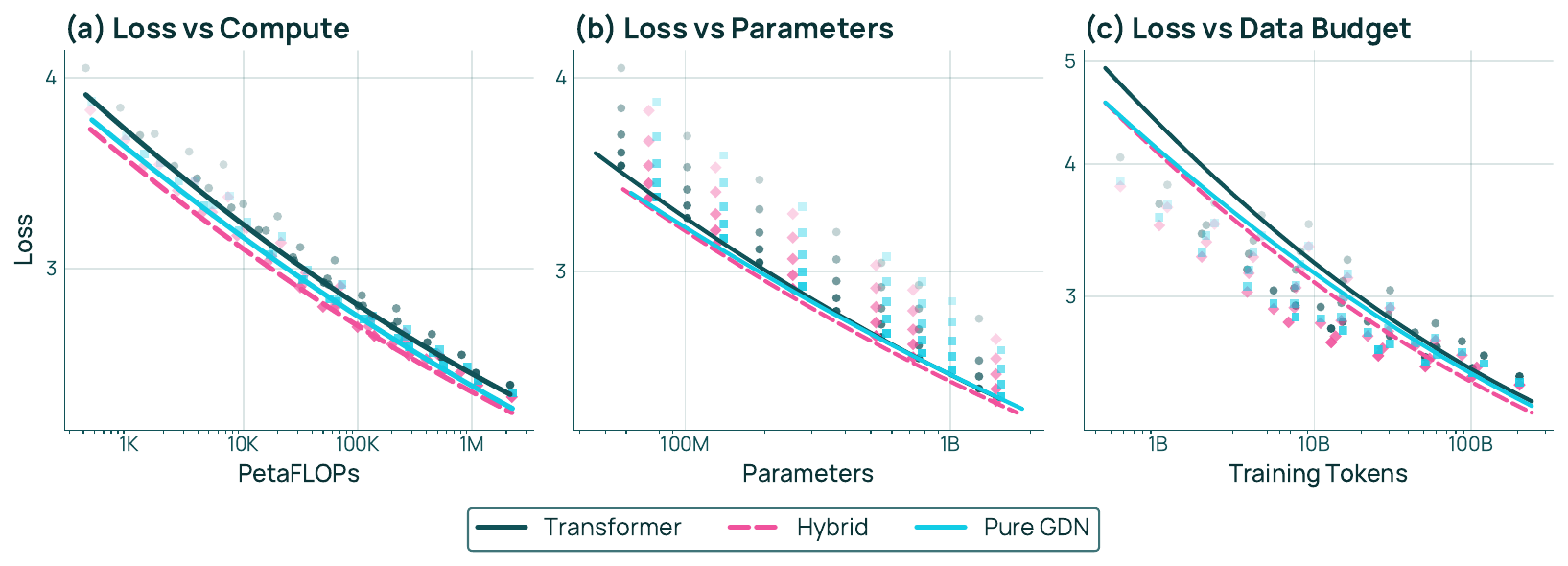}
\caption{Scaling law fits $L(N,D) = E + A/N^\alpha + B/D^\beta$ for all architectures. Transformer: $\alpha=0.252$, $\beta=0.213$, $R^2=0.998$. Hybrid: $\alpha=0.226$, $\beta=0.219$, $R^2=0.999$. Pure GDN: $\alpha=0.183$, $\beta=0.227$, $R^2=0.999$. \textbf{(a)} Loss vs.~compute; points represent model checkpoints at different model sizes (shown by opacity) and the bold curve shows the fitted scaling law at the compute-optimal frontier. \textbf{(b)} Loss vs.~parameter count; points represent checkpoints at different Chinchilla multiples ($D/N$ ratios, shown by opacity) and the fitted curve is evaluated at the largest multiple. \textbf{(c)} Loss vs.~data budget; points represent checkpoints at different model sizes (shown by opacity) and the fitted curve is evaluated at the largest model size. Loss axes are plotted on a logarithmic scale so that the power-law relationships appear linear.
}
\label{fig:scaling-law-fit-log-opt}
\end{figure*}

\paragraph{Scaling Law Fits and Fit Quality.}
\Cref{fig:scaling-law-fit-log-opt} shows the fitted scaling laws alongside the raw data for all three architectures, plotted against compute, parameter count, and token budget.
The curves in all figures are drawn using the coefficients estimated from the unconstrained fit; we use the fixed-exponent fit only when making quantitative comparisons between architectures (see below).
Qualitatively, the hybrid and pure GDN curves lie consistently below the transformer curve, indicating lower loss at matched compute or parameter budget; the separation is visible across all three panels.
The fits achieve $R^2 \geq 0.998$ for all three architectures, confirming that the Chinchilla scaling law form describes the observed loss curves well across the full range of scales considered.

\paragraph{The Derived Scaling Laws are Consistent with Pretraining Observations.}
\Cref{tab:prediction-validation} validates the derived scaling laws against \emph{larger} models---the final base 7B \model as well as the 7B and 32B base Olmo~3 models---and confirms that they generalize well outside the fitting range.
The predicted loss for the 7B \model trained on 5.5T tokens is $2.142$, compared to the observed $2.136$ (error $0.28\%$); for Olmo~3 7B (5.93T tokens), the error is similarly $0.28\%$; and even for Olmo~3 32B---well beyond the ${\sim}1\text{B}$ fitting range---the error is only $1.69\%$.
These results establish that the fitted scaling laws are reliable enough to project performance at larger scales.

\begin{table}[tbp]
    \centering
    \small
    \begin{tabular}{llrrrrr}
        \toprule
        Architecture & Model & $N$ & $D$ & Observed & Predicted & Error (\%) \\
        \midrule
        Transformer & Olmo 3 & 7B & 5.9T & 2.17 & 2.16 & 0.28 \\
        Transformer & Olmo 3 & 32B & 5.5T & 2.02 & 2.05 & 1.69 \\
        \midrule
        Hybrid GDN & \model{} & 7B & 5.5T & 2.14 & 2.14 & 0.28 \\
        \bottomrule
    \end{tabular}
    \caption{Scaling law prediction validation against final models. Predicted loss is computed from the fitted law $L(N, D) = E + A/N^\alpha + B/D^\beta$.
    Extrapolation is generally strong despite differences in the learning rate schedule between our scaling studies and the large-scale training runs.
    }
    \label{tab:prediction-validation}
\end{table}

\paragraph{The Hybrid Model Has a Robustly Lower Data Coefficient.}
\Cref{fig:scaling-law-fit-log-opt} visually suggests that the hybrid model has better scaling efficiency than the transformer, particularly with respect to data.
To quantify this, we compare the fitted scaling law parameters across architectures in \Cref{fig:scaling_params_bar} (see \Cref{tab:scaling-law-parameters} in the appendix for full details).
While the unconstrained fit allows all five parameters to vary and is thus the most flexible, the resulting CIs are wide and largely overlapping, yielding no statistically robust per-coefficient conclusions.
We therefore turn to the fixed-exponent fit to draw robust conclusions: we fix $\alpha = \beta = 0.22$ shared across architectures (the mean of the unconstrained estimates) and refit only $E$, $A$, and $B$.
This gives us a clear signal in the \emph{data efficiency} coefficient $B$: the hybrid's $B = 83.7$ (CI $[80.2, 87.1]$) is significantly lower than the transformer's $94.9$ (CI $[88.7, 102.0]$), with non-overlapping confidence intervals.
The parameter coefficient $A$ also slightly favors the hybrid ($70.1$ vs.\ $71.8$), but the CIs overlap.
The scaling exponents in the unconstrained fit are statistically indistinguishable across architectures---consistent with the theory (\Cref{sec:expressivity-to-scaling}), which predicts that expressivity shifts the efficiency constants without altering the power-law exponents.
The primary \emph{robust} advantage of the hybrid model is thus its data efficiency coefficient $B$, as also predicted by our theoretical results in \cref{sec:expressivity-to-scaling}.

\begin{figure}[t!]
\centering

\includegraphics[width=\textwidth]{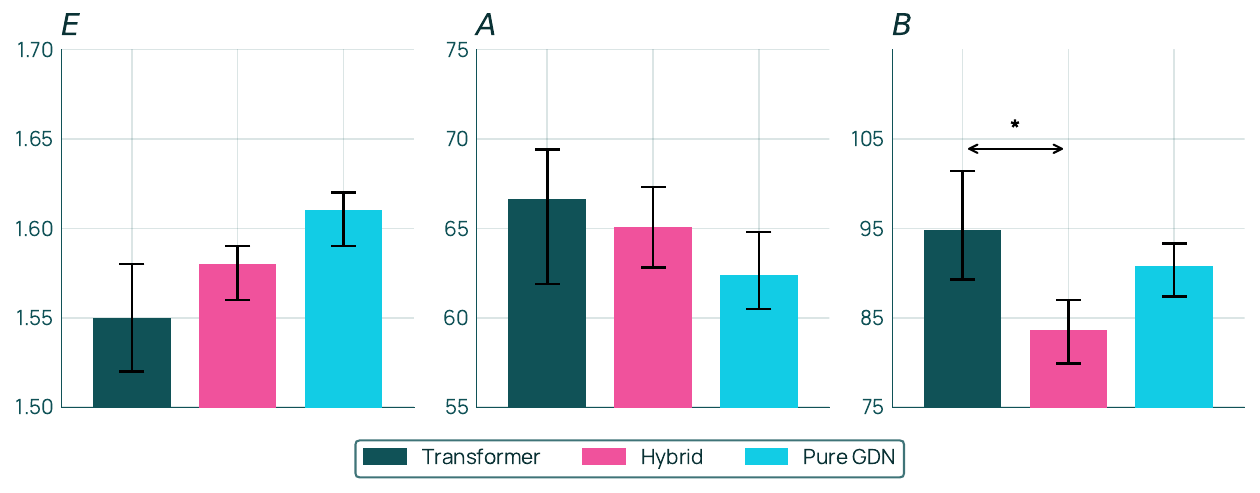}

\caption{Scaling law coefficients $E$, $A$, $B$ with 95\% bootstrap CIs for the fixed-exponent fit ($\alpha{=}\beta{=}0.22$ shared), which yields tight confidence intervals.
The data efficiency coefficient $B$ shows a robust advantage for Hybrid GDN over the transformer ($83.7$ vs.\ $94.9$, non-overlapping CIs, marked~$*$), while differences in $E$ and $A$ are not statistically conclusive.}
\label{fig:scaling_params_bar}
\end{figure}

\paragraph{Token Savings Grow Steadily with Model Size.}
The improved scaling trend of hybrid models naturally translates to better scaling of the loss with compute.
The projected savings are derived by analytically inverting the fitted scaling law: for a target loss $L^*$, the number of tokens required by a model with $N$ parameters is $D^*(N) = \bigl(B / (L^* - E - A/N^\alpha)\bigr)^{1/\beta}$, evaluated at each model size using each architecture's fitted coefficients (see \Cref{sec:fig-table-details} for full details).
While the pure GDN achieves the lowest estimated loss at small parameter budgets, the hybrid model achieves better loss at the more typical medium and large scales, where its advantage in the data-efficiency coefficient $B$ becomes the dominant and most reliable signal.
\Cref{fig:combined-savings} visualizes the projected parameter and data savings factors across model sizes for a target loss of $2.474$ (the minimum observed training loss); the token requirements are obtained by analytically inverting the fitted scaling law as described in \Cref{sec:fig-table-details}.
The parameter savings factor (cf.\ \Cref{fig:combined-savings}(a)) reveals that transformers are more parameter-efficient at larger token budgets, a consequence of the larger scaling factor $\alpha$.
\Cref{fig:combined-savings}(b), in contrast, shows that the \emph{token} savings factor grows steadily with model size, rising from ${\sim}1.3\times$ at 1B parameters to ${\sim}1.9\times$ at 70B parameters.
Concretely, to match a transformer trained at a given scale, one can train a hybrid model of the same size on ${\sim}1.3$--$1.9\times$ fewer tokens (e.g., $1.68\times$ at 7B, $1.82\times$ at 30B, $1.89\times$ at 70B; see \Cref{fig:combined-savings}(b) and \Cref{tab:token-savings-by-scale}).
The pure GDN model, by contrast, initially requires \emph{more} data than the transformer at the 1B scale ($0.82\times$, i.e., no savings) before catching up and achieving savings of ${\sim}1.7\times$ at 70B.

\paragraph{Compute Savings Also Grow Steadily with Model Size.}
Token savings directly lead to improved compute efficiency.
\Cref{fig:combined-savings}(c) visualizes this by showing the compute savings factor needed for reaching a target loss, revealing steadily growing savings.
At $10^{22}$ FLOPs, for example, the hybrid model is projected to achieve a loss of $2.267$ compared to $2.308$ for the transformer ($\Delta = -0.042$, 95\% CI $[-0.14, +0.06]$); full results across compute budgets from $10^{18}$ to $10^{23}$ FLOPs are reported in \Cref{tab:compute-equivalent-loss} in the appendix.

\begin{figure*}[t!]
\centering
\includegraphics[width=\textwidth]{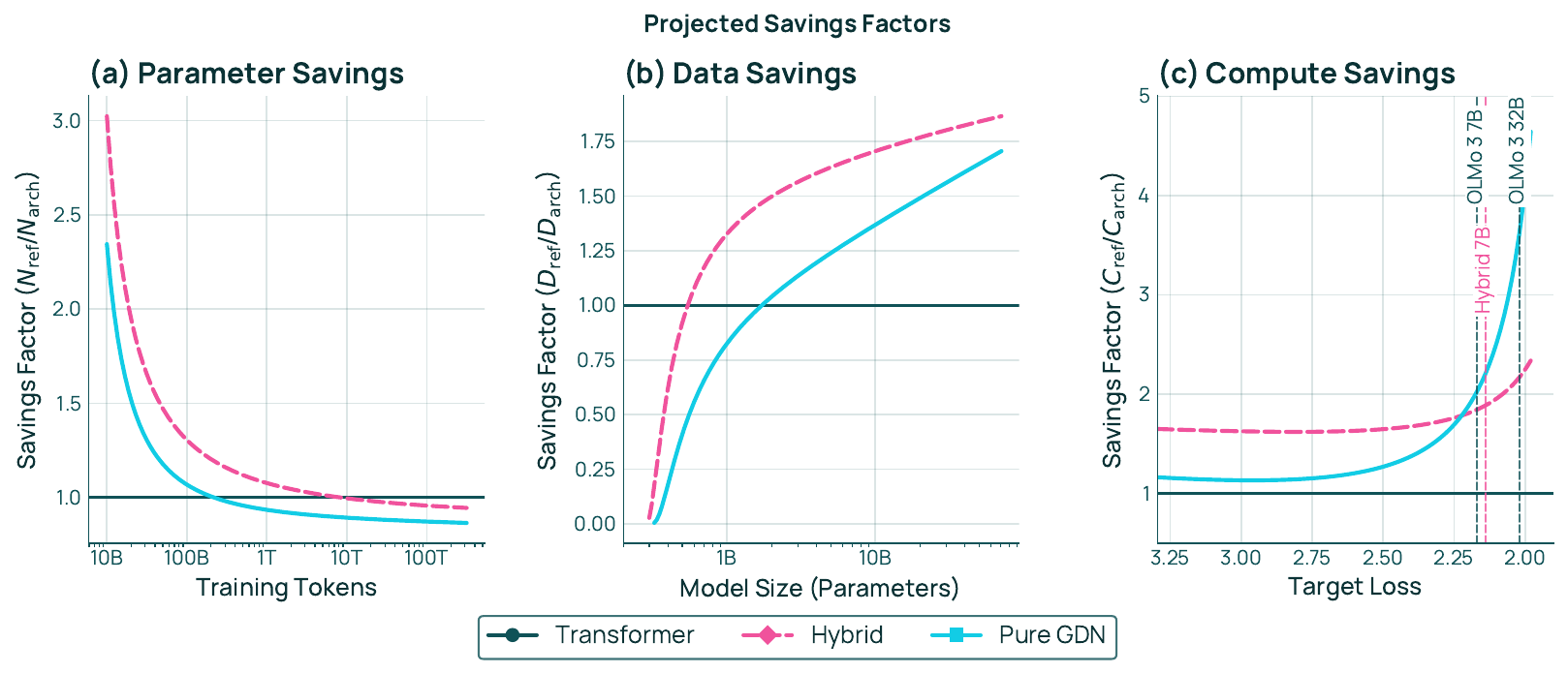}
\caption{Projected savings factors across model scales (target loss $= 2.474$, selected as the minimum of all observed training losses). \textbf{(a)} Parameter savings ($N_\mathrm{ref}/N_\mathrm{arch}$) vs training tokens: fewer parameters are needed to reach the same loss at a given data budget. \textbf{(b)} Data savings ($D_\mathrm{ref}/D_\mathrm{arch}$) vs model size: fewer training tokens are needed to reach the target loss at a given model size. \textbf{(c)} Compute savings ($C_\mathrm{ref}/C_\mathrm{arch}$, $C \propto N{\cdot}D$) vs target loss: for each target loss the minimum total compute is found by optimising over model size $N$; harder targets (lower loss) are on the right. Dashed vertical lines mark the observed losses of our final production models. The savings factor grows with difficulty for both Pure GDN and Hybrid models. Values above $1\times$ indicate an advantage over Transformer. Note that estimates at lower loss values (right side of the plot) involve extrapolation beyond the fitting range and are therefore less certain.}
\label{fig:combined-savings}
\end{figure*}

\paragraph{Implications.}
The derived scaling laws suggest favorable scaling for hybrid models compared to both pure transformers and purely linear RNN models.
The most statistically robust finding, from the fixed-exponent analysis, is that the hybrid model has a meaningfully lower data coefficient $B$ (83.7 vs.\ 94.9 for the transformer, non-overlapping 95\% CIs), indicating that hybrid models are more data-efficient learners.
The parameter coefficient $A$ also slightly favors the hybrid, though this difference is less statistically clear.
In \Cref{sec:expressivity-to-scaling}, we argue that the improved $B$ is consistent with the expressivity-grounded prediction that more expressive architectures can learn a larger proportion of discrete tasks from data, making each training token more valuable.
The scaling exponents $\alpha$ and $\beta$ are statistically indistinguishable across architectures---consistent with the theoretical prediction that expressivity shifts the constant-factor efficiency without altering the power-law exponents.
Together, these results suggest that the primary practical benefit of hybrid models for scaling is more efficient use of training data, which is consistent with our pretraining observations (cf.\ \Cref{fig:pretraining-perf}) and the projected ${\sim}1.9\times$ data savings at 70B scale (cf.\ \Cref{fig:combined-savings}(b) and \Cref{tab:token-savings-by-scale}).
We study additional architecture ablations in \Cref{sec:architecture-ablations}, where we justify choosing GDN for the linear component of \model and the specific hybridization strategy.

\subsection{Theory: Increased Expressive Power Improves Scaling} \label{sec:expressivity-to-scaling}

In \Cref{sec:scaling-laws}, we saw that the hybrid model scaled more efficiently during pretraining than the transformer, achieving better performance with the same token or compute budget.
At first glance, it may be unclear why pretraining efficiency should be related to the greater expressivity of hybrid models~(\Cref{sec:expressivity}): after all, our expressivity guarantees imply a binary difference in abilities on synthetic tasks, whereas scaling laws concern smooth improvements on modeling natural-language data.
However, we now formalize a plausible explanation for why the greater expressivity of hybrid models should translate to improved scaling behavior.
In particular, recent work explains neural scaling laws as emerging from the gradual aggregation of many discrete tasks reflected in language modeling data.
Working within this framework, we show that increasing expressivity improves scaling trends (\Cref{cor:efficient-scaling,cor:irreducible-loss}) because it increases the number of discrete tasks that a model can learn on a fixed parameter and token budget. In aggregate, this means that more expressive models can achieve lower loss on the same budget, as shown in \Cref{fig:scaling-theory-example}.

\paragraph{Quantization Model {\normalfont\slightlylarger\citep{michaud2023quantization}}.}
It is well established that language modeling loss decreases smoothly with model size and token budget despite the fact that many properties of language are discrete \citep{kaplan2020scalinglaws,hoffmann2022chinchilla}.
One explanation is that language modeling is fundamentally a multi-task problem: LMs acquire individual tasks discretely, leading to a smooth reduction of aggregate loss \citep[inter alia]{hutter2021learningcurvetheory,arora2023theoryemergence,nam2024exactly}.
In particular, this intuition has recently been formalized in the \emph{quantization model} of neural scaling laws \citep{michaud2023quantization}, a minimal framework for deriving LM scaling laws.
We encapsulate the quantization model via the following basic assumptions:
\begin{enumerate}
    \item Language modeling consists of a large number of discrete tasks (originally called ``quanta''), each of which is either unlearned or learned.

    \item The distribution of tasks in the data follows a power law, resembling the Zipfian distribution of word types in natural language.
    Each token leverages exactly one task, and the probability of a token (in the training or test data) leveraging task $k$ is $p_k \propto k^{-(\alpha + 1)}$ for $\alpha > 0$.

    \item During training, task $k$ becomes learnable once a critical threshold $T_k$ of relevant tokens leveraging the task have been observed.
    An LM learner acquires learnable tasks incrementally, spending parameters $C_k$ on each one, in order of their rank $k$ (formalized in \Cref{def:quantization-model} in \Cref{sec:scaling-laws-proofs}).
    Once an LM learns a task, it predicts tokens leveraging that task with lower loss than before.
\end{enumerate}
Under these assumptions, the loss $\tilde L$ depends on the number of tasks learned, which is controlled by the budget for parameters $N$ or training tokens $D$.
\citet{michaud2023quantization} show that loss will follow a smooth power law as a function of tasks, parameters, or tokens, resulting from learning many discrete tasks one by one, with each successive task having smoothly decaying probability in the data.

\paragraph{Incorporating Expressivity.}
We now adapt the quantization model to study how the expressive power of an LM impacts scaling efficiency.
\Cref{sec:expressivity} analyzed which computational tasks can be expressed by hybrid models and transformers; now we consider the \emph{multi-task} setting, where achieving low loss requires learning many individual tasks.
We first stipulate that each of these tasks is expressible or inexpressible by our LM, independent of its frequency in the data:

\begin{restatable}[Only Some Tasks Are Expressible]{assumption}{assmExpressibility} \label{assm:expressible}
    For a given architecture, each task is either expressible or inexpressible,
    modeled as an \textit{iid} boolean random variable where the probability that
    any individual task is expressible is $1 - \epsilon$.
\end{restatable}

Increasing expressive power (e.g., going from a transformer to a hybrid model) corresponds to decreasing $\epsilon$.
We will analyze the \emph{expected loss} under \Cref{assm:expressible}.
Conceptually, increasing a model's expressivity ($1 - \epsilon$) could improve scaling efficiency because it increases the proportion of tasks that a model can learn efficiently, which is the mechanism by which loss decreases.
There are two plausible mechanisms: first, the model might fail to learn inexpressible tasks entirely, or, second, the model might approximate them but require more parameters and data because the architecture does not admit a compact subnetwork for solving the task.
We formalize a unified \emph{expressivity-aware} extension to the quantization model that allows for either (or both) of these effects of expressivity.
First, we formalize the fact that the loss achieved on a task after it has been learned may depend on whether the task is expressible:

\begin{restatable}[Expressible Tasks Can Be Learned to Lower Loss]{assumption}{assmLossReduction} \label{assm:loss-incurred}
    Tokens leveraging unlearned tasks incur loss $L_0$.
    If a task $k$ has been learned,
    the loss incurred by the learner on a token leveraging $k$ is reduced to $L_0 - \Delta$ if $k$ is expressible by that learner and to $L_0 - \Delta'$ if $k$ is inexpressible,
    for some $\Delta, \Delta' > 0$ and $\Delta' \leq \Delta$.
\end{restatable}

We will refer to $\Delta$ and $\Delta'$ as \emph{loss reductions} for expressible and inexpressible tasks, respectively.
When $\Delta' = 0$, inexpressible tasks cannot be learned at all (as briefly explored by \citealp{michaud2025quanta}), and when $\Delta' = \Delta$, they can be fully learned, though they might require more resources (parameters and data) to learn.

We now formalize the resource requirements of different tasks.
In the original quantization model \citep{michaud2023quantization}, all tasks require the same number of parameters $C$ to represent and have the same sample complexity $T$.
In contrast, we now imagine that the number of parameters and tokens needed per task can change based on whether the task is expressible:

\begin{restatable}[Expressible Tasks Can Be Learned with Fewer Parameters and Tokens]{assumption}{assmResources} \label{assm:less-succint}
    Each expressible task is representable with $C$ parameters and learnable with $T$ relevant tokens.
    An inexpressible task requires $C' \geq C$ parameters to represent approximately and is learnable with $T' \geq T$ relevant tokens.
\end{restatable}

\Cref{assm:less-succint} can be motivated as follows.
If a task is fully expressible, a \emph{small} subnetwork exists that solves the task across all inputs, requiring a reasonable number of samples to learn.
If the task is inexpressible, no such small subnetwork exists, but the model can still learn a \emph{large} lookup table that approximates the function over common inputs, requiring more samples to cover every case.
This connection between expressiveness and succinctness of representation recalls philosophical arguments in theoretical computer science for analyzing computation in terms of infinite formal languages \citep{savitch1993mightpay}.
It is natural to imagine that expressibility leads to an exponential reduction in required resources (i.e., $C' = 2^C$), though we do not need to commit to this.
The case where $C' = C$ recovers the case where inexpressible tasks do not require more parameters.

Having extended the quantization model to be expressivity-aware, we now characterize how expressivity affects scaling laws under these assumptions.
The following general result establishes that increasing expressiveness leads to more parameter- and data-efficient scaling in the quantization model, and can also improve the irreducible loss achievable with infinite parameters and data.

Let $\trueloss(N)$ and $\trueloss(D)$ denote the actual loss incurred under the quantization model as a function of the number of parameters $N$ and the token budget $D$.

\begin{restatable}[Expressivity-Aware Scaling Laws]{theorem}{scalingEfficiency} \label{thm:scaling-efficiency}
    Consider the quantization model of neural scaling laws augmented by \Cref{assm:expressible,assm:loss-incurred,assm:less-succint}.
    Then, as a function of the number of parameters $N$, the loss $\trueloss(N)$ is closely approximated by a power law $L(N)$, i.e., $\trueloss(N) \approx L(N)$ where:
    \begin{equation*}
        L(N) - L^\epsilon_\infty \propto A_\epsilon \cdot N^{-\alpha} ,
        \quad \quad \textrm{where} \;
        A_\epsilon = (L_0 - L^\epsilon_\infty) \cdot \left( C + \epsilon (C' - C) \right)^\alpha .
    \end{equation*}
    Similarly, as a function of the token budget $D$, the loss $\trueloss(D)$ is closely approximated by a power law $L(D)$, i.e., $\trueloss(D) \approx L(D)$ where
    \begin{equation*}
        L(D) - L^\epsilon_\infty \propto B_\epsilon \cdot D^{-\alpha/(\alpha + 1)} ,
        \quad \quad \textrm{where} \;
        B_\epsilon = (1 - \epsilon) \Delta T^{\alpha/(\alpha+1)} + \epsilon \Delta' T'^{\alpha/(\alpha+1)} .
    \end{equation*}
    Finally, the irreducible loss $L^\epsilon_\infty$ for both power laws depends on expressivity via
    \begin{equation*}
        L^\epsilon_\infty = L_0 - (1 - \epsilon) \Delta - \epsilon \Delta' .
    \end{equation*}
\end{restatable}

See \Cref{sec:scaling-laws-proofs} for a proof.
\Cref{thm:scaling-efficiency} can be interpreted as follows.
If we can express all tasks ($\epsilon = 0$), we recover the standard irreducible loss $L^0_\infty = L_0 - \Delta$ and scaling coefficients $A_0 = \Delta C^\alpha$ and $B_0 = \Delta T^{\alpha/(\alpha + 1)}$ from the quantization model \citep{michaud2023quantization}.
On the other hand, if there are some tasks we cannot express ($\epsilon > 0$), the scaling coefficients $A_\epsilon$ and $B_\epsilon$ (and potentially irreducible loss $L^\epsilon_\infty$) will change.
In particular, as illustrated in \Cref{fig:scaling-theory-example}, decreasing $\epsilon$ \emph{improves} loss towards the $\epsilon = 0$ case.
Formally, under any non-trivial instantiation of \Cref{assm:expressible,assm:loss-incurred,assm:less-succint}, increasing expressivity shifts the loss curve down:

\begin{figure}[t]
    \centering
    \includegraphics[width=\linewidth]{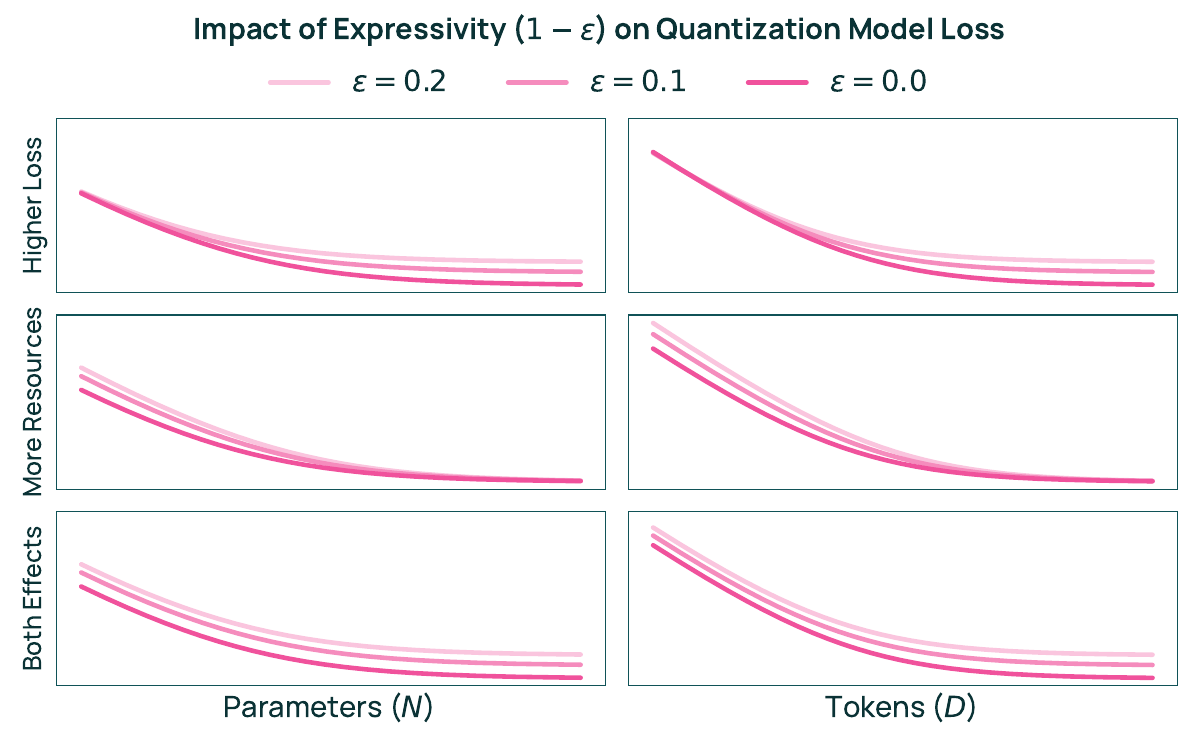}
    \caption{
        Impact of expressivity $1 - \epsilon$ on loss under three instantiations of the expressivity-aware quantization model; both axes are log-scaled.
        In the first row, $\Delta' < \Delta$, so decreasing $\epsilon$ lowers irreducible loss~(\Cref{cor:irreducible-loss}) and reduces loss everywhere~(\Cref{cor:efficient-scaling}), but especially for large $N$ and $D$.
        In the second row, $\Delta' = \Delta$, but inexpressible tasks require more parameters and tokens than expressible tasks.
        Thus, while increasing expressivity shifts down the scaling curve everywhere, irreducible loss remains the same, so for large enough $N, D$, the loss reduction diminishes.
        Finally, the third row incorporates both effects of expressivity. This means decreasing $\epsilon$ both causes an initial difference in loss and lowers irreducible loss, leading to a visible gap in loss as $\epsilon$ decreases across values of $N$ and $D$.
    }
    \label{fig:scaling-theory-example}
\end{figure}

\begin{restatable}[Expressivity Always Improves Scaling]{corollary}{expressivityAlwaysImprovesScaling} \label{cor:efficient-scaling}
    Fix a nontrivial instantiation of \Cref{assm:expressible,assm:loss-incurred,assm:less-succint}, i.e., where either $\Delta' < \Delta$, 
    or $C' > C$ and $T' > T$.
    Then, both $L(N)$ and $L(D)$ strictly decrease as $\epsilon$ decreases, elementwise for all $N, D$.
\end{restatable}

See justification in \Cref{sec:main-scaling-results}.
\Cref{cor:efficient-scaling} predicts that increased expressive power will translate into more efficient model scaling.
Conceptually, we can understand two separate beneficial effects, coming from \Cref{assm:loss-incurred,assm:less-succint}, respectively.
First, if we instantiate \Cref{assm:loss-incurred} such that inexpressible tasks achieve a strictly lower loss reduction ($\Delta' < \Delta$), the \emph{irreducible loss} $L^\epsilon_\infty$ part of the scaling laws is affected:

\begin{restatable}[Expressivity Can Shift Irreducible Loss]{corollary}{irreducibleLoss} \label{cor:irreducible-loss}
    If and only if $\Delta' < \Delta$, irreducible loss $L^\epsilon_\infty$ strictly decreases as $\epsilon$ decreases.
\end{restatable}

As visualized in the top row of \Cref{fig:scaling-theory-example}, this leads to a clear difference in loss for large $N$ and $D$ (when the irreducible loss starts to dominate), even though the loss curves start more similarly for small $N$ and $D$.
In contrast, if \Cref{assm:loss-incurred,assm:less-succint} are instantiated so that inexpressible tasks can be fully learned $(\Delta' = \Delta)$ but require more parameters and tokens ($C' > C$, $T' > T$), the picture is different: we see dramatic differences in loss for small values of $N$ and $D$ based on $\epsilon$, but these differences go to 0 in the limit of large $N$ and $D$, as illustrated in the middle row of \Cref{fig:scaling-theory-example}.
Finally, if we combine both assumptions, i.e., inexpressible tasks achieve a lower loss reduction ($\Delta' < \Delta$) and also require more parameters and tokens to learn, we see a gap between the loss curves for all values of $N$ and $D$, as seen in the bottom row of \Cref{fig:scaling-theory-example}.

A natural question is which instantiation of \Cref{assm:expressible,assm:loss-incurred,assm:less-succint} is most consistent with the empirical scaling laws from \Cref{sec:scaling-laws}.
Based on \Cref{cor:irreducible-loss}, lower irreducible loss for more expressive architectures supports a model where expressible tasks have greater loss reduction, i.e., $\Delta' < \Delta$.
Because we do not find clear empirical evidence in \Cref{sec:scaling-laws} that architecture affects irreducible loss, it appears that a quantization model with $\Delta' = \Delta$ may better fit the data, though more precise estimates of the scaling law parameters could change this conclusion.
Further, based on \Cref{cor:efficient-scaling}, a model where expressivity does not improve parameter efficiency requires $C' = C$ and $\Delta' = \Delta$.
Thus, a quantization model with $T' > T$ but $C' = C$ and $\Delta' = \Delta$ appears most consistent with the observed scaling behavior of hybrid models vs.~transformers, though future work that more precisely estimates the scaling coefficients (in particular, the irreducible loss) for these architectures could change these conclusions.

\subsection{Discussion: From Expressivity to Scaling}

Through controlled scaling studies across model sizes and data budgets, we showed our hybrid architecture can attain better data and parameter efficiency compared to the transformer baseline.
This matches the improved pretraining efficiency (both in terms of tokens and compute spent) on Common Crawl, MMLU, and other evaluations that we observed for the \model pretraining run compared to Olmo 3.
In \Cref{sec:expressivity-to-scaling}, we argued that the pretraining efficiency of hybrid models could be explained by their greater expressivity relative to transformers, as standard explanations of scaling laws can be minimally extended to predict that more expressive models should have better loss curves.

In particular, \Cref{thm:scaling-efficiency} provides a conceptual explanation for why more expressive architectures might scale better.
In line with our empirical findings (\Cref{sec:scaling-laws}), greater expressivity improves loss across the scaling curve under our formal model but does not affect the scaling law exponent, which depend only on $\alpha$, which parameterizes the task power law underlying the data distribution.
Nevertheless, expressivity reduces loss across the loss curve, in line with the comparison of hybrid models and transformers in \Cref{sec:scaling-laws}.
As mentioned earlier, some instantiations of the quantization model (cf.~\Cref{assm:loss-incurred,assm:less-succint}) also predict improved irreducible loss and parameter efficiency for more expressive models, which we do not find clear evidence for in \Cref{sec:scaling-laws}.

Interestingly, and perhaps counterintuitively, even tasks that were already expressible by an architecture can be learned faster when expressivity is increased.
Under \Cref{assm:less-succint}, learning inexpressible tasks more efficiently frees up parameters than can be allocated to learning (potentially already expressible) tasks faster.
This mirrors results reported by \citet{hu2025circuits}, where allowing transformers to learn syntax more efficiently seemed to enable more efficient learning of other aspects of language.

We emphasize that our theoretical results were obtained in the simplified quantization model of scaling laws \citep{michaud2023quantization}, rather than by analyzing the actual learning dynamics of LMs.
While this setting is somewhat simplistic, it nicely captures fundamental properties of language modeling such as its multi-task data distribution.
We thus take it to provide a plausible conceptual explanation for the link between expressivity and scaling improvements,
though we caution against reading too much into the precise quantitative predictions such as the scaling law coefficients and exponents (cf.~\citealp{michaud2025quanta}).
There is ample opportunity for more in-depth theoretical work to expand the analysis presented here to more realistic models of training, as well as empirical work that tests predictions of different theory variants.

While expressivity is one clear advantage of GDN-based hybrid models, other factors may also contribute to the performance gains, such as increased training stability, which we explore in \Cref{sec:pretraining}.
In \Cref{sec:architecture-ablations}, we see that GDN hybrid models shows better scaling trends than hybrid models using Mamba, whose expressivity is constrained to $\TC^0$, in line with the hypothesized link between expressivity and scaling.
However, GDN without the negative eigenvalue extension (which is thought to be less expressive; cf.~\Cref{sec:rnn-expressivity}) shows very similar scaling trends to GDN with negative values, in potential disagreement with theory.
This could suggest that some other benefit of GDN beyond expressivity explains its improved scaling relative to transformers, though it is also an open question whether GDN, even without negative eigenvalues, could have expressivity advantages over transformers such as the ability to solve some $\NC^1$-complete problems.

\section{Other Research Questions} \label{sec:other-questions}

\subsection{RNN and Hybridization Architecture Choices} \label{sec:architecture-ablations}

To decide on the final hybrid architecture for \model, we ran a series of ablation experiments evaluating the performance of different sequence mixers and hybridization strategies.
In particular, we were interested in three key design questions:
\begin{enumerate}
    \item \textbf{RNN Architecture:} Which popular linear RNN architecture (GDN vs.\ Mamba2) results in the best hybrid model?
    \item \textbf{Layer Placement:} How should hybridization be performed---should attention layers be interleaved uniformly throughout the model, or concentrated in specific regions?\footnote{We only consider inter-layer hybridization and did not test intra-layer hybridization strategies such as those explored in \citet{ren2025samba}.}
    \item \textbf{Attention Ratio:} What is the optimal ratio of attention-to-RNN layers?
\end{enumerate}
To inform the final model choice and provide guidelines for future work on hybrid models, we evaluated a comprehensive suite of pure and hybrid architectures at multiple scales, illustrated in \Cref{fig:architecture-ablations}.

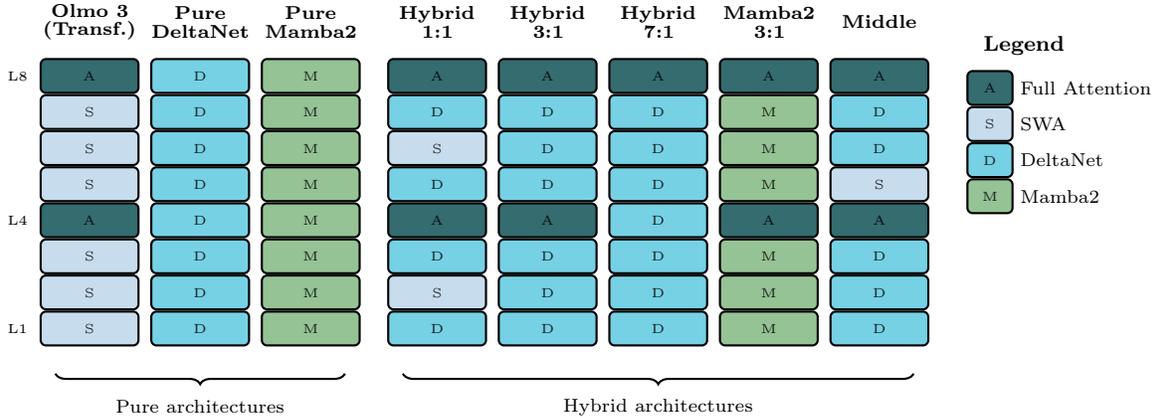
\begin{figure}[t]
    \centering
    \usetikzlibrary{decorations.pathreplacing}
\begin{tikzpicture}[
    xscale=0.7,
    yscale=0.8,
    layer/.style={minimum width=1.3cm, minimum height=0.45cm, draw, thick, font=\tiny, rounded corners=2pt},
    attn/.style={layer, fill=transformercolor, fill opacity=0.85},
    swa/.style={layer, fill=swacolor, fill opacity=0.85},
    deltanet/.style={layer, fill=deltanetcolor, fill opacity=0.85},
    mamba/.style={layer, fill=mambacolor, fill opacity=0.85},
    modellabel/.style={font=\scriptsize\bfseries, align=center, text width=1.5cm},
    grouplabel/.style={font=\scriptsize, align=center},
    bracestyle/.style={decorate, decoration={brace, amplitude=4pt, mirror}},
]

\def\colA{0}      
\def\colB{2.1}    
\def\colC{4.2}    
\def\colD{6.6}    
\def\colE{8.7}    
\def\colF{10.8}   
\def\colFb{12.9}  
\def\colG{15.0}   

\def\layerheight{0.6}
\def\nlayers{8}

\node[modellabel] at (\colA, \nlayers*\layerheight + 0.9) {Olmo 3\\[-1pt](Transf.)};
\node[modellabel] at (\colB, \nlayers*\layerheight + 0.9) {Pure\\[-1pt]DeltaNet};
\node[modellabel] at (\colC, \nlayers*\layerheight + 0.9) {Pure\\[-1pt]Mamba2};
\node[modellabel] at (\colD, \nlayers*\layerheight + 0.9) {Hybrid\\[-1pt]1:1};
\node[modellabel] at (\colE, \nlayers*\layerheight + 0.9) {Hybrid\\[-1pt]3:1};
\node[modellabel] at (\colF, \nlayers*\layerheight + 0.9) {Hybrid\\[-1pt]7:1};
\node[modellabel] at (\colFb, \nlayers*\layerheight + 0.9) {Mamba2\\[-1pt]3:1};
\node[modellabel] at (\colG, \nlayers*\layerheight + 0.9) {Middle};

\node[swa]  at (\colA, 1*\layerheight) {S};
\node[swa]  at (\colA, 2*\layerheight) {S};
\node[swa]  at (\colA, 3*\layerheight) {S};
\node[attn] at (\colA, 4*\layerheight) {A};
\node[swa]  at (\colA, 5*\layerheight) {S};
\node[swa]  at (\colA, 6*\layerheight) {S};
\node[swa]  at (\colA, 7*\layerheight) {S};
\node[attn] at (\colA, 8*\layerheight) {A};

\foreach \i in {1,...,8} {
    \node[deltanet] at (\colB, \i*\layerheight) {D};
}

\foreach \i in {1,...,8} {
    \node[mamba] at (\colC, \i*\layerheight) {M};
}

\node[deltanet] at (\colD, 1*\layerheight) {D};
\node[swa]      at (\colD, 2*\layerheight) {S};
\node[deltanet] at (\colD, 3*\layerheight) {D};
\node[attn]     at (\colD, 4*\layerheight) {A};
\node[deltanet] at (\colD, 5*\layerheight) {D};
\node[swa]      at (\colD, 6*\layerheight) {S};
\node[deltanet] at (\colD, 7*\layerheight) {D};
\node[attn]     at (\colD, 8*\layerheight) {A};

\node[deltanet] at (\colE, 1*\layerheight) {D};
\node[deltanet] at (\colE, 2*\layerheight) {D};
\node[deltanet] at (\colE, 3*\layerheight) {D};
\node[attn] at (\colE, 4*\layerheight) {A};
\node[deltanet] at (\colE, 5*\layerheight) {D};
\node[deltanet] at (\colE, 6*\layerheight) {D};
\node[deltanet] at (\colE, 7*\layerheight) {D};
\node[attn] at (\colE, 8*\layerheight) {A};

\foreach \i in {1,...,7} {
    \node[deltanet] at (\colF, \i*\layerheight) {D};
}
\node[attn] at (\colF, 8*\layerheight) {A};

\node[mamba] at (\colFb, 1*\layerheight) {M};
\node[mamba] at (\colFb, 2*\layerheight) {M};
\node[mamba] at (\colFb, 3*\layerheight) {M};
\node[attn] at (\colFb, 4*\layerheight) {A};
\node[mamba] at (\colFb, 5*\layerheight) {M};
\node[mamba] at (\colFb, 6*\layerheight) {M};
\node[mamba] at (\colFb, 7*\layerheight) {M};
\node[attn] at (\colFb, 8*\layerheight) {A};

\node[deltanet] at (\colG, 1*\layerheight) {D};
\node[deltanet] at (\colG, 2*\layerheight) {D};
\node[deltanet] at (\colG, 3*\layerheight) {D};
\node[attn]     at (\colG, 4*\layerheight) {A};
\node[swa]      at (\colG, 5*\layerheight) {S};
\node[deltanet] at (\colG, 6*\layerheight) {D};
\node[deltanet] at (\colG, 7*\layerheight) {D};
\node[attn]     at (\colG, 8*\layerheight) {A};

\draw[bracestyle, thick] (\colA-0.65, -0.15) -- (\colC+0.65, -0.15);
\node[grouplabel] at (2.1, -0.7) {Pure architectures};

\draw[bracestyle, thick] (\colD-0.65, -0.15) -- (\colG+0.65, -0.15);
\node[grouplabel] at (10.8, -0.7) {Hybrid architectures};

\begin{scope}[shift={(16.8, 3.5)}]
    \node[anchor=west, font=\footnotesize\bfseries] at (0, 1.8) {Legend};

    \node[attn, minimum width=0.6cm] at (0.3, 1.1) {A};
    \node[anchor=west, font=\scriptsize] at (0.7, 1.1) {Full Attention};

    \node[swa, minimum width=0.6cm] at (0.3, 0.5) {S};
    \node[anchor=west, font=\scriptsize] at (0.7, 0.5) {SWA};

    \node[deltanet, minimum width=0.6cm] at (0.3, -0.1) {D};
    \node[anchor=west, font=\scriptsize] at (0.7, -0.1) {DeltaNet};

    \node[mamba, minimum width=0.6cm] at (0.3, -0.7) {M};
    \node[anchor=west, font=\scriptsize] at (0.7, -0.7) {Mamba2};
\end{scope}

\node[font=\tiny, anchor=east] at (\colA-1.0, 1*\layerheight) {L1};
\node[font=\tiny, anchor=east] at (\colA-1.0, 4*\layerheight) {L4};
\node[font=\tiny, anchor=east] at (\colA-1.0, 8*\layerheight) {L8};

\end{tikzpicture}
    \caption{
        \textbf{Architecture configurations evaluated in our ablation study.}
        We compare pure architectures (Transformer, GDN, Mamba2) against hybrid variants with different linear-to-attention ratios (1:1, 3:1, 7:1), RNN backbones (GDN vs.\ Mamba2 at 3:1), and placement strategies (interleaved vs.\ middle).
        The highlighted configuration (3:1 interleaved with GDN) was selected for the final \model.
    }
    \label{fig:architecture-ablations}
\end{figure}

\paragraph{Experimental Setup.}
All ablation experiments follow the same setup as \Cref{sec:scaling-laws} and use identical training data, optimizer settings, and evaluation protocols, varying only the architecture.
We train models at the same seven scales: $60\text{M}$, $100\text{M}$, $190\text{M}$, $370\text{M}$, $600\text{M}$, $760\text{M}$, and $1\text{B}$ parameters, and each model is trained on $8\times$ Chinchilla-optimal tokens using a WSD-S learning rate schedule.
\Cref{tab:ablation-configs} lists the detailed configurations for each model size.
All hybrid models were implemented with the Flash Linear Attention library based on the Olmo 3 configuration---the only difference is the implementation of the sequence mixing component.\footnote{Thus, the position-wise MLPs were identical between the architectures. This aligns well with the GDN architecture but is slightly non-standard for Mamba2, which, as a Gated Linear Unit, contains non-linear activations in the gating mechanism. However, we found Mamba2 models without additional MLPs to perform worse, which is why we included them in the implementations used in this section.}

As in \Cref{sec:scaling-laws}, we measure performance as the average validation loss, computed as the average cross-entropy across 11 held-out evaluation domains: C4, Dolma Books, Dolma Common Crawl, Dolma pes2o, Dolma Reddit, Dolma Stack, Dolma Wiki, ICE, M2D2 S2ORC, Pile, and WikiText-103. Additionally, we evaluate the models on the \olmothreeeval Easy suite in bits per byte (BPB), reporting the average score across Math, Code, and general reasoning tasks.
We again compute scaling law coefficients; in \Cref{fig:base_easy_avg}, the curves represent the projected loss at each FLOP value at the optimal parameter and token counts, computed by fitting scaling laws analogously to \Cref{sec:scaling-laws}.
See also \Cref{sec:fig-table-details} for more details.

\Cref{fig:base_easy_avg} visually shows the performance of the tested architectures and \Cref{tab:ablation-evals-avg} presents the average performance on the \olmothreeeval Easy suite.
The rest of the section interprets these results to justify the final \model architecture.

\paragraph{RNN Architecture: GDN vs.\ Mamba2.}
We compare the two linear RNN architectures in both pure and hybrid (3:1 interleaved) settings against the transformer baseline.
Pure Mamba2 performed worse than both the transformer and GDN at most scales---particularly larger ones---with average BPB of 0.72 for Mamba2 at the 1B scale compared to 0.68 for the transformer (\Cref{tab:ablation-evals-avg}).
The hybrid Mamba2 configuration (Mamba2 + 3:1 Attn) improved over pure Mamba2 and outperformed the transformer at most scales, but fell behind at 1B (0.698 vs.\ 0.682).
In contrast, pure GDN outperformed the transformer at all scales (e.g., 0.722 vs.\ 0.747 at 760M; 0.677 vs.\ 0.682 at 1B), and the hybrid GDN 3:1 configuration improved further, achieving the best overall results (0.717 at 760M; 0.669 at 1B).
While hybrid Mamba2 is competitive with the transformer, GDN-based models consistently outperform it across Math, Code, and QA domains (\Cref{tab:ablation-evals-detail-a,tab:ablation-evals-detail-b}), confirming GDN as the right choice for the linear component of \model.

\paragraph{Layer Placement: Interleaved vs.\ Middle.}
We compare two hybridization strategies that use the same 3:1 linear-to-attention ratio but place the attention layers differently:
\begin{itemize}
    \item \textbf{Interleaved:} Attention layers placed at regular intervals (every 4th layer).
    \item \textbf{Middle:} Attention layers concentrated in the middle of the network, with an additional attention layer at the final position.
\end{itemize}
Despite having very similar parameter counts (948M vs.\ 932M at 760M; \Cref{tab:ablation-configs}), the interleaved configuration consistently outperforms the middle placement, with the gap growing at larger scales.
Both configurations outperform the transformer baseline, confirming that the benefit of hybridization is robust to placement strategy, though interleaving appears preferable.
One interpretation of the advantage of interleaved placement is that uniformly distributing attention layers allows every part of the network to access global context, whereas concentrating attention in the middle creates a bottleneck through which all ``attention-relevant information'' must pass.
Additionally, connecting to our theoretical results in \Cref{sec:hybrid-expressivity}, interleaving maximizes the number of alternations between layer types, which may unlock additional expressive power: for example, \emph{stacked} compositions of tasks like state-based recall (\Cref{thm:pointer-recall}) could benefit from multiple rounds of alternation between state tracking and recall.
While \Cref{thm:padded-hybrid} shows that a single alternation suffices with padding tokens, in practice, multiple alternations without padding may be more effective.

\paragraph{Linear-to-Attention Ratio: 1:1 vs.\ 3:1 vs.\ 7:1.}
We vary the fraction of attention layers from 50\% down to 12.5\%, using interleaved placement with GDN throughout.
At the smallest scales (60M--190M parameters), the 7:1 ratio (12.5\% attention) tends to perform best or on par with 3:1, suggesting that fewer attention layers can suffice when models are small.
At larger scales (600M--1B), however, the 3:1 ratio consistently achieves the best or second-best performance across all domains (\Cref{tab:ablation-evals-detail-b}), while the 7:1 ratio falls slightly behind despite having more parameters.
The 1:1 ratio (50\% attention) performs comparably to 3:1 at small scales but underperforms at larger scales, suggesting that the higher computational cost of more attention layers does not pay off relative to the GDN layers.
Overall, the 3:1 ratio (25\% attention) offers the best trade-off across scales and domains, and we select it for \model.
Notably, all three interleaved hybrid configurations substantially outperform the transformer at large scales, so the exact ratio is less critical than the decision to hybridize at all.

\paragraph{GDN Architecture: Gate and Eigenvalue Sign.}
\label{par:gdn-arch-ablations}
We additionally ablate two internal GDN design choices---the use of the output gate and the sign of the recurrence eigenvalues---across both pure and hybrid (3:1 interleaved) configurations.
In the GDN, the \emph{output gate} multiplies the RNN output by a learned sigmoid-gated projection, while the \emph{eigenvalue sign} determines whether the recurrence allows oscillatory dynamics (negative EVs) or monotone decay only (positive EVs).
Results are reported in the bottom two sections of \Cref{tab:ablation-evals-avg,tab:ablation-evals-detail-a,tab:ablation-evals-detail-b}.

For \textbf{pure GDN}, removing the gate consistently hurts performance across all scales and domains: both no-gate variants (Neg EV, no gate and Pos EV, no gate) underperform their gated counterparts, confirming the gate is a useful component in a pure-RNN setting.
Interestingly, positive EVs with a gate (Pos EV, gate) perform comparably to or slightly better than negative EVs with a gate at several scales, though the differences are small and inconsistent.

For \textbf{hybrid GDN (3:1)}, the picture changes: the selected architecture (Neg EV, gate) remains the most consistent performer, but the no-gate variants are more competitive in the hybrid setting than in the pure setting, and occasionally outperform the gated variants at individual scales.
This suggests that, in the hybrid setting, attention layers may partially compensate for the removed gate, reducing its marginal value.
Overall, the differences among the four hybrid variants are small (within $\sim$0.01 BPB at most scales), and we retain the gated negative-EV configuration as the default based on its slight edge in consistency.

\begin{table*}[t]
    \centering
    \caption{
    \textbf{Architecture ablation results --- averages.}
    Average \olmothreeeval BPB (across Math, Code, QA) for pure and hybrid architectures at 7 representative scales.
    \textbf{Bold} indicates best; \underline{underline} second best; $^\dagger$ best within section.
    $\bigstar$ marks our selected architecture.
    Note that per-size comparisons should be interpreted with care, as architectures at the same nominal scale differ in actual parameter count (see \Cref{tab:ablation-configs}).
    }
    \label{tab:ablation-evals-avg}
    \small
    \begin{tabular}{@{}l c c c c c c c c@{}}
        \toprule
        \textbf{Architecture} & \textbf{Attn \%} & \textbf{60M} & \textbf{100M} & \textbf{190M} & \textbf{370M} & \textbf{600M} & \textbf{760M} & \textbf{1B} \\
        \midrule
        \multicolumn{9}{l}{\textit{Pure Architectures}} \\
        \quad Transformer & 100\% & 1.155 & 1.037 & 0.950 & 0.839 & 0.781 & 0.747 & 0.682 \\
        \quad GDN & 0\% & 1.080$^\dagger$ & 0.990$^\dagger$ & 0.895$^\dagger$ & \textbf{0.795} & 0.761$^\dagger$ & 0.722$^\dagger$ & 0.677$^\dagger$ \\
        \quad Mamba2 & 0\% & 1.146 & 1.036 & 0.941 & 0.843 & 0.787 & 0.750 & 0.718 \\
        \midrule
        \multicolumn{9}{l}{\textit{Hybrid: Interleaved Attention}} \\
        \quad GDN (1:1) & 50\% & 1.093 & 0.992 & 0.896 & 0.804 & 0.787 & 0.724 & 0.674 \\
        \quad GDN (3:1)$^\bigstar$ & 25\% & \underline{1.077} & 0.986$^\dagger$ & \textbf{0.891} & 0.799$^\dagger$ & 0.754 & \underline{0.717} & 0.669$^\dagger$ \\
        \quad GDN (7:1) & 12.5\% & 1.082 & 0.988 & 0.892 & 0.803 & \textbf{0.748} & 0.721 & 0.675 \\
        \quad Mamba2 (3:1) & 25\% & 1.130 & 1.012 & 0.921 & 0.828 & 0.772 & 0.732 & 0.698 \\
        \midrule
        \multicolumn{9}{l}{\textit{Hybrid GDN (3:1): Middle Placement}} \\
        \quad Interleaved$^\bigstar$ & 25\% & 1.077$^\dagger$ & 0.986$^\dagger$ & \underline{0.891} & 0.799$^\dagger$ & 0.754 & 0.717$^\dagger$ & 0.669$^\dagger$ \\
        \quad Middle & 25\% & 1.087 & 0.988 & 0.899 & 0.800 & 0.754$^\dagger$ & 0.721 & 0.672 \\
        \midrule
        \multicolumn{9}{l}{\textit{Hybrid GDN (3:1): Gate/Eigenvalue Ablations}} \\
        \quad Neg EV, gate$^\bigstar$ & 25\% & 1.077$^\dagger$ & 0.986 & 0.891$^\dagger$ & 0.799 & 0.754 & 0.717 & 0.669 \\
        \quad Neg EV, no gate & 25\% & 1.081 & \textbf{0.976} & 0.894 & 0.801 & 0.752$^\dagger$ & 0.721 & \textbf{0.666} \\
        \quad Pos EV, gate & 25\% & 1.086 & \underline{0.981} & 0.897 & 0.828 & 0.757 & \textbf{0.714} & \underline{0.667} \\
        \quad Pos EV, no gate & 25\% & 1.082 & 0.988 & 0.901 & 0.795$^\dagger$ & 0.754 & 0.717 & 0.670 \\
        \midrule
        \multicolumn{9}{l}{\textit{Pure GDN: Gate/Eigenvalue Ablations}} \\
        \quad Neg EV, gate & 0\% & 1.080 & 0.990 & 0.895 & \underline{0.795} & 0.761 & 0.722 & 0.677 \\
        \quad Pos EV, gate & 0\% & \textbf{1.070} & 0.982$^\dagger$ & 0.893$^\dagger$ & 0.798 & \underline{0.751} & 0.720$^\dagger$ & 0.674$^\dagger$ \\
        \quad Neg EV, no gate & 0\% & 1.087 & 0.999 & 0.904 & 0.809 & 0.761 & 0.731 & 0.680 \\
        \quad Pos EV, no gate & 0\% & 1.092 & 1.002 & 0.901 & 0.806 & 0.765 & 0.727 & 0.678 \\
        \bottomrule
    \end{tabular}
\end{table*}

\paragraph{Summary of Findings.}
The ablation study allows us to answer the three design questions posed at the beginning of this section.
First, \textbf{GDN is preferred over Mamba2} as the linear RNN component: the Mamba2-based models performed worse at every scale in both pure and hybrid configurations, while GDN matched or outperformed the transformer even as a pure architecture.
Second, \textbf{interleaved placement outperforms middle placement}: distributing attention layers uniformly throughout the network consistently yields better results than concentrating them, likely because it allows every part of the network to access global context and maximizes the number of layer-type alternations.
Third, \textbf{the 3:1 linear-to-attention ratio is a good default overall choice}: while the 7:1 ratio is competitive at small scales, the 3:1 ratio provides the most consistent gains at large scales, and the 1:1 ratio does not justify the additional attention cost.

\subsection{\model vs.~Other Open Models} \label{sec:full-comparison}

\textbf{Setup.} We organize our baseline open-weight models into four groups, distinguishing dense vs.\ MoE MLP layers and RNN-only vs.\ Hybrid vs.\ Attention-only architectures. We emphasize that baselines were trained with different datasets and token budgets, making direct comparisons between them less meaningful for evaluating architecture choices.

The closest group to \model are other hybrid models with dense MLP layers: Nemotron-H \citep{nvidia2025nemotronh}, Falcon H1 \citep{zuo2025falcon}, and RecurrentGemma \citep{botev2024recurrentgemma}.
Falcon H1 is a parallel hybrid mixer (also referred to as ``intra-layer hybridization''): it uses Mamba-2 SSM and attention blocks on 100\% of layers and performs a channel-wise concatenation prior to the MLP layer.
Nemotron-H is the closest architecture to ours, with 8\% of layers using self-attention (GQA; \citealp{ainslie2023gqa}) and all other layers using Mamba-2 SSM blocks.
RecurrentGemma applies local self-attention (with a 2K context window) and an SSM using the Griffin architecture \citep{de2024griffin} on 100\% of layers and takes a gated sum.

We also compare to open-weight pure RNN models: Falcon Mamba~\citep{zuo2024falconmamba} and xLSTM~\citep{beck2025xlstm7b}.
Falcon Mamba uses the Mamba 1 architecture on 100\% of layers. xLSTM uses mLSTM~\citep{beck2026xlstm} (a gated linear RNN with a memory state) on 100\% of layers, and is the only model running in FP32.

Finally, we report performance for hybrid models using MoE for MLP layers: Nemotron 3 Nano \citep{nvidia_nemotron_nano_v3_2025}, which alternates between Mamba-2 SSM on 85\% of layers and self-attention (GQA) on 15\% of layers; and Kimi Linear \citep{kimi2025linear}, which uses Kimi DeltaNet on 75\% of layers and self-attention (MLA; \citealp{deepseekv2}) on 25\% of layers.

\textbf{Evaluation Details.} Not all models apply long-context extension to the same window size, so for our base model results on RULER we evaluate up to the maximum supported context window at the end of pretraining. We use the most recent vLLM~\citep{vllm} and transformers versions across all models, except for the pure RNN baselines\footnote{We observed slight differences in scores from those reported in the Olmo 3 paper ($<$0.25\% absolute difference for all tasks for Olmo 3) when using the most recent \texttt{transformers==5.0.0}. As \texttt{transformers==5.0.0} breaks both pure RNN baselines (RecurrentGemma and xLSTM), we use older versions for those baselines.}.

\begin{table}[t!]
\centering
\footnotesize
\setlength\tabcolsep{5pt}
\renewcommand{\arraystretch}{0.9}
\adjustbox{max width=\linewidth}{
{\fontsize{9}{9}\selectfont
\begin{NiceTabular}{@{}Hl
P{38pt}C{38pt}C{38pt}C{38pt}|C{38pt}C{38pt}|C{38pt}C{38pt}|C{38pt}C{38pt}H@{}}
\toprule
& & \multicolumn{4}{c}{\textbf{\texttt{Hybrid Dense}}} & \multicolumn{2}{c}{\textbf{\texttt{Pure RNN Dense}}} & \multicolumn{2}{c}{\textbf{\texttt{Hybrid MoE}}} & \multicolumn{2}{c}{\textbf{\texttt{Dense Baselines}}} \\
&

& \textbf{\model{}} & \textbf{Nemo.-H} & \textbf{Falcon H1} & \textbf{Recurr. Gemma} & \textbf{Falcon Mamba} & \textbf{xLSTM} & \textbf{Nemo. 3 Nano} & \textbf{Kimi Linear} & \textbf{Olmo 3} & \textbf{Qwen 3} \\

\midrule
-- & \# Parameters & 7B & 8B & 7B & 9B & 7B & 7B & 30B A3B & 48B A3B & 7B & 8B \\
-- & \# Train Tokens & 6T & 15T & 12T & 2T & 6T & 2T & 25T & 6T & 6T & 36T \\
-- & Train Compute ($10^{23}$ FLOPs) & 2.6 & 7.2 & 5.5 & 1.0 & 2.5 & 0.9 & 5.4 & 1.0 & 2.6 & 17.3 \\


\midrule
\rowcolor{midgrey}   -- & {\textbf{\olmothreeeval \fontsize{9}{9}\selectfont~Math}}                                              & 55.1                        & 54.6                & 65.7                  & 32.1                      & 33.7                     & 18.3              & 53.2                          & 68.5                         & 54.6               & 67.2              \\
\rowcolor{lightgrey} -- & GSM8k                                                                                                  & 74.3                        & 76.8                & 80.9                  & 49.6                      & 57.1                     & 32.8              & 86.3                          & 84.9                         & 75.2               & 84.2              \\
\rowcolor{lightgrey} -- & GSM Symbolic                                                                                           & 49.2                        & 55.4                & 64.8                  & 25.3                      & 25.8                     & 10.8              & 68.6                          & 66.7                         & 48.4               & 65.4              \\
\rowcolor{lightgrey} -- & MATH                                                                                                   & 41.8                        & 31.7                & 51.4                  & 21.3                      & 18.2                     & 11.3              & 4.6                           & 54.0                         & 40.1               & 52.0              \\
\midrule
\rowcolor{midgrey}   -- & {\textbf{\olmothreeeval \fontsize{9}{9}\selectfont~Code}}                                              & 32.4                        & 37.1                & 45.3                  & 23.7                      & 14.6                     & 3.1               & 47.3                          & 30.3                         & 30.9               & 46.2              \\
\rowcolor{lightgrey} -- & BigCodeBench                                                                                           & 35.1                        & 40.3                & 40.5                  & 20.2                      & 0.2                      & 4.0               & 45.7                          & 44.5                         & 34.8               & 43.1              \\
\rowcolor{lightgrey} -- & HumanEval                                                                                              & 49.0                        & 60.1                & 61.0                  & 34.5                      & 0.1                      & 13.3              & 77.1                          & 72.4                         & 49.0               & 71.2              \\
\rowcolor{lightgrey} -- & DeepSeek LeetCode                                                                                      & 2.2                         & 4.1                 & 3.5                   & 0.6                       & 0.1                      & 0.0               & 10.1                          & 1.1                          & 1.5                & 8.8               \\
\rowcolor{lightgrey} -- & DS 1000                                                                                                & 21.1                        & 24.7                & 29.0                  & 19.1                      & 16.3                     & 2.3               & 33.3                          & 34.3                         & 20.6               & 33.3              \\
\rowcolor{lightgrey} -- & MBPP                                                                                                   & 50.3                        & 58.0                & 63.5                  & 33.7                      & 37.5                     & 1.3               & 66.4                          & 58.3                         & 43.6               & 65.8              \\
\rowcolor{lightgrey} -- & MultiPL HumanEval                                                                                      & 29.4                        & 32.3                & 59.7                  & 21.6                      & 15.7                     & 0.2               & 48.3                          & 1.1                          & 28.8               & 52.5              \\
\rowcolor{lightgrey} -- & MultiPL MBPPP                                                                                          & 39.5                        & 40.2                & 59.8                  & 36.4                      & 32.1                     & 0.3               & 50.0                          & 0.6                          & 38.3               & 48.5              \\
\midrule
\rowcolor{midgrey}   -- & {$\textbf{\olmothreeeval \fontsize{9}{9}\selectfont~MC}_\textbf{\fontsize{6}{6}\selectfont~STEM}$}     & 70.0                        & 72.4                & 75.7                  & 61.6                      & 64.2                     & 36.9              & 78.6                          & 77.5                         & 66.2               & 78.7              \\
\rowcolor{lightgrey} -- & ARC MC                                                                                                 & 90.8                        & 93.3                & 94.2                  & 82.6                      & 85.7                     & 46.7              & 94.8                          & 94.3                         & 89.2               & 95.4              \\
\rowcolor{lightgrey} -- & MMLU STEM                                                                                              & 64.6                        & 64.2                & 73.8                  & 49.6                      & 52.0                     & 34.3              & 73.8                          & 67.8                         & 59.7               & 76.7              \\
\rowcolor{lightgrey} -- & MedMCQA MC                                                                                             & 52.1                        & 56.2                & 59.9                  & 47.3                      & 48.9                     & 31.4              & 64.2                          & 64.6                         & 48.0               & 63.5              \\
\rowcolor{lightgrey} -- & MedQA MC                                                                                               & 48.7                        & 54.5                & 55.9                  & 39.4                      & 42.5                     & 23.2              & 64.9                          & 66.6                         & 41.6               & 62.0              \\
\rowcolor{lightgrey} -- & SciQ MC                                                                                                & 93.9                        & 94.1                & 94.7                  & 88.9                      & 91.9                     & 48.8              & 95.0                          & 94.0                         & 92.8               & 96.0              \\
\midrule
\rowcolor{midgrey}   -- & {$\textbf{\olmothreeeval \fontsize{9}{9}\selectfont~MC}_\textbf{\fontsize{6}{6}\selectfont~Non-STEM}$} & 80.4                        & 80.7                & 84.1                  & 71.1                      & 74.2                     & 39.9              & 83.9                          & 76.2                         & 78.2               & 84.9              \\
\rowcolor{lightgrey} -- & MMLU Humanities                                                                                        & 71.6                        & 76.8                & 78.7                  & 62.1                      & 65.8                     & 37.6              & 80.5                          & 79.0                         & 69.2               & 78.6              \\
\rowcolor{lightgrey} -- & MMLU Social Sci.                                                                                       & 79.7                        & 80.5                & 84.4                  & 68.4                      & 70.8                     & 39.5              & 85.1                          & 73.4                         & 75.2               & 84.9              \\
\rowcolor{lightgrey} -- & MMLU Other                                                                                             & 71.0                        & 73.2                & 76.3                  & 63.9                      & 64.1                     & 38.7              & 78.0                          & 78.2                         & 66.8               & 76.7              \\
\rowcolor{lightgrey} -- & CSQA MC                                                                                                & 78.4                        & 75.9                & 78.1                  & 67.8                      & 73.0                     & 32.5              & 75.3                          & 52.8                         & 75.2               & 84.1              \\
\rowcolor{lightgrey} -- & PiQA MC                                                                                                & 82.7                        & 85.9                & 87.8                  & 76.4                      & 83.5                     & 51.5              & 90.0                          & 82.8                         & 80.2               & 89.9              \\
\rowcolor{lightgrey} -- & SocialIQA MC                                                                                           & 81.1                        & 78.6                & 81.5                  & 75.5                      & 77.5                     & 34.1              & 80.7                          & 78.8                         & 80.3               & 83.3              \\
\rowcolor{lightgrey} -- & CoQA Gen2MC MC                                                                                         & 93.9                        & 91.6                & 94.4                  & 83.3                      & 86.5                     & 33.6              & 92.2                          & 91.2                         & 92.5               & 93.6              \\
\rowcolor{lightgrey} -- & DROP Gen2MC MC                                                                                         & 68.1                        & 60.9                & 81.3                  & 48.7                      & 51.4                     & 33.9              & 70.0                          & 71.9                         & 67.4               & 78.4              \\
\rowcolor{lightgrey} -- & Jeopardy Gen2MC MC                                                                                     & 89.3                        & 92.5                & 91.8                  & 83.9                      & 90.2                     & 46.7              & 94.9                          & 54.1                         & 86.9               & 92.1              \\
\rowcolor{lightgrey} -- & NaturalQs Gen2MC MC                                                                                    & 71.0                        & 76.0                & 73.3                  & 64.2                      & 67.7                     & 34.3              & 79.1                          & 78.1                         & 69.5               & 74.6              \\
\rowcolor{lightgrey} -- & SQuAD Gen2MC MC                                                                                        & 97.0                        & 95.8                & 97.5                  & 87.6                      & 85.4                     & 56.2              & 97.3                          & 97.3                         & 97.0               & 97.4              \\
\midrule
\rowcolor{midgrey}   -- & {\textbf{\olmothreeeval \fontsize{9}{9}\selectfont~GenQA}}                                             & 72.9                        & 71.2                & 71.7                  & 68.5                      & 68.5                     & 34.8              & 78.1                          & 75.7                         & 72.5               & 71.1              \\
\rowcolor{lightgrey} -- & HellaSwag RC                                                                                           & 79.0                        & 84.3                & 80.9                  & 81.3                      & 81.9                     & 61.1              & 86.0                          & 85.0                         & 77.7               & 80.6              \\
\rowcolor{lightgrey} -- & Winogrande RC                                                                                          & 86.2                        & 89.6                & 87.9                  & 86.3                      & 88.3                     & 62.1              & 88.7                          & 88.2                         & 85.6               & 86.4              \\
\rowcolor{lightgrey} -- & Lambada                                                                                                & 70.2                        & 76.0                & 73.0                  & 72.2                      & 80.9                     & 19.7              & 75.9                          & 77.2                         & 68.8               & 72.8              \\
\rowcolor{lightgrey} -- & Basic Skills                                                                                           & 89.7                        & 90.2                & 93.3                  & 83.0                      & 86.1                     & 60.9              & 92.5                          & 93.1                         & 89.5               & 93.4              \\
\rowcolor{lightgrey} -- & DROP                                                                                                   & 72.9                        & 65.7                & 69.4                  & 44.0                      & 41.6                     & 24.4              & 75.0                          & 72.2                         & 71.5               & 57.2              \\
\rowcolor{lightgrey} -- & Jeopardy                                                                                               & 64.0                        & 69.7                & 61.1                  & 66.4                      & 67.1                     & 9.3               & 77.2                          & 77.1                         & 60.3               & 65.1              \\
\rowcolor{lightgrey} -- & NaturalQs                                                                                              & 34.8                        & 37.7                & 27.4                  & 30.6                      & 35.1                     & 1.8               & 45.7                          & 25.9                         & 32.5               & 33.8              \\
\rowcolor{lightgrey} -- & SQuAD                                                                                                  & 92.1                        & 87.9                & 91.4                  & 85.9                      & 75.9                     & 33.9              & 94.0                          & 91.8                         & 93.6               & 89.0              \\
\rowcolor{lightgrey} -- & CoQA                                                                                                   & 67.4                        & 39.9                & 60.7                  & 66.6                      & 59.7                     & 40.4              & 67.7                          & 70.3                         & 72.8               & 61.8              \\
\midrule
\rowcolor{midgrey}   -- & {\textbf{\olmothreeeval \fontsize{9}{9}\selectfont~HeldOut}}                                           &                             &                     &                       &                           &                          &                   &                               &                              &                    &                   \\
\rowcolor{lightgrey} -- & LBPP                                                                                                   & 16.8                        & 26.8                & 30.2                  & 5.8                       & 5.7                      & 0.6               & 33.7                          & 31.1                         & 17.7               & 26.2              \\
\rowcolor{lightgrey} -- & BBH                                                                                                    & 65.2                        & 69.9                & 75.5                  & 53.3                      & 42.8                     & 24.6              & 78.2                          & 68.6                         & 64.0               & 76.5              \\
\rowcolor{lightgrey} -- & MMLU Pro MC                                                                                            & 41.7                        & 44.4                & 50.0                  & 27.6                      & 24.5                     & 11.4              & 53.5                          & 50.7                         & 37.2               & 50.1              \\
\rowcolor{lightgrey} -- & Deepmind Math                                                                                          & 23.4                        & 26.1                & 34.3                  & 17.3                      & 15.5                     & 7.5               & 32.8                          & 40.8                         & 23.6               & 47.6              \\
\midrule
\rowcolor{midgrey}   -- & {\textbf{\fontsize{9}{9}\selectfont~Long-context Eval}}                                                            &                             &                     &                       &                           &                          &                   &                               &                              &                    &                   \\
\rowcolor{lightgrey} -- & RULER 4K                                                                                               & 92.7                        & 84.1                & 92.0                  & 44.9                      & 68.2                     & 28.6              & 96.2                          & 88.5                         & 94.9               & 95.5              \\
\rowcolor{lightgrey} -- & RULER 8K                                                                                               & 91.4                        & 82.0                & 92.0                  & --                        & 49.0                     & 20.0              & 95.2                          & 89.4                         & 91.2               & 94.1              \\
\rowcolor{lightgrey} -- & RULER 16K                                                                                              & 90.0                        & --                  & 89.7                  & --                        & 23.8                     & 12.4              & 93.8                          & 94.0                         & 84.1               & 93.5              \\
\rowcolor{lightgrey} -- & RULER 32K                                                                                              & 86.4                        & --                  & 83.0                  & --                        & 4.7                      & --                & 90.3                          & 90.2                         & 78.3               & 90.1              \\
\rowcolor{lightgrey} -- & RULER 64K                                                                                              & 85.0                        & --                  & 77.2                  & --                        & 0.0                      & --                & 85.9                          & 90.5                         & 67.8               & --                \\


\bottomrule
\end{NiceTabular}}
}
\caption{
Base model performance on \olmothreeeval compared to open-weight baselines using the \olmothreeeval suite. 
\model{} outperforms the dense Olmo 3 on all \olmothreeeval multi-task averages, and both Pure RNN baselines (Falcon Mamba and xLSTM). 
\model{} is also competitive with other open-weight hybrid dense models, like Nemotron-H and Falcon H1, despite being trained on half or fewer tokens.
\model{} was not evaluated on held-out benchmarks prior to release.
RULER was evaluated up to the max context window for each model.
}
\label{tab:base-eval}
\end{table}


\textbf{Findings.} As shown in \Cref{tab:base-eval}, we report base model performance on \olmothreeeval, including reported parameter and token counts. For each baseline, we estimate training compute using the $\text{FLOPs}=6ND$ heuristic from \citet{kaplan2020scalinglaws} and $\text{FLOPs}=6N_{\text{active}}D$ for MoE models following \citet{fedus2022switch, clark2022unified}.
While all dense baselines have a similar number of parameters (7--9B), they were trained across a wide range of token budgets, presenting a clear compute-performance tradeoff.

\model substantially outperforms both pure RNN baselines, including Falcon Mamba (which is closely matched on tokens and parameters), across all task averages in \olmothreeeval. \model also outperforms the older open-weight model xLSTM, trained on roughly 2T tokens.

Across \olmothreeeval task averages, \model 7B shows competitive performance against both Nemotron-H 8B (trained for 15T tokens, 2.5${\times}$ more than \model) and Falcon H1 (trained for 12T tokens, $2{\times}$ as many as \model).

Among hybrid dense models, there is a large discrepancy in token budgets. \model outperforms RecurrentGemma, which was trained on 3${\times}$ fewer tokens. It also appears generally strong for the amount of data it was trained on, matching or outperforming Nemotron-H and Falcon H1 (which have 2${\times}$ larger token budgets) on some tasks, while performing worse on others (e.g., \model only outperforms Falcon H1 on GenQA tasks).

\subsection{Post-Training \model} \label{sec:posttraining}

\begin{table}[t!]
\setlength\tabcolsep{2pt}
\renewcommand{\arraystretch}{0.9}
{\fontsize{8}{8}\selectfont
\begin{center}
\begin{footnotesize}
\begin{tabular}{l|ccccc|cccc|ccc|ccc}
\toprule
    & \multicolumn{5}{c|}{\textbf{\texttt{Knowledge \& Reasoning}}}
    & \multicolumn{4}{c|}{\textbf{\texttt{Math}}}
    & \multicolumn{3}{c|}{\textbf{\texttt{Code}}}
    & \multicolumn{3}{c}{\textbf{\texttt{Chat \& IF}}} \\
    {\textbf{\fontsize{8}{8}\selectfont~Model}} &
    {\textbf{\fontsize{7}{7}\selectfont~MMLU}} &
    {\textbf{\fontsize{7}{7}\selectfont~PopQA}} &
    {\textbf{\fontsize{7}{7}\selectfont~BBH}} &
    {\textbf{\fontsize{7}{7}\selectfont~GPQA}} &
    {\textbf{\fontsize{7}{7}\selectfont~Zebra}} &
    {\textbf{\fontsize{7}{7}\selectfont~MATH}} &
    {\textbf{\fontsize{7}{7}\selectfont~$\Omega$}} &
    {\textbf{\fontsize{6}{6}\selectfont~AIME'25}} &
    {\textbf{\fontsize{6}{6}\selectfont~AIME'24}} &
    {\textbf{\fontsize{7}{7}\selectfont~HE+}} &
    {\textbf{\fontsize{7}{7}\selectfont~MBPP+}} &
    {\textbf{\fontsize{7}{7}\selectfont~LCB}} &
    {\textbf{\fontsize{7}{7}\selectfont~IFEval}} &
    {\textbf{\fontsize{7}{7}\selectfont~IFB}} &
    {\textbf{\fontsize{6}{7}\selectfont~AE3}}
    \\
\midrule

\rowcolor{lightgrey}    Olmo 3 Think SFT & 74.9 & 20.8 & 84.1 & 45.8 & 57.9 & 94.4 & 37.8 & 57.6 & 69.6 & 88.2 & 63.2 & 67.8 & 77.9 & 30.0 & 43.9 \\
\rowcolor{lightgrey}    \model{} Think SFT & 80.5 & 25.1 & 84.6 & 47.0 & 55.1 & 93.8 & 35.1 & 55.2 & 66.2 & 86.3 & 63.7 & 65.5 & 80.4 & 31.6 & 49.0 \\
\midrule

\rowcolor{midgrey}      Olmo 3 Instruct SFT & 67.1 & 16.5 & 51.0 & 30.0 & 18.0 & 65.1 & 14.4 & 7.2 & 6.7 & 69.8 & 56.5 & 20.0 & 81.7 & 27.4 & 21.8 \\
\rowcolor{midgrey}      \model{} Instruct SFT & 71.9 & 16.8 & 47.3 & 36.8 & 17.0 & 66.7 & 16.0   & 8.8 & 6.7 & 69.2 & 55.3 & 21.3 & 81.5 & 29.0 & 25.6 \\
\midrule

\rowcolor{ai2lightpink} Olmo 3 DPO & 69.1 & 20.7 & 69.3 & 37.9 & 28.4 & 79.6 & 22.8 & 20.4 & 23.5 & 72.9 & 55.9 & 18.8 & 82.0 & 29.3 & 43.3 \\
\rowcolor{ai2lightpink} \model{} DPO & 73.6 & 21.0 & 57.3 & 38.0 & 29.1 & 72.9 & 19.5 & 10.2 & 10.1 & 75.1 & 56.9 & 22.0 & 80.5 & 33.3 & 56.3 \\

\midrule
\multicolumn{16}{c}{\textit{Hybrid $-$ Dense ($\Delta$)}} \\
\midrule

$\Delta$ Think SFT & \cellcolor{green!15}+5.6 & \cellcolor{green!11}+4.3 & +0.5 & \cellcolor{green!3}+1.2 & \cellcolor{red!7}$-$2.8 & \cellcolor{red!2}$-$0.6 & \cellcolor{red!7}$-$2.7 & \cellcolor{red!6}$-$2.4 & \cellcolor{red!9}$-$3.4 & \cellcolor{red!5}$-$1.9 & +0.5 & \cellcolor{red!6}$-$2.3 & \cellcolor{green!7}+2.5 & \cellcolor{green!4}+1.6 & \cellcolor{green!13}+5.1 \\
$\Delta$ Instruct SFT & \cellcolor{green!13}+4.8 & +0.3 & \cellcolor{red!10}$-$3.7 & \cellcolor{green!18}+6.8 & \cellcolor{red!3}$-$1.0 & \cellcolor{green!4}+1.6 & \cellcolor{green!4}+1.6 & \cellcolor{green!4}+1.6 & $-$0.0 & \cellcolor{red!2}$-$0.7 & \cellcolor{red!3}$-$1.2 & \cellcolor{green!3}+1.3 & $-$0.3 & \cellcolor{green!4}+1.6 & \cellcolor{green!10}+3.8 \\
$\Delta$ DPO & \cellcolor{green!12}+4.5 & +0.3 & \cellcolor{red!31}$-$12.0 & +0.1 & \cellcolor{green!2}+0.7 & \cellcolor{red!18}$-$6.7 & \cellcolor{red!9}$-$3.3 & \cellcolor{red!27}$-$10.2 & \cellcolor{red!35}$-$13.4 & \cellcolor{green!6}+2.2 & \cellcolor{green!3}+1.0 & \cellcolor{green!8}+3.2 & \cellcolor{red!4}$-$1.5 & \cellcolor{green!10}+4.0 & \cellcolor{green!34}+13.0 \\

\bottomrule
\end{tabular}
\end{footnotesize}
\vspace{2mm}
\caption{
Post-training evaluation comparison between Olmo 3 (dense) and \model{} (hybrid) across three training stages: Think SFT, Instruct SFT, and DPO.
$\Omega$ = Omega Full. HE+ = HumanEvalPlus. IFB = IFBench. AE3 = AlpacaEval 3.
}
\label{tab:post_training_overview}
\end{center}
}
\end{table}

As a preliminary investigation of hybrid models' capabilities after post-training, we apply a portion of the Olmo 3 post-training pipeline to create a preliminary Instruct version of \model.
Following \citet{olmo2025olmo3}, this involves three stages: thinking SFT on long reasoning traces, then an instruction-tuning phase \textit{without} the model first thinking, and finally direct preference optimization (DPO; \citealp{rafailov2024direct}).
In this section, we compare \model and Olmo 3 7B at each stage and discuss details and challenges we encountered when adapting our post-training recipe to the hybrid architecture.

\paragraph{Data Changes.}
Relative to Olmo 3, there is one minor change to the Think SFT stage: the addition of function-calling data for tool use.
The lack of advanced tool use in our Think models is a known limitation for building agentic models; this new data is one part of broader changes planned for future Olmo models.
We use the same prompts from the Olmo 3 Instruct SFT models with new thinking traces, reusing thinking traces from DR Tülu \citep{shao2025drtulu} for the Web Search QA prompts and generating reasoning traces with GPT-4.1 for the remaining data, upsampling each data point 3$\times$ to increase token coverage.\footnote{The new Think SFT dataset is available at: \url{https://hf.co/datasets/allenai/Dolci-Think-SFT-Olmo-Hybrid}}
The Instruct SFT and DPO data are identical to the Olmo 3 7B and 32B Instruct SFT and DPO variants.\footnote{SFT: \url{https://hf.co/datasets/allenai/Dolci-Instruct-SFT}. DPO: \url{https://hf.co/datasets/allenai/Dolci-Instruct-DPO}}

\paragraph{Decoding Settings.}
Finding a stable configuration for \model required careful attention to vLLM flag settings and implementation details, particularly for evaluation.
Correct generation and throughput were both sensitive to careful choice of these settings, although this will likely change as hybrid models are more widely adopted and open-source tooling improves.
The two key flags needed to get maximum performance with the post-trained models were \verb|--disable-cascade-attn|, which disables cascade attention (an optimization for shared prompt prefixes), and \verb|--enforce-eager|, which disables torch compilation.
These two flags have been used in our RL setup dating back to Olmo 3, but are new additions to evaluations, and scores drop precipitously without them.
We hypothesize that \verb|--enforce-eager| matters more for \model than for a dense transformer because torch compilation can introduce subtle numerical differences that compound across recurrent GDN state updates in a way that standard attention layers do not experience.
Consistent with this interpretation, re-enabling compilation while storing the GDN cache in FP32 via \verb|--mamba_ssm_cache_dtype float32| \citep{nvidia2025nvidianemotronnano2} recovers similar scores, suggesting that repeated low-precision downcasting across recurrent steps is part of the issue.

\begin{table}[t]
\centering
\small
\begin{tabular}{llrrrrr}
\toprule
\textbf{Model} & \textbf{Metric} & \textbf{1K} & \textbf{4K} & \textbf{8K} & \textbf{16K} & \textbf{32K} \\
\midrule
\multirow{4}{*}{Olmo 3 7B (MHA)}
  & Tokens/sec/node                  & 2{,}909 & 2{,}164 & 1{,}510 & 1{,}690 & 1{,}247 \\
  & Inference GPUs (RL batch 1K)    &      32 &      48 &      56 &      64 &      96 \\
  & Generation wall time (h)        &    10.5 &    56.7 &   162.5 &   283.7 &   768.6 \\
  & GPU-hours ($\times$1K)          &     0.3 &     2.7 &     9.1 &    18.2 &    73.8 \\
\midrule
\multirow{4}{*}{Olmo Hybrid 7B}
  & Tokens/sec/node                   & 3{,}039 & 3{,}077 & 2{,}736 & 1{,}876 & 1{,}422 \\
  & Inference GPUs (RL batch 1K)    &       8 &      16 &      24 &      40 &      64 \\
  & Generation wall time (h)        &    10.1 &    39.9 &    89.6 &   261.5 &   689.9 \\
  & GPU-hours ($\times$1K)          &     0.1 &     0.6 &     2.2 &    10.5 &    44.2 \\
\midrule
\multirow{4}{*}{\shortstack[l]{Olmo Hybrid 7B \\(enforce eager)}}
  & Tokens/sec/node                  & 1{,}410 & 1{,}565 & 1{,}592 & 1{,}495 & 1{,}066 \\
  & Inference GPUs (RL batch 1K)    &       8 &      16 &      24 &      40 &      64 \\
  & Generation wall time (h)        &    21.7 &    78.4 &   154.1 &   328.1 &   920.4 \\
  & GPU-hours ($\times$1K)          &     0.2 &     1.3 &     3.7 &    13.1 &    58.9 \\
\bottomrule
\end{tabular}
\caption{Inference-bound scaling for async RL generation across context lengths.
In async RL, each training step requires generating a full batch of rollouts before the policy update;
generation throughput is the bottleneck.
We measure the minimum GPU allocation needed to serve an RL batch of 1{,}024 rollouts concurrently and
project generation wall time and total GPU-hours over 1{,}725{,}440 episodes (1{,}684 steps).
}
\label{tab:throughput}
\end{table}

\paragraph{Inference Throughput.}
We conducted early investigations into completing our recipe with reinforcement learning with verifiable rewards (RLVR; \citealp{lambert2025tulu3}).
To quantify basic inference throughput---the bottleneck in our synchronous Olmo RL setup~\citep{olmo2025olmo3}---we benchmarked \model{} on a single 8$\times$A100 node with a 2,048-token prompt, 2 degrees of tensor parallelism, and a batch size of 16 prompts, each prompt sampled 4$\times$. 
The details are shown in Table~\ref{tab:post_training_overview}.
Without \verb|--enforce-eager|, \model is competitive with or faster than the dense
OLMo 3 7B
(4K: 3,023 vs.\ 2,164, 8K: 2,247 vs.\ 1,510, 16K: 1,486 vs.\ 1,690, and 32K: 1,499 vs.\
1,247 tokens/sec/node),
largely because the gated delta net layers reduce KV cache memory pressure
compared to the full multi-head attention used in Olmo 3 7B (which lacks GQA).
However, enabling \verb|--enforce-eager| -- as required for numerical stability -- substantially reduces
throughput (4K: 1,425, 8K: 1,365, 16K: 1,197, and 32K: 654 tokens/sec),
roughly 0.5--0.9$\times$ that of the dense baseline.
The \verb|--enforce-eager| flag was also used in our Olmo 3 training runs, but the dense model did not suffer from this substantial slowdown.\footnote{Subsequent improvements to our vLLM implementation have recovered much of the throughput gap between \model and OLMo 3 7B in eager mode: \url{https://github.com/vllm-project/vllm/pull/32550\#issuecomment-3994432476}.}

\paragraph{Post-Training Performance.}
\Cref{tab:post_training_overview} shows the performance of the Think SFT, Instruct SFT, and DPO checkpoints for \model compared to Olmo 3.
Overall, the stronger pretraining performance translates to persistent gains on knowledge tasks, but the model still lags behind Olmo 3 on extended reasoning tasks such as AIME and Omega~\citep{sun2025omega}.
We anticipate that adapting the post-training data for the hybrid model could improve performance on these benchmarks.
Beyond the data, early post-training results were also sensitive to decoding kernels and related settings, so improvements there could further close the gap.
Other work has shown that the strongest teacher model does not always improve the downstream student proportionally~\citep{guha2025openthoughts, openthoughts-agent}, which could be base-model-specific, further highlighting the need to iterate on the post-training data beyond what was used for Olmo 3.\footnote{The long-context recipe used for \model also differs slightly from Olmo 3 7B, which could contribute to differences in post-training behavior.}
At the level of both engineering and research, post-training hybrid models is in its infancy;
we will continue to work on these recipes and share our findings with the community. 

One limitation of our current investigation into post-training is that we have not considered the implications of switching to a hybrid architecture for safety.
In future work, it would be interesting to assess whether \model exhibits different behavior on safety evaluations compared to transformers.



\section{Conclusion}

Compared to Olmo 3, \model achieves better pretraining efficiency, which translates to improvements on downstream tasks, including long-context abilities.
While there have been many releases of strong hybrid models mixing recurrence with attention, our results provide evidence for the benefit of hybrid architectures over transformers in a controlled, large-scale setting.
Beyond these headline results, we present theoretical analysis and fully controlled scaling studies that further substantiate these advantages.
In particular, we show that hybrid models are more expressive than the sum of their parts: they can represent synthetic tasks that neither transformers nor RNNs in isolation can.
We also explore the theoretical link between expressivity and language model scaling, showing that standard conceptual explanations for scaling laws predict that increasing expressivity should benefit training efficiency, consistent with our empirical findings.
Overall, \model provides evidence for hybrid models over transformers and a conceptual explanation for their observed gains.
This work raises many questions for future research: post-training hybrid models, optimizing RNN architecture details, and more deeply understanding the connection between an architecture's expressive power and its scaling behavior.

\clearpage
\section*{Author Contributions} \label{sec:contrib}
\begin{compactitem}
    \item \textbf{William Merrill} led the project, contributing to early experiments, pretraining, theory, and writing.
    \item \textbf{Yanhong Li} conducted early proof-of-concept experiments, implemented the HuggingFace and vLLM versions of \model{}, ran evaluations, and owned the synthetic experiments and their writeup.
    \item \textbf{Tyler Romero} ran pretraining, mid-training, and long-context extension. He also optimized training throughput, implemented Ulysses-style context parallelism, improved the throughput and numerical correctness of \model{}'s vLLM implementation, and helped with writing.
    \item \textbf{Anej Svete} ran early proof-of-concept experiments alongside WM, ran the experiments for the empirical scaling laws and architecture ablation sections, and wrote up those sections.
    \item \textbf{Caia Costello} led the collaboration from Lambda's side, including project scoping, data management, preliminary research on \model{} training across B200 and H100 hardware, and real-time support for pretraining.
    \item \textbf{Pradeep Dasigi} generated tool use data for \model{} Think SFT and helped with artifact release logistics.
    \item \textbf{Dirk Groeneveld} helped plan the \model{} pretraining run and analyze the training stability of \model{}.
    \item \textbf{David Heineman} ran generative evaluations for \model{} and open-weight models and wrote up these results.
    \item \textbf{Bailey Kuehl} handled technical aspects of preparing \model{} artifacts and documentation for release.
    \item \textbf{Nathan Lambert} led post-training for \model{} and wrote up post-training results.
    \item \textbf{Chuan Li} provided high-level supervision on the Lambda side of the project.
    \item \textbf{Kyle Lo} contributed to pre-training evaluations, writing, and presentation.
    \item \textbf{Saumya Malik} contributed to post-training methodology (tokenization and chat templates) and evaluation.
    \item \textbf{DJ Matusz} helped validate the compatibility of \model{} pretraining with B200 hardware.
    \item \textbf{Benjamin Minixhofer} implemented FLA in OLMo-core to support the initial experiments for \model{} and also benchmarked inference throughput.
    \item \textbf{Jacob Morrison} contributed to post-training, focusing on SFT.
    \item \textbf{Luca Soldaini} helped plan initial experiments and pretraining, and contributed to vLLM implementation and post-training debugging.
    \item \textbf{Finbarr Timbers} contributed to post-training \model{}, implementing the hybrid model in open-instruct.
    \item \textbf{Pete Walsh} worked on training code, built the Slurm launch system for pretraining on Lambda's cluster, and helped run pretraining.
    \item \textbf{Noah Smith} provided high-level guidance on many aspects of the project.
    \item \textbf{Hannaneh Hajishirzi} provided high-level guidance on many aspects of the project.
    \item \textbf{Ashish Sabharwal} provided guidance on the initial experiments and made core technical contributions to the theoretical aspects of the project.
\end{compactitem}

\section*{Acknowledgments} \label{sec:acks}
The authors thank Taira Anderson, Kyle Wiggers, David Albright, and Stephen Kelman for contributing to the release of \model{}.
WM thanks Songlin Yang and Mehryar Mohri for relevant discussions.
AS acknowledges the support of the ETH AI Center doctoral fellowship.
This research used resources of the Oak Ridge Leadership Computing Facility, which is a DOE Office of Science User Facility supported under Contract DE-AC05-00OR22725.
Additionally, this research used Lambda's computational resources for part of \model{} pretraining; the authors thank Amir Zadeh, Allison Beck, Abhi Sarma, and Long Fei for their support during that phase.
Finally, this material is based upon work supported by the National Science Foundation under Award No. 2413244.

\clearpage
\bibliographystyle{abbrvnat}
\bibliography{references}

\clearpage
\appendix
\section{Training \model{}} \label{sec:training-details}

We elaborate on experiments and implementation details in the development of \model.

\subsection{Pretraining} \label{sec:pretraining}
Each \model{} GDN head is sized proportionately to an Olmo 3 attention head. In line with standard mappings from attention to GDN\footnote{\url{https://github.com/fla-org/flash-linear-attention/blob/f24317a6a4f513748cd7eb05818534ce66029957/fla/layers/gated\_deltanet.py\#L40}} this means that the size of the query and key becomes $d = 3/4 \cdot 128 = 96$. Similarly, the size of the value becomes $2d = 192$.

We used the Flash Linear Attention \citep{yang2024fla} implementation of Gated DeltaNet. In particular, we used the default parameter implemented when instantiating single layer directly (as opposed to the model-level initialization).

Beyond the individual GDN heads, the overall architecture is roughly the same Olmo 3 7B \citep{olmo2025olmo3} with a few small tweaks to the architecture, learning rate schedule and training data:
\begin{itemize}
    \item We removed two heads from the model to make Olmo 3 and \model{} more comparable in parameter count and training throughput (see below).
    \item Rather than using the ad-hoc piecewise learning rate schedule from Olmo 3 7B \citep{olmo2025olmo3}, we use a standard cosine decay to 10\% of the maximum learning rate. However, the learning rate schedules still match closely for the majority of training, especially towards the beginning.
    \item For data, we use the improved data mix from Olmo 3 32B rather than the data mix from Olmo 3 7B---in preliminary experiments, we ran with the Olmo 3 7B data mix and saw similar trends in pretraining evaluation metrics toward the beginning of training.
\end{itemize}

The GDN architecture and 3:1 hybridization ratio were chosen based on early experiments at the scale of 1B parameters and 100B tokens.
We replicate those early experiments as carefully controlled architecture ablations in \Cref{sec:architecture-ablations}.

The model was trained on 512 GPUs. At the beginning of training, these were H100s, but roughly halfway through pretraining we migrated to 512 B200s.

\begin{table}[!t]
    \centering
    \caption{Rough training throughput measurements run before launching the \model{} pretraining run.
    }
    \label{tab:training-throughput}
    \begin{tabular}{lcccc}
        \toprule
        Model & \# Heads & $d_\textrm{model}$ & \# Parameters & Training Throughput (TPS) \\
        \midrule
        Olmo 3 & 32 & 4096 & 6.8B & 8.0K \\
        \model{} & 32 & 4096 & 7.7B & 7.7K \\
        & 31 & 3968 & 7.4B & 7.7K \\
        & 30 & 3840 & 7.0B & 8.2K \\
        \bottomrule
    \end{tabular}
\end{table}

\paragraph{Training Throughput.}
We calibrated training throughput to be comparable to that of Olmo 3 in initial experiments by removing individual heads until it closely matched the transformer.
Concretely, we benchmarked model size and throughput for several training configurations running on 128 H100s.
As shown in \Cref{tab:training-throughput}, the hybrid model with 2 heads removed closely matches the Olmo 3 transformer in terms of parameters and slightly outperforms it in terms of training throughput.
We therefore selected a hybrid architecture with 30 heads for the \model{} 7B training run.

\begin{figure}[!t]
    \centering
    \includegraphics[width=1.0\linewidth]{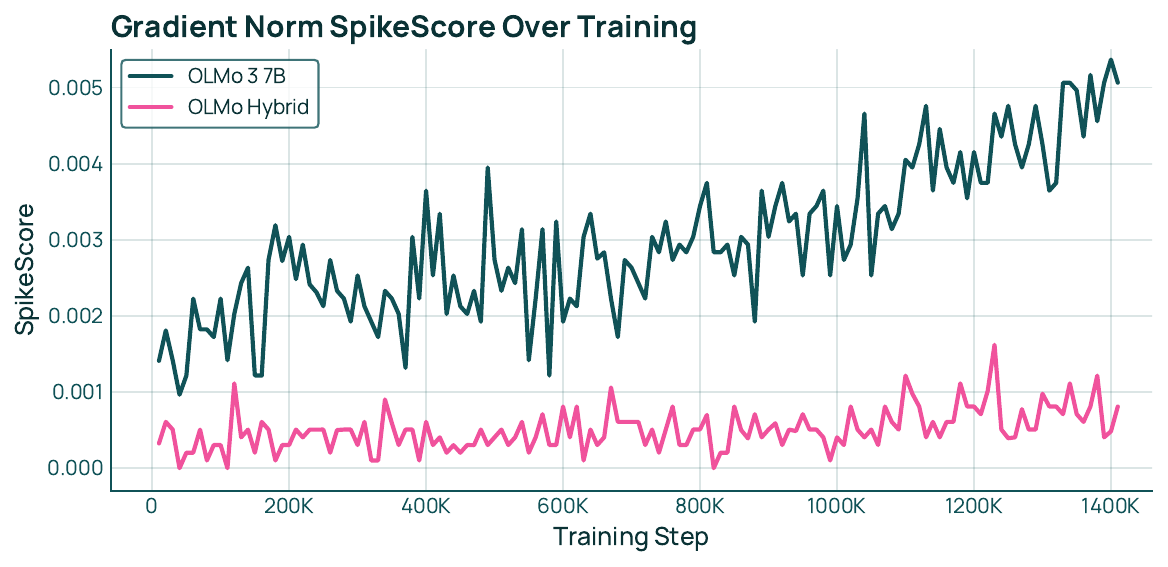}
    \caption{SpikeScore grows throughout training for Olmo 3, but it stays more stable when training \model{}.}
    \label{fig:rnnspikescore}
\end{figure}

\paragraph{Training Stability.}
To estimate the stability of Olmo training runs, we compute a \emph{spike score} as an objective measure. Concretely, we define the spike score as the percentage of values in a time series that are at least six standard deviations away from a rolling average of the last 128 values. We use spike score on the L2 norm of the gradient. This is a slight modification from the method used in~\cite{olmo2}, leading to slightly higher scores.

\Cref{fig:rnnspikescore} shows that Olmo 3 exhibits a high and growing number of spikes in the gradient norm, while \model{} shows a lower, flat trajectory. We take this as preliminary evidence that the \model{} architecture may be more stable (i.e., more tolerant of large learning rates and noisy data) compared to Olmo 3.

\subsection{Mid-Training and Long Context Extension} \label{sec:midtraining}
\paragraph{Mid-Training.}
We adapt our mid-training procedure from Olmo 3 \citep{olmo2025olmo3}, using the Olmo 3 32B mid-training data (the Olmo 3 7B mid-training data mixture with additional light filtering applied).
One notable change is a doubled batch size, motivated by recent insights into the relationship between learning rate and batch size \citep{merrill2025critical}.
Following the Olmo 3 32B recipe, we perform two independent mid-training runs on separate 100B token subsets of Dolma 3 Dolmino Mix and merge the resulting checkpoints.

\paragraph{Long Context Extension.}
After mid-training, we extend context length by continuing training on 100B tokens of Dolma 3 Longmino Mix. We compare two positional encoding strategies for this stage.
The first is YaRN \citep{peng2024yarn}, which was also used for Olmo 3 \citep{olmo2025olmo3}. YaRN modifies the RoPE frequency basis by partitioning dimensions according to frequency: low-frequency components (encoding longer-range position) are interpolated to cover the target context length, while high-frequency components (encoding local position) are left unchanged. An attention temperature factor corrects for the shift in attention logit magnitudes introduced by the interpolation.
The second is DroPE \citep{gelberg2025extendingcontextpretrainedllms}, which removes RoPE entirely during long-context extension. The model then relies on the causal attention mask and any positional signal already captured in its weights. The motivation is that RoPE's rotation frequencies, fit to shorter contexts during pretraining, can limit extrapolation to longer sequences; dropping them removes this constraint.

Both strategies produce strong long-context results for \model{} (\Cref{tab:lc_evals_ruler}), but DroPE shows a clear advantage at the longest evaluation lengths (e.g., 85.0 vs.\ 76.9 on RULER 64k). We attribute this in part to the hybrid architecture: the GDN layers carry implicit positional information through their recurrent structure, so the attention layers are less dependent on explicit positional encodings like RoPE. We therefore adopt DroPE for the released \model{} checkpoint.

To train at these longer sequence lengths, we implemented Ulysses-style context parallelism \citep{jacobs2023deepspeed_ulysses} through both the attention and GDN layers. Ulysses distributes the sequence across devices and uses all-to-all communication to transpose from a sequence-parallel layout to a head-parallel layout before each layer. After the all-to-all, each device holds the full sequence for a subset of heads, which suffices for both attention (where each head attends independently) and GDN (where the recurrent state update in \Cref{def:gdn} is also per-head). The short depthwise convolutions (kernel size~4) applied to the $q$, $k$, and $v$ streams in GDN layers operate per-channel along the sequence dimension; since channels are partitioned across heads, the convolution weights must be sharded consistently with the head assignment on each device.

\subsection{Post-Training}


\begin{table}[ht]
\centering
\caption{Training hyperparameters for \model{} 7B. Total tokens includes masked tokens (e.g.\ prompts). Think SFT total tokens is computed from the Olmo 3 baseline (45.4B) plus the $3\times$ tool-use upscale: $45.4\text{B} \times (1 + 2 \times 2.47\%) = 47.6\text{B}$.}
\label{tab:hyperparams}
\begin{tabular}{l c c c}
\toprule
& \textbf{7B Think SFT} & \textbf{7B Instruct SFT} & \textbf{7B Instruct DPO} \\
\midrule
Instances & 2{,}932{,}239 & 2{,}153{,}716 & 259{,}922 \\
Total Tokens & 47.6B & 3.4B & N/A \\
Batch Size & 1M tokens & 1M tokens & N/A \\
Learning Rate & $2.5 \times 10^{-5}$ & $2.5 \times 10^{-5}$ & $1 \times 10^{-6}$ \\
Num.\ GPUs & 64 & 64 & 32 \\
Max Sequence Length & 32K & 32K & 16K \\
Epochs & 2 & 2 & 1 \\
Loss & --- & --- & DPO Norm ($\beta{=}5$) \\
\bottomrule
\end{tabular}
\end{table}

\begin{table}[ht]
\centering
\caption{Differences between Olmo 3 and \model{} SFT training configurations.}
\label{tab:sft-diffs}
\begin{tabular}{l l l}
\toprule
& \textbf{Olmo 3} & \textbf{\model{}} \\
\midrule
Data Parallel & HSDP (shard within node) & FSDP (full sharding) \\
Context Parallel & Ring (degree${}=8$) & Ulysses (degree${}=2$) \\
Activation Checkpointing & Selected modules (\texttt{feed\_forward}) & Budget-based (0.1) \\
\bottomrule
\end{tabular}
\end{table}

Table~\ref{tab:hyperparams} summarizes the hyperparameters used for the \model{} 7B model variants. Table~\ref{tab:sft-diffs} highlights the key configuration differences between the Olmo 3 and \model{} SFT training setups, primarily involving changes to the parallelism strategy and activation checkpointing. The SFT models were trained with \href{https://github.com/allenai/OLMo-core}{OLMo-core}, and the DPO model was trained with \href{https://github.com/allenai/open-instruct/tree/main}{Open-Instruct}.

\subsection{Evaluation Details} \label{sec:eval-details}

We evaluate \model{} using the same evaluation suite as OLMo 3~\citep{olmo2025olmo3}.
\Cref{tab:eval-config-base} describes the base evaluation configuration
and \Cref{tab:eval-config-posttrain} describes the post-training evaluation configuration.

\begin{table*}[!ht]
  \centering

\begin{scriptsize}
\renewcommand{\arraystretch}{1}
\adjustbox{max width=\linewidth}{
\begin{tabular}{llHHlllllHllHl} 
\toprule
& {\bf Task} & {\bf Capability} & {\bf \# inst} & {\bf ICL} & {\bf Format} & {\bf Metric} & {\bf Temp} & {\bf Top-p} & {\bf Extract} & {\bf Max toks} & {\bf P@k (n)} & {\bf N} & {\bf \# sub} \\
\midrule

\rowcolor{ai2midwhite}\multicolumn{14}{c}{\rule{0pt}{1pt}} \\[-9pt]
\rowcolor{ai2midwhite}\multicolumn{14}{c}{\textbf{Base Main Suite}} \\
\rowcolor{ai2midwhite}\multicolumn{14}{c}{\rule{0pt}{1pt}} \\[-9pt]
\rowcolor{ai2offwhite} & GSM8K* (\citeyear{cobbe2021gsm8k}) & Math Gen & - & 8$^\alpha$ & CoT EM & pass@k & 0.6 & 0.6 & GSM & 512 & 1, 4 (8) & 8 & - \\
\rowcolor{ai2offwhite} & GSM Symbolic* (\citeyear{mirzadeh2024gsmsymbolic}) & Math Gen & - & 8$^\alpha$ & CoT EM & pass@k & 0.6 & 0.6 & GSM & 512 & 1, 4 (8) & 8 & 3 \\
\rowcolor{ai2offwhite} \multirow{-3}{*}{\rotatebox[origin=c]{90}{\textit{Math}}} & MATH 500* (\citeyear{lewkowycz2022solving,lightman2023lets}) & Math Gen & - & 4$^\alpha$ & CoT EM & pass@k & 0.6 & 0.6 & Minerva & 1024 & 1, 16 (32) & 32 & - \\
\rowcolor{lightgrey} & HumanEval* (\citeyear{chen2021codex}) & Code Gen & - & 3 & Code Exec & pass@k & 0.6 & 0.6 & - & 512 & 1, 16 (32) & 32 & - \\
\rowcolor{lightgrey} & MBPP* (\citeyear{austin2021program}) & Code Gen & - & 3 & Code Exec & pass@k & 0.6 & 0.6 & - & 512 & 1, 16 (32) & 32 & - \\
\rowcolor{lightgrey} & BigCodeBench* (\citeyear{zhuo2024bigcodebench}) & Code Gen & - & 3 & Code Exec & pass@k & 0.6 & 0.6 & - & 1280 & 1 (5) & 5 & - \\
\rowcolor{lightgrey} & DS 1000* (\citeyear{Lai2022DS1000}) & Code Gen & - & 3 & Code Exec & pass@k & 0.6 & 0.6 & - & 1024 & 1 (5) & 5 & - \\
\rowcolor{lightgrey} & Deepseek LeetCode* (\citeyear{guo2024deepseekcoder}) & Code Gen & - & 0 & Code Exec & pass@k & 0.6 & 0.6 & - & 512 & 1, 16 (32) & 32 & - \\
\rowcolor{lightgrey} & MultiPL-E HumanEval* (\citeyear{cassano2022multiple}) & Code Gen (6 Lang) & - & 0 & Code Exec & pass@k & 0.6 & 0.6 & - & 1024 & 1, 16 (32) & 32 & 6 \\
\rowcolor{lightgrey} \multirow{-7}{*}{\rotatebox[origin=c]{90}{\textit{Code}}} & MultiPL-E MBPP* (\citeyear{cassano2022multiple}) & Code Gen (6 Lang) & - & 0 & Code Exec & pass@k & 0.6 & 0.6 & - & 1024 & 1, 16 (32) & 32 & 6 \\
\rowcolor{ai2offwhite} & ARC (\citeyear{clark2018arc}) & Science QA & - & 5 & MC & Acc & - & - & - & - & - & - & 2 \\
\rowcolor{ai2offwhite} & MMLU STEM (\citeyear{hendryckstest2021}) & General QA & - & 5 & MC & Acc & - & - & - & - & - & - & 19 \\
\rowcolor{ai2offwhite} & MedMCQA* (\citeyear{pal2022medmcqa}) & Medical QA & - & 5 & MC & Acc & - & - & - & - & - & - & - \\
\rowcolor{ai2offwhite} & MedQA* (\citeyear{jin2021medqa}) & Medical QA & - & 5 & MC & Acc & - & - & - & - & - & - & - \\
\rowcolor{ai2offwhite} \multirow{-5}{*}{\rotatebox[origin=c]{90}{\textit{STEM QA}}} & SciQ* (\citeyear{welbl2017sciq}) & Science QA & - & 5 & MC & Acc & - & - & - & - & - & - & - \\
\rowcolor{lightgrey} & MMLU Humanities (\citeyear{hendryckstest2021}) & General QA & - & 5 & MC & Acc & - & - & - & - & - & - & 13 \\
\rowcolor{lightgrey} & MMLU Social Sci. (\citeyear{hendryckstest2021}) & General QA & - & 5 & MC & Acc & - & - & - & - & - & - & 12 \\
\rowcolor{lightgrey} & MMLU Other (\citeyear{hendryckstest2021}) & General QA & - & 5 & MC & Acc & - & - & - & - & - & - & 14 \\
\rowcolor{lightgrey} & CSQA (\citeyear{talmor2019commonsenseqa}) & Commonsense QA & - & 5 & MC & Acc & - & - & - & - & - & - & - \\
\rowcolor{lightgrey} & PiQA (\citeyear{bisk2020piqa}) & Physical QA & - & 5 & MC & Acc & - & - & - & - & - & - & - \\
\rowcolor{lightgrey} & SocialIQA (\citeyear{sap2019socialiqa}) & Social QA & - & 5 & MC & Acc & - & - & - & - & - & - & - \\
\rowcolor{lightgrey} & DROP Gen2MC* (introduced in \citet{olmo2025olmo3}; \citeyear{dua2019drop}) & Passage QA & - & 5 & MC & Acc & - & - & - & - & - & - & - \\
\rowcolor{lightgrey} & Jeopardy Gen2MC* (introduced in \citet{olmo2025olmo3}; \citeyear{mosaic2024jeopardy}) & Trivia QA & - & 5 & MC & Acc & - & - & - & - & - & - & - \\
\rowcolor{lightgrey} & NaturalQs Gen2MC* (introduced in \citet{olmo2025olmo3}; \citeyear{kwiatkowski2019naturalquestions}) & General QA & - & 5 & MC & Acc & - & - & - & - & - & - & - \\
\rowcolor{lightgrey} & SQuAD Gen2MC* (introduced in \citet{olmo2025olmo3}; \citeyear{rajpurkar2016squad}) & General QA & - & 5 & MC & Acc & - & - & - & - & - & - & - \\
\rowcolor{lightgrey} & CoQA Gen2MC* (introduced in \citet{olmo2025olmo3}; \citeyear{reddy2019coqa}) & Conversation QA & - & 0$^\dagger$ & MC & Acc & - & - & - & - & - & - & - \\
\rowcolor{lightgrey} \multirow{-12}{*}{\rotatebox[origin=c]{90}{\textit{Non-STEM QA}}} & Basic Skills* (introduced in \citet{olmo2025olmo3}) & Basic QA & - & 5 & MC & Acc & - & - & - & - & - & - & 6 \\
\rowcolor{ai2offwhite} & HellaSwag (\citeyear{zellers2019hellaswag}) & Language Modeling & - & 5 & RC$_\text{per-char}$ & Acc & - & - & - & - & - & - & - \\
\rowcolor{ai2offwhite} & WinoGrande (\citeyear{sakaguchi2020winogrande}) & Language Modeling & - & 5 & RC$_\text{none}$ & Acc & - & - & - & - & - & - & - \\
\rowcolor{ai2offwhite} & Lambada (\citeyear{paperno2016lambada}) & Language Modeling & - & 0 & RC$_\text{per-char}$ & Acc & - & - & - & - & - & - & - \\
\rowcolor{ai2offwhite} & Basic Skills* (introduced in \citet{olmo2025olmo3}) & Basic QA & - & 5 & RC$_\text{per-token}$ & Acc & - & - & - & - & - & - & 6 \\
\rowcolor{ai2offwhite} & DROP (\citeyear{dua2019drop}) & Passage QA & - & 5 & GenQA & F1 & 0 & 1 & - & 100 & - & - & - \\
\rowcolor{ai2offwhite} & Jeopardy (\citeyear{mosaic2024jeopardy}) & Trivia QA & - & 5 & GenQA & F1 & 0 & 1 & - & 50 & - & - & - \\
\rowcolor{ai2offwhite} & NaturalQs (\citeyear{kwiatkowski2019naturalquestions}) & General QA & - & 5 & GenQA & F1 & 0 & 1 & - & 50 & - & - & - \\
\rowcolor{ai2offwhite} & SQuAD (\citeyear{rajpurkar2016squad}) & General QA & - & 5 & GenQA & F1 & 0 & 1 & - & 50 & - & - & - \\
\rowcolor{ai2offwhite} \multirow{-9}{*}{\rotatebox[origin=c]{90}{\textit{GenQA}}} & CoQA (\citeyear{reddy2019coqa}) & Conversation QA & - & 0$^\dagger$ & GenQA & F1 & 0 & 1 & - & 50 & - & - & - \\
\rowcolor{ai2midwhite}\multicolumn{14}{c}{\rule{0pt}{1pt}} \\[-9pt]
\rowcolor{ai2midwhite}\multicolumn{14}{c}{\textbf{Base Held-out Suite}} \\
\rowcolor{ai2midwhite}\multicolumn{14}{c}{\rule{0pt}{1pt}} \\[-9pt]
\rowcolor{ai2offwhite} & MMLU Pro (\citeyear{wang2024mmlupro}) & - & - & 5 & MC & Acc & - & - & - & - & - & - & 13 \\
\rowcolor{ai2offwhite} & LBPP* (\citeyear{matton2024lbpp}) & - & - & 0 & Code Exec & pass@k & 0.6 & 0.6 & - & 4096 & 1 (32) & - & - \\
\rowcolor{ai2offwhite} & Deepmind Math* (\citeyear{saxton2019analysing}) & 5500 & - & 5 & CoT EM & pass@k & 0.6 & 0.6 & - & 2048 & 1 (1) & - & - \\
\rowcolor{ai2offwhite} & BigBench Hard (\citeyear{suzgun2022bbh}) & - & - & 3 & CoT EM & Acc & 0.6 & 0.6 & - & 512 & 1 (1) & - & 55 \\

\bottomrule
\end{tabular}
}
\end{scriptsize}
  \caption{
  \textbf{Details of the base evaluation suite}, as used in OLMo 3~\citep{olmo2025olmo3}.
  Tasks were formatted as multiple-choice (MC), rank choice (RC), short-form generative (GenQA), chain-of-thought with exact-match scoring (CoT EM), or code execution (Code Exec).
  We use * to indicate additions to the OLMo 2~\citep{olmo2} base suite, $^\dagger$ for tasks with few-shot examples already specified within each instance, and $^\alpha$ for tasks with human-written few-shot examples.
  }
  \label{tab:eval-config-base}
\end{table*}

\begin{table*}[!ht]
  \centering

\begin{scriptsize}
\begin{tabular}{HlHHllllllHll} 
\toprule
& \textbf{Task} & \textbf{Capability} & \textbf{\# Inst} & \textbf{Format} & \textbf{Metric} & \textbf{Temp} & \textbf{Top-p} & \textbf{Ans. Extract} & \textbf{Max Toks} & \textbf{P@k (N)} & \textbf{N} & \textbf{\# Sub} \\
\midrule

\rowcolor{midgrey}\multicolumn{13}{c}{\rule{0pt}{1pt}} \\[-9pt]
\rowcolor{midgrey}\multicolumn{13}{c}{\textbf{Chat Suite}} \\
\rowcolor{midgrey}\multicolumn{13}{c}{\rule{0pt}{1pt}} \\[-9pt]

\rowcolor{lightgrey} & IF Eval (\citeyear{zhou2023ifeval}) & Instruction Following & - & CoT & Custom & 0.6 & 0.95 & Custom & 32768 & - & 1 & - \\
\rowcolor{lightgrey} & IFBench (\citeyear{pyatkin2025ifbench}) & Instruction Following & - & CoT & Custom & 0.6 & 0.95 & Custom & 32768 & - & 1 & - \\
\rowcolor{lightgrey} & MATH 500 (\citeyear{lewkowycz2022solving,lightman2023lets}) & Math Gen & - & CoT EM & EM Flex & 0.6 & 0.95 & Minerva & 32768 & - & 1 & - \\
\rowcolor{lightgrey} & AIME 2024* & Math Gen & - & CoT EM & EM Flex & 0.6 & 0.95 & Minerva & 32768 & - & 32 & - \\
\rowcolor{lightgrey} & AIME 2025* & Math Gen & - & CoT EM & EM Flex & 0.6 & 0.95 & Minerva & 32768 & - & 32 & - \\
\rowcolor{lightgrey} & Omega Math (\citeyear{sun2025omega}) & Math Gen & - & CoT EM & EM Flex & 0.6 & 0.95 & Custom Regexes & 32768 & - & 1 & 55 \\
\rowcolor{lightgrey} & HumanEval+ (\citeyear{evalplus}) & Code Gen & - & CoT Code & pass@1 & 0.6 & 0.95 & Split on \texttt{```} & 32768 & - & 10 & - \\
\rowcolor{lightgrey} & MBPP+* (\citeyear{evalplus}) & Code Gen & - & CoT Code & pass@1 & 0.6 & 0.95 & Split on \texttt{```} & 32768 & - & 10 & - \\
\rowcolor{lightgrey} & LiveCodeBench v3* (\citeyear{jain2024livecodebench}) & Code Gen & - & CoT Code & pass@1 & 0.6 & 0.95 & Split on \texttt{```} & 32768 & - & 10 & - \\
\rowcolor{lightgrey} & ZebraLogic* (\citeyear{lin2025zebralogic}) & Puzzle Solving & - & CoT JSON & Custom & 0.6 & 0.95 & Custom JSON & 32768 & - & 1 & - \\
\rowcolor{lightgrey} & BigBench-Hard (\citeyear{suzgun2022bbh}) & Puzzle Solving & - & CoT EM & EM Flex & 0.6 & 0.95 & Custom Regex & 32768 & - & 1 & 23 \\
\rowcolor{lightgrey} & GPQA* (\citeyear{rein2024gpqa}) & General QA & - & CoT MC & Acc & 0.6 & 0.95 & Custom Regex & 32768 & - & 1 & - \\
\rowcolor{lightgrey} & MMLU (\citeyear{hendryckstest2021}) & General QA & - & CoT MC & Acc & 0.6 & 0.95 & Custom Regex & 32768 & - & 1 & 57 \\
\rowcolor{lightgrey} & PopQA (\citeyear{mallen2022popqa}) & Trivia QA & - & CoT MC & Acc & 0.6 & 0.95 & EM Recall & 32768 & - & 1 & - \\
\rowcolor{lightgrey} & Alpaca Eval v2 (\citeyear{dubois2024alpacaeval}) & - & - & CoT & Winrate & 0.6 & 0.95 & - & 32768 & - & 1 & - \\

\bottomrule
\end{tabular}
\end{scriptsize}
  \caption{
  \textbf{Details of the post-training evaluation suite}, as used in OLMo 3~\citep{olmo2025olmo3}.
  We mark tasks with * to indicate new additions compared to the OLMo 2 suite~\citep{olmo2}. All evaluation generations have thinking traces (text between \texttt{<think>...</think>}) stripped before passing to the answer scorer. We use zero-shot setting for all metrics.
}
  \label{tab:eval-config-posttrain}
\end{table*}

\section{Proofs: Expressive Power of Hybrid Models} \label{sec:hybrid-expressivity-proofs}

Throughout this section, we assume by default that complexity classes (e.g., $\TC^0$) refer to their $\FO$-uniform variants (i.e., $\FO$-uniform $\TC^0$).
We formalize \emph{bounded precision} to mean logarithmic precision, i.e., on input sequences of length $n$, we compute our model with $c \log n$ bits of precision, for some $c$.

Theoretical analysis of transformers makes various assumptions about the types of attention allowed \citep{hao-etal-2022-formal,strobl2024formal}.
In our results, we will consider two types.
First, we will consider unique-hard-attention transformers (UHATs), where attention can only attend to one unique position.
Additionally, we will consider averaging-hard-attention transformers (AHATs), where attention weight is uniformly distributed over all positions that maximize the attention score.
Our AHAT definition also allows masked pre-norm as in \citet[Section 2.1]{merrill2025exact}.
Our constructions for UHATs also work for AHATs, but we state them for UHATs for more generality.

We first establish the negative side of the main theorems (\Cref{thm:pointer-recall,cor:formula-eval}) via core lemmas in \Cref{sec:negative-results}.
We then complete the proofs by giving explicit constructions for hybrid models solving these problems.
Finally, we turn to giving a complete padded characterization.

\subsection{Limitations of Transformers and RNNs} \label{sec:negative-results}

We will now formalize the limitations of transformers and RNNs on several problems of interest via two core lemmas.
The same problems are inexpressible for these architectures for different reasons.
In the case of transformers, it is because these problems are inherently sequential, which we formalize via circuit complexity.
In the case of RNNs, it is because they require remembering a significant amount of information from the prefix of the string, which we formalize using communication complexity.

We first establish the inability of fixed-depth transformers (both AHAT and UHAT) to express the key problems mentioned in \Cref{thm:pointer-recall} and \Cref{cor:formula-eval}:

\begin{lemma} \label{lem:transformers-tc0}
    Assuming $\TC^0 \neq \NC^1$,
    fixed-depth transformers cannot solve problems that are $\NC^1$-hard under $\FO$ reductions,
    which includes the $A_5$ word problem, state-based recall over 5 variables, and formula evaluation.
    This holds even with polynomial padding.
\end{lemma}

\begin{proof}
    This follows from the fact that fixed-depth AHATs, UHATs, and softmax transformers (even with padding) are in $\TC^0$ \citep{merrill-sabharwal-2023-parallelism,chiang2025transformers}.
    If transformers could solve a problem hard for $\NC^1$ under $\FO$ reductions, then any $\NC^1$ problem could be solved in $\TC^0$ by composing an $\AC^0$ circuit for the $\FO$ reduction with the $\TC^0$ circuit for the complete problem.
    This would imply $\TC^0 = \NC^1$, which contradicts the given premise.

    We now justify that each of the mentioned problems is $\NC^1$-hard:
    \begin{itemize}
        \item The $\NC^1$ completeness of the $A_5$ word problem is a fundamental result \citep{barrington1986bounded}. This holds even for the variant where the input string is a sequence of transpositions over 5 elements and the output is the number that 1 gets mapped to \citep[Section 3]{merrill2024illusion}.
        \item To show state-based recall is $\NC^1$-complete, we construct an $\FO$ reduction from the transposition variant of the $A_5$ word problem as follows.
        Given a sequence of transpositions $w$, instantiate a fixed-length list where the first entry is 1 and others are 0.
        Then simply copy over the list of transpositions from the $A_5$ instance.
        By construction, the output reconstructs the number that $w$ maps 1 to.
        \item The $\NC^1$-hardness of formula evaluation over booleans or integers is straightforward. Moreover, evaluating boolean formulas is also $\NC^1$-complete \citep{buss1987alogtime}, whereas evaluating integer formulas is $\PNC^1$-complete \citep{caussinus1998nondeterministic}. \qedhere
    \end{itemize}
\end{proof}

Next we formalize the memory limitations of RNNs (linear or nonlinear) using standard techniques from communication complexity:

\begin{lemma} \label{lemma:rnns-logn}
    Log-precision RNNs (including DeltaNet) cannot solve problems with $\Omega(n)$ communication complexity, which includes recall, state-based recall, and formula evaluation in Polish notation.
    This result holds even with polynomial padding.
\end{lemma}

\begin{proof}
    It is natural that the bounded state size of RNNs restricts their expressive power \citep{jelassi2024repeat}.
    Log-precision RNNs have a hidden state of size $O(\log n)$, and thus any problem that requires passing $\Omega(n)$ bits from prefix to suffix will not be expressible.

    We next show that each of the listed problems have $\Omega(n)$ communication complexity, i.e., require passing this many bits from prefix to suffix when processing via streaming:
    \begin{itemize}
        \item It is natural that recall problems (e.g., recognizing $ww$) have $\Omega(n)$ communication complexity.
        \item State-based recall can be reduced to recall by a string homomorphism that simply deletes the state tracking tokens. Thus, it also requires $\Omega(n)$ communication complexity.
        \item Finally, evaluating formulas in Polish notation has a communication complexity of $\Omega(n)$ \citep[Theorem 6]{merrill2021linguisticcapacityrealtimecounter}.
    \end{itemize}

    With polynomial padding tokens, precision remains logarithmic in the input sequence length:
    \begin{equation*}
        O(\log (n^c)) = O(c \log n) = O(\log n) .
    \end{equation*}
    Thus, we are still restricted to remembering at most $O(\log n)$ bits of an input prefix.
\end{proof}

Thus, assuming both $\TC^0 \neq \NC^1$ and log precision, both transformers and linear RNNs cannot solve state-based recall, the permuted $A_5$ word problem, or formula evaluation.
As a side note, the same communication complexity argument also shows that RNNs cannot even recognize recall languages like $ww$ and $ww^R$, which have communication complexity $\Omega(n)$. These simple languages are, however, in uniform $\AC^0$ and can also be easily recognized by fixed-depth AHATs \citep{strobl2024formal}.

\subsection{Power of Hybrid Models}

We now turn to proving the main results that hybrid models can solve problems beyond the capabilities of both transformers and RNNs.
In this section, we consider a hybrid model where the transformer component is an AHAT, though the construction also goes through with UHAT blocks as long as there is at least one linear RNN layer prior to it. The inexpressibility part of the result applies to both UHAT and AHAT.
Assume the initial pointer values $p_1, \ldots, p_5$ are encoded in unary (binary encoding is discussed after the following result and its proof) for the state-based recall problem.

\pointerRecall*
\begin{proof}
    \Cref{lem:transformers-tc0} shows transformers cannot solve state-based recall assuming $\TC^0 \neq \NC^1$.
    Similarly, \Cref{lemma:rnns-logn} shows log-precision RNNs cannot express state-based recall unconditionally.
    We now describe how hybrid models with a single alternation can express pointer-based recall, considering each order separately.

    \paragraph{} \textit{(GDN + Attention)}
    First, we use GDN to read the pointers $p_1, \ldots, p_5$ into a residual stream cell, which can be done for either unary- or binary-encoded pointers.
    As mentioned in the main text, the hybrid model can capture state-based recall by first composing permutations over the pointers with a GDN block and then using attention to retrieve the value at $p_1$.
    Since each transposition can be expressed as an identity + rank 1 matrix, it can directly be implemented by GDN \citep{grazzi2025unlocking}.
    In the attention layer that implements recall, we first use layer-norm to project the final value of $p_1$ onto the unit sphere, which we denote $\phi(p_1)$.
    At each input bit $x_i$, we also compute the projection $\phi(i)$ using standard tricks with NoPE \citep{merrill2024cot,merrill2025exact}.
    We then implement the recall by attending with query $\phi(p_1)$ over keys $\phi(i)$ and values $x_i$.
    Since attention is maximized exactly when $p_1 = i$, we will retrieve the correct bit $x_{p_1}$.
    Thus, we have constructed a hybrid model of DeltaNet layers followed by attention layers that solves state-based recall.

    \paragraph{} \textit{(Attention + GDN)}
    The order of layers is reversed, but the overall idea remains similar.
    We first use an AHAT block to compute, for each pointer $k$, quantities $\phi(p_k/i)$ and $\phi(1/i)$, from which $\phi(p_k)$ can be computed.
    This is possible with averaging-hard attention (hence reliance on AHAT) because $p_k$ is encoded in unary.
    Next, we use another layer of 5 attention heads to retrieve each value $x_{p_k}$ from each pointer $p_k$ using query $\phi(p_k)$, key $\phi(i)$, and value $x_i$.
    Finally, we use GDN to implement transposition composition over the values $x_{p_1}, \ldots, x_{p_5}$, rather than over the pointers.
    Thus, in either order, one alternation allows a hybrid model to express state-based recall.
\end{proof}

If GDN precedes attention, then this construction works for binary-encoded pointers in addition to unary-encoded pointers, as well as for UHAT blocks.
It follows that, with more than one alternation, hybrid models (regardless of UHAT or AHAT) can also handle binary-encoded pointers, since this setting subsumes the one where GDN precedes attention.



\subsection{Padded Hybrid Models}

First, we describe padding in a little more detail.
Given a function $t : \mathbb N \to \mathbb N$, a model is run with $t(n)$ padding as follows.
For any input $w$, we append $t(n)$ padding tokens ($\square$) to get a padded input $w \square^{t(\abs{w})}$.
We then interpret the prediction at the final token of this padded input as the prediction for $w$.
We say that a language $L$ can be recognized by a transformer with polynomial padding if there exists $c$ and a transformer $T$ such that $T$ recognizes $L$ with $n^c$ padding.

We show that polynomially-padded hybrid models (with averaging-hard attention) can capture all of $\NC^1$:

\paddedHybrid*

\begin{proof}
    Let $L \in \NC^1$.
    Recognizing $L$ can be decomposed to implementing an $\FO$ reduction to an $\NC^1$-complete problem $L'$.
    Let $L'$ be the transposition variant of the $S_5$ word problem described by \citet[Section 3.1]{merrill2024illusion}.
    We use the fact that every language $L \in \NC^1$ is $\FO$-reducible to $L'$.
    That is, any $w \in \Sigma^n$ can be mapped via an $\FO$ reduction to a new sequence $u$ of length $n^k$ such that $\prod_{i=1}^{n^k} u_i = 1$ if and only if $w \in L$.
    We use a block of transformer layers to implement the $\FO$ reduction from $w$ to $u$. At this point, we have $n^k$ padding tokens where token $i$ encodes $u_i$.
    Next, we can construct a fixed-depth GDN with negative eigenvalues that recognizes $L'$ \citep{grazzi2025unlocking}, which is possible because each element in the transition monoid can be written as an identity + rank 1 operator.
    Thus, by composition, following Lemma 3 of \citet{merrill2025exact}, this hybrid model will accept iff $u \in L'$. By construction, $u \in L'$ iff $w \in L$.
    Thus, given an arbitrary language $L \in \NC^1$, we have constructed a hybrid model that recognizes $L$ with $n^c$ padding, where $c$ is fixed for each $L$.
\end{proof}

Based on recent work, \Cref{thm:padded-hybrid} can be naturally extended to give the stronger lower bound of $\FO$-uniform $\PNC^1$ using the fact that GDN with negative eigenvalues can represent $\PNC^1$-complete problems \citep{caussinus1998nondeterministic,merrill2026lrnns}.
This implies padded hybrid models can solve some additional problems not known to be in $\NC^1$, such as simulated weighted automata or evaluating formulas over integers.
Moreover, nonuniform $\PNC^1$ holds as an upper bound for hybrid models \citep{merrill2026lrnns}.
Thus, modulo details about uniformity, $\PNC^1$ represents an exact expressivity characterization for padded hybrid models.

Moreover, recent related work has studied hybrid models that mix transformers with weighted transducers, approaching such models from the perspective of expressive power and learning theory \citep{mohri2026rationaltransductors}.
Combined with the link between linear RNNs and weighted automata \citep{merrill2026lrnns}, there is a close connection between these models and hybrid models that mix transformers and linear RNNs, with both architectures falling in $\PNC^1$ and capable of expressing $\PNC^1$-complete problems.

\section{Additional Details: Synthetic Evaluations} \label{app:synthetic-evals}

This appendix documents the experimental setup used for the synthetic evaluation curves in \Cref{sec:synthetic-evals},
including model architectures, the hyperparameter search protocol, curricula, and the exact numbers plotted.

\subsection{Tasks and Evaluation Protocol}

All tasks are generated online (no fixed dataset) by sampling code-like strings and training models as next-token predictors.

\paragraph{Recall.}
A bit array of size $m$ is instantiated, followed by a query of the form \texttt{assert bits[i] == \_}.
The model must output the correct bit.
We evaluate across $m \in \{4,8,16,32,64,128\}$.

\paragraph{State Tracking.}
Variables are initialized and updated by $n$ swaps/assignments, followed by a query \texttt{assert v == \_}.
We evaluate across $n \in \{4,8,16,32,64,128\}$.

\paragraph{State-Based Recall.}
A bit array of size $m$ is instantiated; variables are initialized to indices in $[0,m{-}1]$ and then updated by $n$ swaps.
Finally, the model must answer \texttt{assert bits[v] == \_} for a queried variable $v$.
In the default state-based recall setting used in \Cref{sec:synthetic-evals}, we set $m=n$ and evaluate across $n \in \{4,8,16,32,64,128\}$.

\paragraph{Metric.}
We report next-token accuracy on the final answer token (the token immediately following the last \texttt{==}).
Each reported accuracy is computed over $256$ freshly generated evaluation samples at the specified difficulty.

\subsection{Model Architectures}

To address potential confounds from over-parameterization on shorter sequences, we use a compact, standardized architecture across all three models. They share the same base width and depth, differing only in their sequence-mixing layers (full attention vs.\ GDN vs.\ hybrid).
Table~\ref{tab:synth-arch} summarizes the shared hyperparameters and the model-specific settings.

\begin{table}[h]
\centering
\small
\caption{Architecture hyperparameters used for synthetic evaluations. }
\label{tab:synth-arch}
\begin{tabular}{l l}
\toprule
\textbf{Shared hyperparameters} & \textbf{Value} \\
\midrule
Layers ($L$) & $4$ \\
Model width ($d_{\text{model}}$) & $256$ \\
FFN intermediate size ($d_{\text{ff}}$) & $1024$ \\
Attention heads ($H$) & $4$ \\
Head dimension ($d_{\text{head}}$) & $64$ \\
Max positions & $1024$ (Recall) / $4096$ (State Tracking, State-based Recall) \\
\midrule
\textbf{Transformer} & full softmax attention in all layers \\
\textbf{Linear RNN (neg.\ eig.)} & GatedDeltaNet with \texttt{allow\_neg\_eigval=True} \\
\textbf{Linear RNN (pos.\ eig.)} & GatedDeltaNet with \texttt{allow\_neg\_eigval=False} \\
\textbf{Hybrid (neg.\ eig.)} & first 3 layers are GDN with \texttt{allow\_neg\_eigval=True}; last layer is full attention \\
\textbf{Hybrid (pos.\ eig.)} & first 3 layers are GDN with \texttt{allow\_neg\_eigval=False}; last layer is full attention \\
\bottomrule
\end{tabular}
\end{table}

\subsection{Training Details}

We train with bf16 using the HuggingFace \texttt{Trainer} (AdamW). To ensure robust comparisons, we standardize fundamental training parameters across all tasks: a batch size of $32$, gradient accumulation of $1$, and a warmup period of $250$ steps. Each experiment is conducted on a single H100 GPU.

Because model performance on these synthetic algorithmic tasks is highly sensitive to the learning rate, scheduler, and weight initialization, we conduct a systematic hyperparameter sweep rather than relying on arbitrary fixed values.

\paragraph{Recall and State Tracking Sweep.}
For the Recall and State Tracking tasks, we randomly sample $10$ configurations from a grid of common learning rates ($10^{-4}, 3\cdot 10^{-4}, 10^{-3}$) and schedulers (cosine, constant). Models are trained for $50{,}000$ steps (Recall) and $20{,}000$ steps (State Tracking).

\paragraph{State-Based Recall Sweep.}
State-based recall is a harder compositional task and exhibits high sensitivity to initialization. Initial trials indicated that a learning rate of $10^{-3}$ was unstable, causing all architectures to fail to learn. We therefore refined our search space to learning rates of $\{10^{-4}, 3\cdot 10^{-4}\}$ and schedulers $\in \{\text{cosine}, \text{constant}\}$. For each of the $4$ combinations, we run $5$ different random seeds across all three model types, totaling $120$ runs. Models on this task are trained for up to $200{,}000$ steps.
\paragraph{Curricula.}
To facilitate learning on these complex tasks, we employ task-specific curriculum strategies.
For \textbf{State Tracking}, we use a deterministic, time-based step schedule, advancing the difficulty $n$ through $\{8, 16, 32, 64\}$ at fixed cumulative step milestones ($500$, $1{,}500$, $3{,}500$, and $7{,}500$ steps, respectively).
For \textbf{Recall}, the task does not require a curriculum and is trained directly on a fixed bit-array size of $m=128$.
For \textbf{State-Based Recall}, we use a hybrid threshold-and-budget curriculum. The model starts at difficulty $n=8$ and advances through $n \in \{8, 16, 32, 64\}$ if it achieves a $0.95$ accuracy threshold (evaluated every $100$ steps over a hold-out batch of $256$ samples). To prevent stalling, the curriculum also forces an advancement if a step budget is exhausted ($10{,}000$ steps for $n=8$; $30{,}000$ steps for subsequent levels).

\paragraph{Data Generation.}
All tasks are framed as next-token prediction over Python-like execution traces.
For \textbf{State Tracking}, we follow the method in \citet{siems2026learningstatetrackingcodeusing}, where the generator inserts intermediate \texttt{assert} statements to explicitly reveal parts of the evolving program state across $5$ variables. These reveals occur at fixed, difficulty-dependent intervals.
For \textbf{Recall}, the sequence simply consists of a bit-array definition followed by a direct query, without any intermediate state updates.
For \textbf{State-Based Recall}, which composes state tracking and retrieval, we use intermediate \texttt{assert} statements but further strengthen the supervision and discourage pattern-matching through two techniques: first, we randomize the spacing between \texttt{assert} statements across examples (sampling uniformly from powers of $2$ up to $n$); second, we mix in a fraction of strict examples ($20\%$) that omit intermediate \texttt{assert} statements entirely, forcing the model to answer the final query.

\paragraph{No-Negative-Eigenvalue Ablation.}
In addition to the main three-model comparison, we evaluate ablated linear-RNN and hybrid variants in which the GDN blocks disallow negative eigenvalues by setting \texttt{allow\_neg\_eigval=False}.
These ablations use the same task definitions, metrics, and curricula as the main experiments.

\paragraph{Run Selection Criteria.}
To report the final numbers in \Cref{sec:synthetic-evals}, we select the best run for each model and task using a strict, standardized criterion: we identify the run that achieves the maximum accuracy on the hardest difficulty setting (maximum sequence length / max steps) at the final evaluation step. In the event of a tie, we break the tie by looking at the accuracy of the next longest length, continuing until a clear best run is isolated.

To ensure full reproducibility and transparency, the complete set of runs for our hyperparameter sweeps can be viewed on Weights \& Biases at: \url{https://wandb.ai/ai2-llm/Olmo-Hybrid-synth-eval}.

\paragraph{WandB Naming Conventions.}
Note that the task names in the WandB run logs differ slightly from the terminology used in this paper. When navigating the logs, please use the following mapping:
\begin{itemize}
    \item \textbf{Recall} is labeled as \texttt{retrieval} (e.g., \texttt{retrieval\_transformer\_lr1e-4\_bs32\_cosine}). The relevant evaluation metric is the \texttt{eval\_normal} field.
    \item \textbf{State Tracking} is labeled as \texttt{state\_tracking}. The relevant evaluation metric is the \texttt{eval\_strict} field.
    \item \textbf{State-Based Recall} is labeled as \texttt{ptr} (e.g., \texttt{ptr-transformer-lr3e-4-cosine-seed8}). The relevant evaluation metric is the \texttt{eval\_strict} field.
\end{itemize}

The no-negative-eigenvalue ablation projects use the same task naming, but are stored in separate projects whose names end with \texttt{no-neg-eigval}.

The exact numbers plotted in \Cref{fig:synth-results} and detailed in Tables~\ref{tab:synth-results-state}--\ref{tab:synth-results-ptr} are drawn from the following selected best runs:

\paragraph{State Tracking.}
\begin{itemize}
    \item \textbf{Transformer:} \url{https://wandb.ai/ai2-llm/Olmo-Hybrid-synth-eval/runs/925yzp3l}
    \item \textbf{Linear RNN (neg.\ eig.):} \url{https://wandb.ai/ai2-llm/Olmo-Hybrid-synth-eval/runs/i2t6vfgc}
    \item \textbf{Hybrid (neg.\ eig.):} \url{https://wandb.ai/ai2-llm/Olmo-Hybrid-synth-eval/runs/7oxbr81u}
    \item \textbf{Linear RNN (pos.\ eig.):} \url{https://wandb.ai/ai2-llm/Olmo-Hybrid-synth-eval/runs/8hc3oeoq}
    \item \textbf{Hybrid (pos.\ eig.):} \url{https://wandb.ai/ai2-llm/Olmo-Hybrid-synth-eval/runs/2hsy3ksf}
\end{itemize}

\paragraph{Recall.}
\begin{itemize}
    \item \textbf{Transformer:} \url{https://wandb.ai/ai2-llm/Olmo-Hybrid-synth-eval/runs/dattbjr0}
    \item \textbf{Linear RNN (neg.\ eig.):} \url{https://wandb.ai/ai2-llm/Olmo-Hybrid-synth-eval/runs/570kpuuh}
    \item \textbf{Hybrid (neg.\ eig.):} \url{https://wandb.ai/ai2-llm/Olmo-Hybrid-synth-eval/runs/5eljlaeq}
    \item \textbf{Linear RNN (pos.\ eig.):} \url{https://wandb.ai/ai2-llm/Olmo-Hybrid-synth-eval/runs/phcklxlr}
    \item \textbf{Hybrid (pos.\ eig.):} \url{https://wandb.ai/ai2-llm/Olmo-Hybrid-synth-eval/runs/elbinpt6}
\end{itemize}

\paragraph{State-Based Recall.}
\begin{itemize}
    \item \textbf{Transformer:} \url{https://wandb.ai/ai2-llm/Olmo-Hybrid-synth-eval/runs/9rexm8xm}
    \item \textbf{Linear RNN (neg.\ eig.):} \url{https://wandb.ai/ai2-llm/Olmo-Hybrid-synth-eval/runs/jnlavjjn}
    \item \textbf{Hybrid (neg.\ eig.):} \url{https://wandb.ai/ai2-llm/Olmo-Hybrid-synth-eval/runs/oqy3t8n3}
    \item \textbf{Linear RNN (pos.\ eig.):} \url{https://wandb.ai/ai2-llm/Olmo-Hybrid-synth-eval/runs/2onzuhs7}
    \item \textbf{Hybrid (pos.\ eig.):} \url{https://wandb.ai/ai2-llm/Olmo-Hybrid-synth-eval/runs/zn13y4j1}
\end{itemize}

\vspace{1em}

\begin{table}[H]
\centering
\caption{State tracking accuracy (varying number of updates $n$).}
\label{tab:synth-results-state}
\resizebox{\linewidth}{!}{%
\begin{tabular}{c c c c c c}
\toprule
$n$ & Transformer & Linear RNN (neg.\ eig.) & Linear RNN (pos.\ eig.) & Hybrid (neg.\ eig.) & Hybrid (pos.\ eig.) \\
\midrule
4   & 0.51172 & 1.00000 & 1.00000 & 1.00000 & 1.00000 \\
8   & 0.27734 & 1.00000 & 1.00000 & 1.00000 & 1.00000 \\
16  & 0.17188 & 1.00000 & 0.97266 & 1.00000 & 1.00000 \\
32  & 0.23047 & 1.00000 & 0.64453 & 1.00000 & 1.00000 \\
64  & 0.19922 & 1.00000 & 0.21484 & 1.00000 & 0.96094 \\
128 & 0.23047 & 1.00000 & 0.22266 & 1.00000 & 0.36328 \\
\bottomrule
\end{tabular}}
\end{table}

\begin{table}[H]
\centering
\caption{Recall accuracy (varying bit-array size $m$).}
\label{tab:synth-results-recall}
\resizebox{\linewidth}{!}{%
\begin{tabular}{c c c c c c}
\toprule
$m$ & Transformer & Linear RNN (neg.\ eig.) & Linear RNN (pos.\ eig.) & Hybrid (neg.\ eig.) & Hybrid (pos.\ eig.) \\
\midrule
4   & 1.00000 & 1.00000 & 1.00000 & 1.00000 & 1.00000 \\
8   & 1.00000 & 1.00000 & 1.00000 & 1.00000 & 1.00000 \\
16  & 1.00000 & 1.00000 & 1.00000 & 1.00000 & 1.00000 \\
32  & 1.00000 & 1.00000 & 1.00000 & 1.00000 & 1.00000 \\
64  & 1.00000 & 0.83203 & 0.80078 & 1.00000 & 1.00000 \\
128 & 0.96484 & 0.67578 & 0.67188 & 1.00000 & 1.00000 \\
\bottomrule
\end{tabular}}
\end{table}

\begin{table}[H]
\centering
\caption{State-based recall accuracy (varying number of swaps $n$, with $m=n$).}
\label{tab:synth-results-ptr}
\resizebox{\linewidth}{!}{%
\begin{tabular}{c c c c c c}
\toprule
$n$ & Transformer & Linear RNN (neg.\ eig.) & Linear RNN (pos.\ eig.) & Hybrid (neg.\ eig.) & Hybrid (pos.\ eig.) \\
\midrule
4   & 0.82031 & 1.00000 & 0.76953 & 1.00000 & 0.85938 \\
8   & 0.75000 & 1.00000 & 0.99219 & 1.00000 & 0.99219 \\
16  & 0.73828 & 1.00000 & 0.85547 & 1.00000 & 0.93750 \\
32  & 0.73828 & 1.00000 & 0.64453 & 1.00000 & 0.68750 \\
64  & 0.62891 & 0.78125 & 0.59766 & 1.00000 & 0.61328 \\
128 & 0.54297 & 0.63672 & 0.57422 & 1.00000 & 0.57031 \\
\bottomrule
\end{tabular}}
\end{table}

\section{Scaling and Ablation Experiments: Details and Additional Results}

\subsection{Additional Results}

\subsubsection{Additional Scaling Law Results} \label{sec:additional-results}

This section supplements \Cref{sec:scaling-laws} with the full scaling law coefficient table with confidence intervals (\Cref{tab:scaling-law-parameters}), compute-optimal loss predictions across all three architectures and compute budgets (\Cref{tab:compute-equivalent-loss}), projected token savings by scale (\Cref{tab:token-savings-by-scale}), and scaling law residual diagnostics (\Cref{fig:scaling-law-residuals}).

\paragraph{Scaling Coefficients.}
\Cref{tab:scaling-law-parameters} reports the full Chinchilla scaling law coefficients for all three architectures under both the unconstrained fit (all five parameters free) and the fixed-exponent fit ($\alpha = \beta = 0.22$ shared; only $E$, $A$, $B$ refit), together with 95\% bootstrap confidence intervals.
The unconstrained fit yields wide CIs due to high joint uncertainty when exponents are free; the fixed-exponent fit tightens the CIs considerably.
The clearest signal is in $B$ (data efficiency): the hybrid's $B = 83.7$ (CI $[80.2, 87.1]$) is robustly lower than the transformer's $94.9$ (CI $[88.7, 102.0]$), with non-overlapping CIs.
The parameter coefficient $A$ slightly favors the hybrid ($70.1$ vs.\ $71.8$) but the CIs overlap.
In the fixed-exponent fit the ordering of the irreducible loss $E$ reverses relative to the unconstrained fit: the transformer achieves a lower $E$ ($1.569$, CI $[1.54, 1.60]$) than the hybrid ($1.597$, CI $[1.58, 1.61]$), suggesting the apparent advantage in $E$ for the hybrid in the unconstrained fit was a fitting artifact.
The scaling exponents $\alpha$ and $\beta$ are statistically indistinguishable across architectures---the CIs broadly overlap in the unconstrained fit---consistent with the theoretical prediction that expressivity shifts efficiency constants without altering exponents (\Cref{sec:expressivity-to-scaling}).

\begin{table*}[t!]
    \centering
    \small
    \renewcommand{\arraystretch}{1.2}
    \begin{tabular}{l*{8}{c}}
        \toprule
        Architecture & $E$ & $A$ & $\alpha$ & $B$ & $\beta$ & $a_\text{opt}$ & $b_\text{opt}$ & $R^2$ \\
        \midrule
        \multicolumn{9}{l}{\textit{Free fit (all five parameters fit independently)}} \\
        \midrule
        Olmo 3    & \makecell{1.65 \\ {\scriptsize [1.11, 1.87]}} & \makecell{108.83 \\ {\scriptsize [21.2, 448.7]}} & \makecell{0.25 \\ {\scriptsize [0.14, 0.34]}} & \makecell{83.45 \\ {\scriptsize [29.6, 495.0]}} & \makecell{0.21 \\ {\scriptsize [0.15, 0.30]}} & \makecell{0.46 \\ {\scriptsize [0.33, 0.68]}} & \makecell{0.54 \\ {\scriptsize [0.32, 0.67]}} & 0.9976 \\
        Hybrid    & \makecell{1.60 \\ {\scriptsize [1.25, 1.72]}} & \makecell{71.14 \\ {\scriptsize [24.8, 145.3]}} & \makecell{0.23 \\ {\scriptsize [0.15, 0.27]}} & \makecell{81.72 \\ {\scriptsize [37.4, 219.6]}} & \makecell{0.22 \\ {\scriptsize [0.18, 0.27]}} & \makecell{0.49 \\ {\scriptsize [0.41, 0.64]}} & \makecell{0.51 \\ {\scriptsize [0.36, 0.59]}} & 0.9993 \\
        Pure GDN  & \makecell{1.46 \\ {\scriptsize [1.16, 1.75]}} & \makecell{36.24 \\ {\scriptsize [19.0, 113.3]}} & \makecell{0.18 \\ {\scriptsize [0.14, 0.26]}} & \makecell{102.72 \\ {\scriptsize [52.1, 214.8]}} & \makecell{0.23 \\ {\scriptsize [0.19, 0.26]}} & \makecell{0.55 \\ {\scriptsize [0.45, 0.65]}} & \makecell{0.45 \\ {\scriptsize [0.35, 0.55]}} & 0.9994 \\
        \midrule
        \multicolumn{9}{l}{\textit{Fixed-exponent fit ($\alpha = \beta = 0.22$ shared; only $E$, $A$, $B$ refit)}} \\
        \midrule
        Olmo 3    & \makecell{1.55 \\ {\scriptsize [1.52, 1.58]}} & \makecell{66.63 \\ {\scriptsize [61.9, 69.4]}} & \cellcolor{midgrey}0.22 & \makecell{94.85 \\ {\scriptsize [89.3, 101.4]}} & \cellcolor{midgrey}0.22 & \cellcolor{midgrey}0.50 & \cellcolor{midgrey}0.50 & 0.9975 \\
        Hybrid    & \makecell{1.58 \\ {\scriptsize [1.56, 1.59]}} & \makecell{65.09 \\ {\scriptsize [62.8, 67.3]}} & \cellcolor{midgrey}0.22 & \makecell{83.65 \\ {\scriptsize [79.9, 87.0]}} & \cellcolor{midgrey}0.22 & \cellcolor{midgrey}0.50 & \cellcolor{midgrey}0.50 & 0.9993 \\
        Pure GDN  & \makecell{1.61 \\ {\scriptsize [1.59, 1.62]}} & \makecell{62.38 \\ {\scriptsize [60.5, 64.8]}} & \cellcolor{midgrey}0.22 & \makecell{90.80 \\ {\scriptsize [87.4, 93.3]}} & \cellcolor{midgrey}0.22 & \cellcolor{midgrey}0.50 & \cellcolor{midgrey}0.50 & 0.9994 \\
        \arrayrulecolor{black}\bottomrule
    \end{tabular}
    \caption{Fit Chinchilla scaling law parameters for $L(N, D) = E + A/N^\alpha + B/D^\beta$.
    95\% bootstrap CI shown below each estimate.
    \textbf{Top}: unconstrained fit where all five parameters are fit independently per architecture.
    \textbf{Bottom}: fixed-exponent fit where $\alpha = \beta = 0.22$ (in \textcolor{gray}{gray} cells) (approximately the mean of the unconstrained estimates) are shared across architectures; only $E$, $A$, $B$ are refit. Equal exponents imply $a_\text{opt} = b_\text{opt} = 0.5$.
    }
    \label{tab:scaling-law-parameters}
\end{table*}

\Cref{tab:compute-equivalent-loss} reports compute-optimal model sizes, dataset sizes, and predicted losses for all three architectures across compute budgets from $10^{18}$ to $10^{23}$ FLOPs.
The pure GDN achieves modest but consistent loss improvements ($\Delta \approx -0.02$ to $-0.05$) across scales, though these are smaller than those of the hybrid model, reflecting its weaker data efficiency~$B$ despite lower parameter coefficient~$A$.
The compute-optimal allocation for each architecture is derived from the fitted scaling laws as described in \Cref{sec:fig-table-details}.

\begin{table}[htbp]
    \centering
    \caption{Compute-optimal model size ($N^*$, billions), dataset size ($D^*$, billions of tokens), and predicted loss for all three architectures. Each architecture's fitted $a_\text{opt}, b_\text{opt}$ determine its compute-optimal allocation ($C = 6ND$). For each loss entry: {\scriptsize \textcolor{gray}{[loss CI]}} on the second line; {\scriptsize \textcolor{teal}{[$\Delta$ CI]}} on the third line (loss reduction relative to the transformer). All CIs are 95\% bootstrap intervals.}
    \label{tab:compute-equivalent-loss}
    \small
    \begin{tabular}{r rrr r rrr r rrr}
        \toprule
         & \multicolumn{3}{c}{Transformer} & & \multicolumn{3}{c}{Hybrid GDN} & & \multicolumn{3}{c}{Pure GDN} \\
        \cmidrule{2-4} \cmidrule{6-8} \cmidrule{10-12}
        FLOPs & $N^*$ & $D^*$ & Loss & & $N^*$ & $D^*$ & Loss & & $N^*$ & $D^*$ & Loss \\
        \midrule
        $10^{18}$ & 0.2 & 0.8 & \makecell{3.57 \\ {\scriptsize \textcolor{gray}{[3.52, 3.62]}}} &
                  & 0.2 & 0.7 & \makecell{3.46 \\ {\scriptsize \textcolor{gray}{[3.42, 3.49]}} \\ {\scriptsize \textcolor{teal}{[$-$0.19, $-$0.05]}}} &
                  & 0.2 & 1.0 & \makecell{3.53 \\ {\scriptsize \textcolor{gray}{[3.50, 3.55]}} \\ {\scriptsize \textcolor{teal}{[$-$0.10, $+$0.01]}}} \\[2pt]
        $10^{19}$ & 0.6 & 2.9 & \makecell{3.12 \\ {\scriptsize \textcolor{gray}{[3.10, 3.14]}}} &
                  & 0.7 & 2.3 & \makecell{3.04 \\ {\scriptsize \textcolor{gray}{[3.02, 3.05]}} \\ {\scriptsize \textcolor{teal}{[$-$0.11, $-$0.05]}}} &
                  & 0.6 & 2.9 & \makecell{3.10 \\ {\scriptsize \textcolor{gray}{[3.09, 3.10]}} \\ {\scriptsize \textcolor{teal}{[$-$0.05, $+$0.00]}}} \\[2pt]
        $10^{20}$ & 1.6 & 10.2 & \makecell{2.78 \\ {\scriptsize \textcolor{gray}{[2.75, 2.79]}}} &
                  & 2.2 & 7.4 & \makecell{2.71 \\ {\scriptsize \textcolor{gray}{[2.69, 2.72]}} \\ {\scriptsize \textcolor{teal}{[$-$0.10, $-$0.04]}}} &
                  & 2.0 & 8.2 & \makecell{2.76 \\ {\scriptsize \textcolor{gray}{[2.74, 2.77]}} \\ {\scriptsize \textcolor{teal}{[$-$0.04, $+$0.01]}}} \\[2pt]
        $10^{21}$ & 4.7 & 35.5 & \makecell{2.51 \\ {\scriptsize \textcolor{gray}{[2.45, 2.55]}}} &
                  & 7.0 & 24.0 & \makecell{2.46 \\ {\scriptsize \textcolor{gray}{[2.41, 2.48]}} \\ {\scriptsize \textcolor{teal}{[$-$0.11, $+$0.01]}}} &
                  & 7.3 & 22.9 & \makecell{2.49 \\ {\scriptsize \textcolor{gray}{[2.46, 2.52]}} \\ {\scriptsize \textcolor{teal}{[$-$0.07, $+$0.04]}}} \\[2pt]
        $10^{22}$ & 13.5 & 123.6 & \makecell{2.31 \\ {\scriptsize \textcolor{gray}{[2.21, 2.36]}}} &
                  & 21.6 & 77.2 & \makecell{2.27 \\ {\scriptsize \textcolor{gray}{[2.18, 2.30]}} \\ {\scriptsize \textcolor{teal}{[$-$0.14, $+$0.06]}}} &
                  & 26.0 & 64.1 & \makecell{2.27 \\ {\scriptsize \textcolor{gray}{[2.23, 2.33]}} \\ {\scriptsize \textcolor{teal}{[$-$0.10, $+$0.08]}}} \\[2pt]
        $10^{23}$ & 38.7 & 430.2 & \makecell{2.15 \\ {\scriptsize \textcolor{gray}{[2.00, 2.23]}}} &
                  & 67.0 & 248.7 & \makecell{2.12 \\ {\scriptsize \textcolor{gray}{[2.00, 2.16]}} \\ {\scriptsize \textcolor{teal}{[$-$0.17, $+$0.11]}}} &
                  & 93.0 & 179.2 & \makecell{2.10 \\ {\scriptsize \textcolor{gray}{[2.03, 2.18]}} \\ {\scriptsize \textcolor{teal}{[$-$0.14, $+$0.12]}}} \\
        \bottomrule
    \end{tabular}
\end{table}

\Cref{tab:token-savings-by-scale} reports the projected token requirements and savings factors at a discrete set of model scales, corresponding to \Cref{fig:combined-savings}(b) in the main text.

\begin{table}[htbp]
    \centering
    \caption{Projected token requirements across model scales (target loss $= 2.474$, selected as the minimum of all observed training losses). Savings $> 1\times$ means fewer tokens needed than the transformer. 95\% bootstrap CI shown below each estimate. See also \Cref{fig:combined-savings}(b).}
    \label{tab:token-savings-by-scale}
    \begin{tabular}{rrrrrr}
        \toprule
         & Transformer & \multicolumn{2}{c}{Hybrid} & \multicolumn{2}{c}{Pure GDN} \\
        $N$ & $D$ & $D$ & Savings & $D$ & Savings \\
        \midrule
        1B & 776.5B & 587.1B & \makecell{1.32$\times$ \\ {\scriptsize [0.67, 4.02]}} & 949.6B & \makecell{0.82$\times$ \\ {\scriptsize [0.45, 2.50]}} \\
        3B & 90.0B & 57.2B & \makecell{1.57$\times$ \\ {\scriptsize [1.11, 2.50]}} & 79.5B & \makecell{1.13$\times$ \\ {\scriptsize [0.77, 1.63]}} \\
        7B & 35.1B & 20.9B & \makecell{1.68$\times$ \\ {\scriptsize [0.91, 3.47]}} & 27.0B & \makecell{1.30$\times$ \\ {\scriptsize [0.68, 2.30]}} \\
        13B & 21.5B & 12.3B & \makecell{1.75$\times$ \\ {\scriptsize [0.84, 4.12]}} & 15.2B & \makecell{1.41$\times$ \\ {\scriptsize [0.65, 2.75]}} \\
        30B & 13.1B & 7.2B & \makecell{1.82$\times$ \\ {\scriptsize [0.77, 4.98]}} & 8.4B & \makecell{1.56$\times$ \\ {\scriptsize [0.61, 3.44]}} \\
        70B & 9.0B & 4.8B & \makecell{1.89$\times$ \\ {\scriptsize [0.71, 5.81]}} & 5.3B & \makecell{1.70$\times$ \\ {\scriptsize [0.57, 4.23]}} \\
        \bottomrule
    \end{tabular}
\end{table}

\paragraph{Fit Quality.}
The $R^2$ values from \Cref{tab:scaling-law-parameters} ($0.998$ for Olmo 3, $0.999$ for the hybrid model, and $0.999$ for the linear RNN) indicate that the power-law form fits the observed data well.
\Cref{fig:scaling-law-residuals} provides additional diagnostics for the quality of the fits.
Panel~(a) shows (signed) relative prediction errors $((L - \hat{L})/L) \times 100\%$ as a function of predicted loss $\hat{L}$, for all data points across all model sizes.
Residuals are small (within $\pm 1\%$) and scattered around zero, confirming that the Chinchilla parametric form is an adequate model for all three architectures.
Panel~(b) reports the mean absolute relative error grouped by model size, showing that fit quality is consistent across the full range of scales (60M--1B parameters).

\begin{figure}[htbp]
\centering

\includegraphics[width=\textwidth]{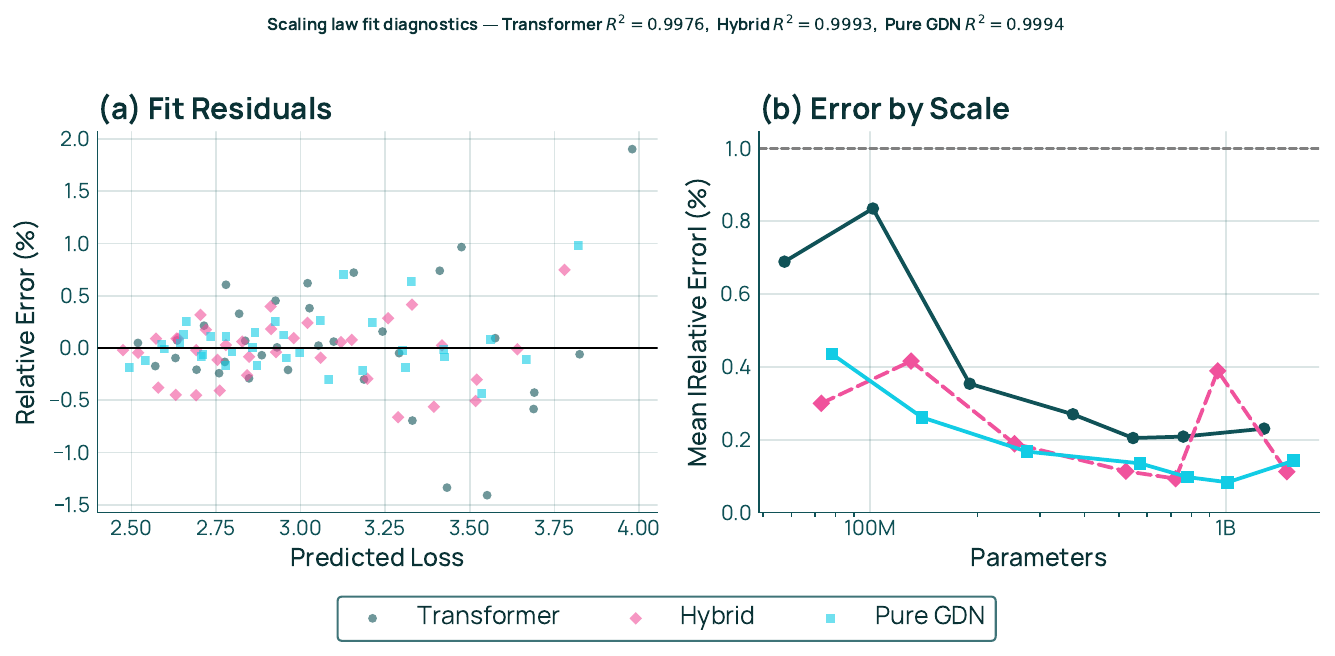}

\caption{Scaling law fit diagnostics. \textbf{(a)} Residuals are small and unbiased for all architectures. \textbf{(b)} Mean absolute relative error stays below 1\% across all model sizes, confirming reliable fits. Overall: Olmo 3 $R^2=0.9976$; Hybrid $R^2=0.9993$; Pure GDN $R^2=0.9994$.}
\label{fig:scaling-law-residuals}
\end{figure}

\begin{figure}[!t]
    \centering
    \includegraphics[width=0.9\linewidth]{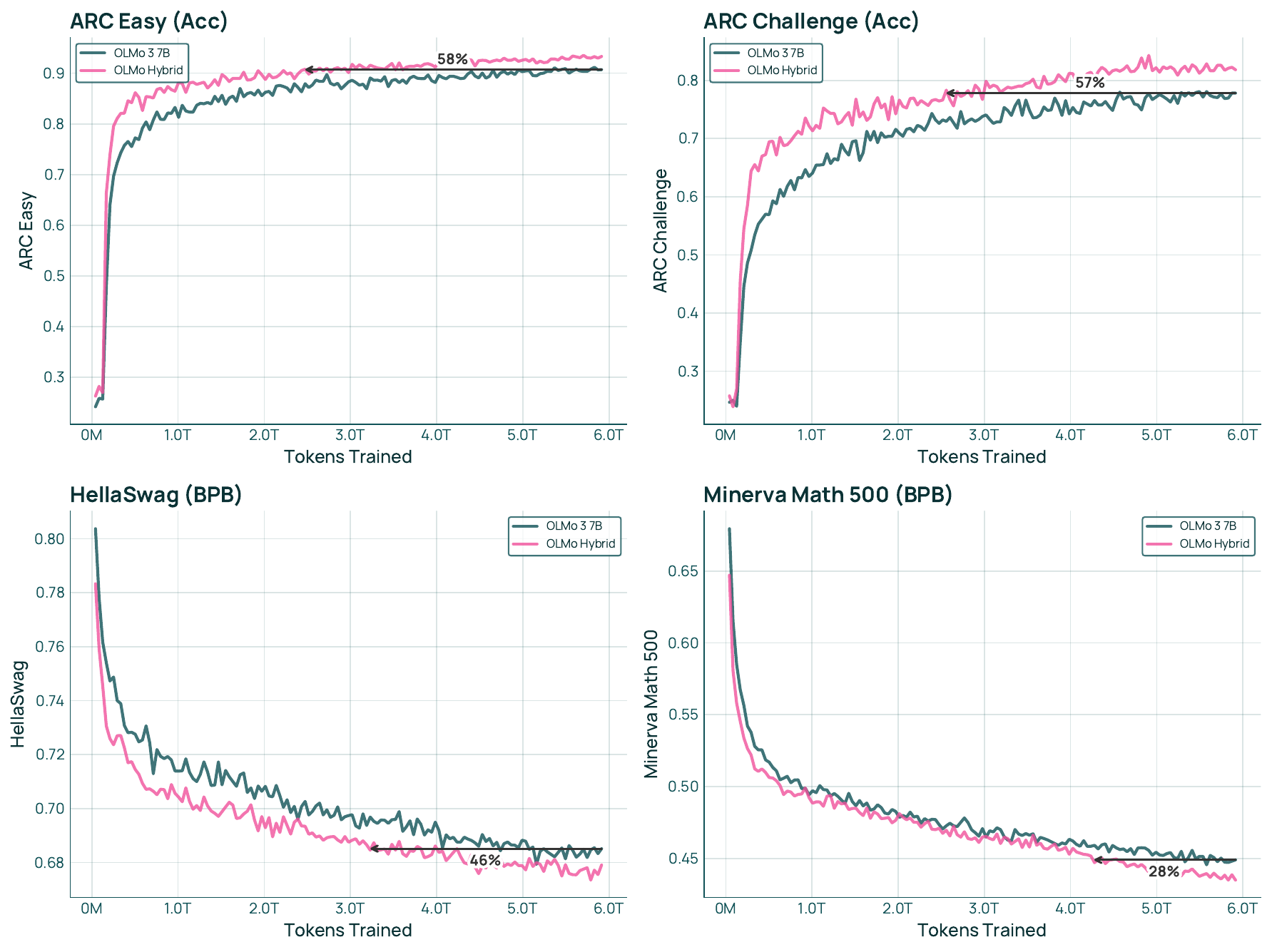}
    \caption{Extended downstream eval training curves comparing \model{} 7B and Olmo 3 7B across 6 benchmarks, sorted by token efficiency (see \Cref{fig:pretraining-perf} for CE loss and MMLU training curves). \model{} reaches Olmo 3's final performance using 19--58\% fewer tokens depending on the benchmark.}
    \label{fig:pretraining-perf-extended}
\end{figure}

\subsubsection{Ablation Results} \label{sec:ablation-results-appendix}

This section supplements \Cref{sec:architecture-ablations} with additional scaling plots across all tested architectures (\Cref{fig:base_easy_avg}) and per-domain BPB evaluations for all configurations at all scales (\Cref{tab:ablation-evals-detail-a}).

\paragraph{Scaling Trends of All Tested Architectures.}
\Cref{fig:base_easy_avg} plots the scaling of the \olmothreeeval average score against training FLOPs for all architectures.
Notably, the 7:1 configuration also shows very favorable scaling behavior, likely owing to the lower FLOP counts from having fewer transformer layers.
The per-domain breakdown (\Cref{fig:base_easy_avg}(c)--(e)) shows that the advantage of hybrid models is present across Math, Code, and QA.
Fine-grained scaling plots broken down by ablation group are provided in \Cref{fig:ablations_pure_vs_tf,fig:ablations_hybrid_vs_tf,fig:ablations_hybrid_ratios,fig:ablations_pure_gdns,fig:ablations_hybrid_gdns}, showing the average and the per-cluster BPB (Math, Code, QA) as a function of training FLOPs for each configuration.

\begin{figure*}[htbp]
\centering
\includegraphics[width=\textwidth]{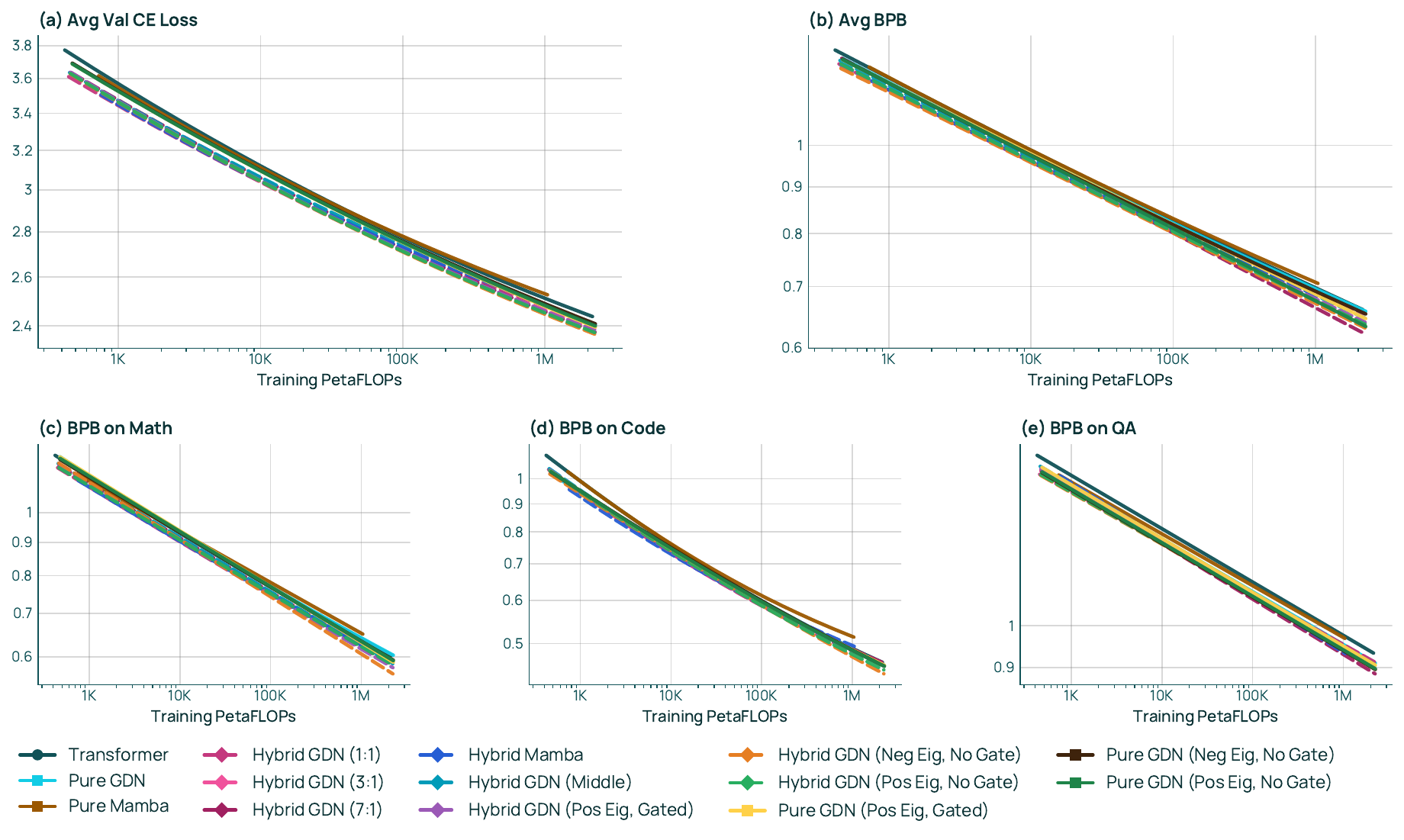}
\caption{Evaluation metrics vs.\ training FLOPs. \textbf{(a)} Average validation cross-entropy loss across 11 held-out corpora (C4, Dolma Books/CC/pes2o/Reddit/Stack/Wiki, ICE, M2D2~S2ORC, Pile, Wikitext-103; lower is better); this is the same loss used to fit the Chinchilla scaling laws. \textbf{(b)} Base Easy Suite average BPB (lower is better). \textbf{(c)--(e)} Per-cluster Base Easy BPB for Math, Code, and QA respectively. The thick line is the fitted compute-optimal frontier.}
\label{fig:base_easy_avg}
\end{figure*}

\begin{figure*}[htbp]
\centering
\includegraphics[width=\textwidth]{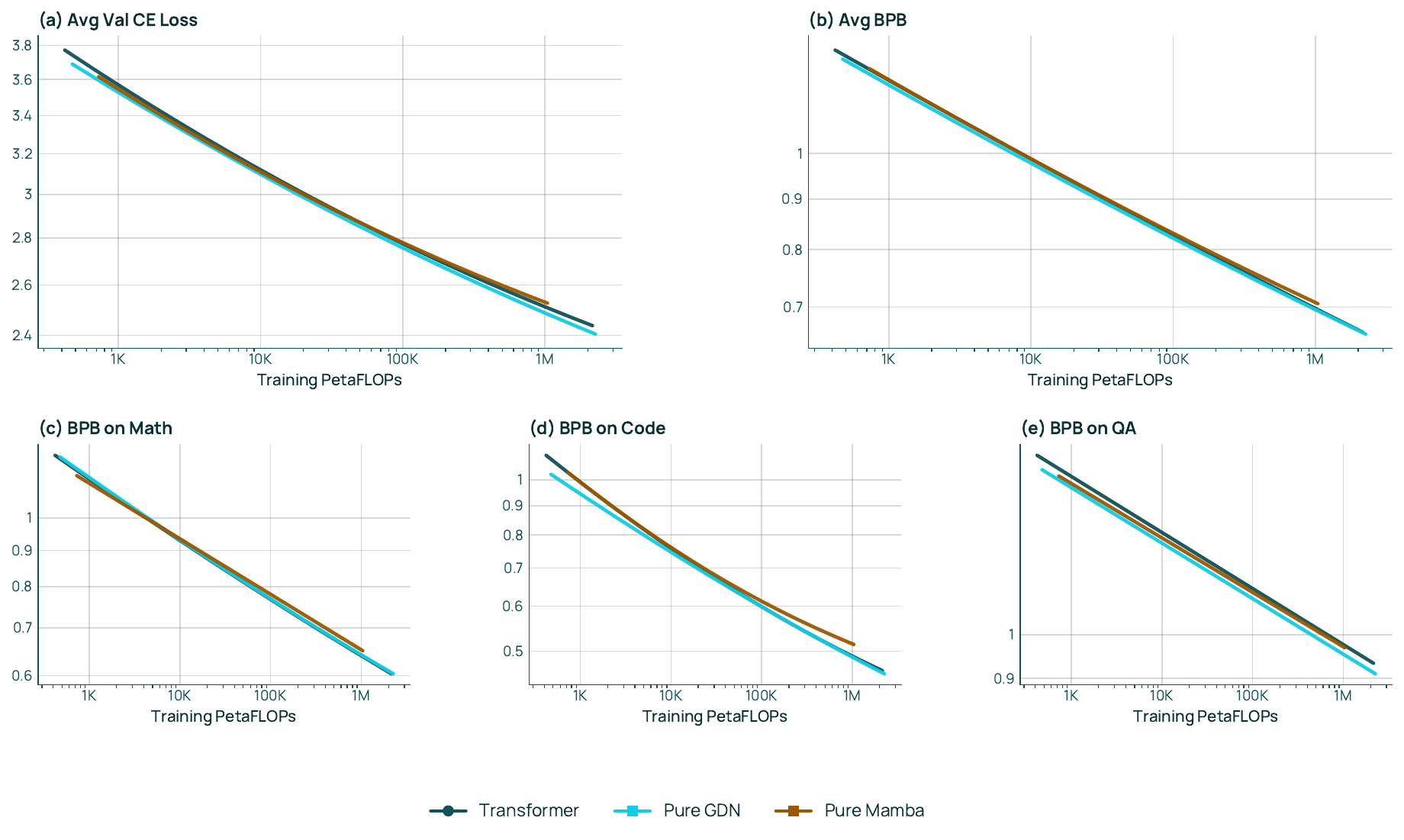}
\caption{Evaluation metrics vs.\ training FLOPs for \emph{pure} architectures (Transformer, pure GDN, pure Mamba2). Panels follow the same layout as \Cref{fig:base_easy_avg}: (a) average validation cross-entropy loss, (b) Base Easy Suite average BPB, and (c)--(e) per-cluster BPB for Math, Code, and QA. The thick line is the fitted compute-optimal frontier.}
\label{fig:ablations_pure_vs_tf}
\end{figure*}

\begin{figure*}[htbp]
\centering
\includegraphics[width=\textwidth]{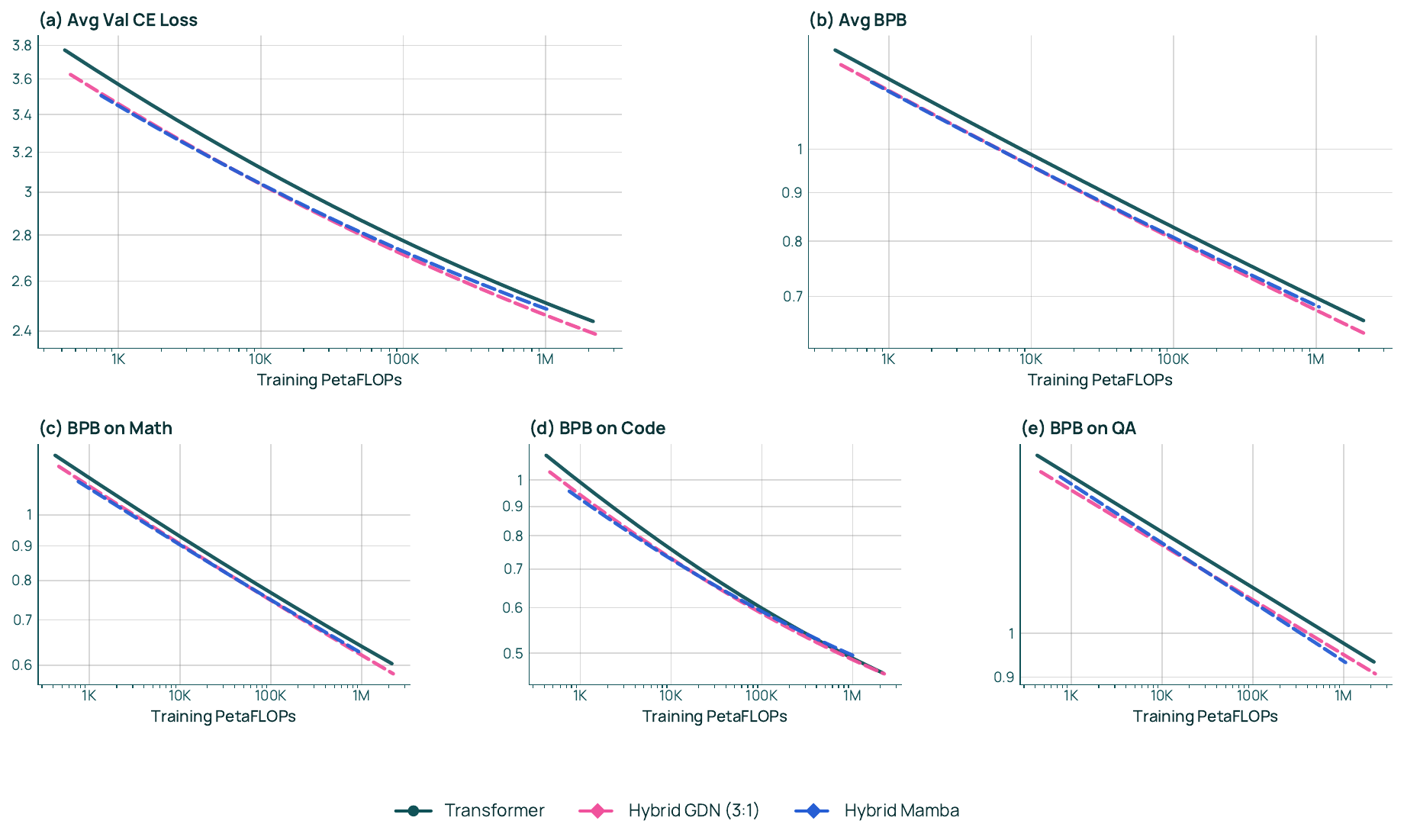}
\caption{Evaluation metrics vs.\ training FLOPs for \emph{hybrid} architectures (GDN hybrid with interleaved and middle placement, Mamba2 hybrid). Panels follow the same layout as \Cref{fig:base_easy_avg}: (a) average validation cross-entropy loss, (b) Base Easy Suite average BPB, and (c)--(e) per-cluster BPB for Math, Code, and QA. The thick line is the fitted compute-optimal frontier.}
\label{fig:ablations_hybrid_vs_tf}
\end{figure*}

\begin{figure*}[htbp]
\centering
\includegraphics[width=\textwidth]{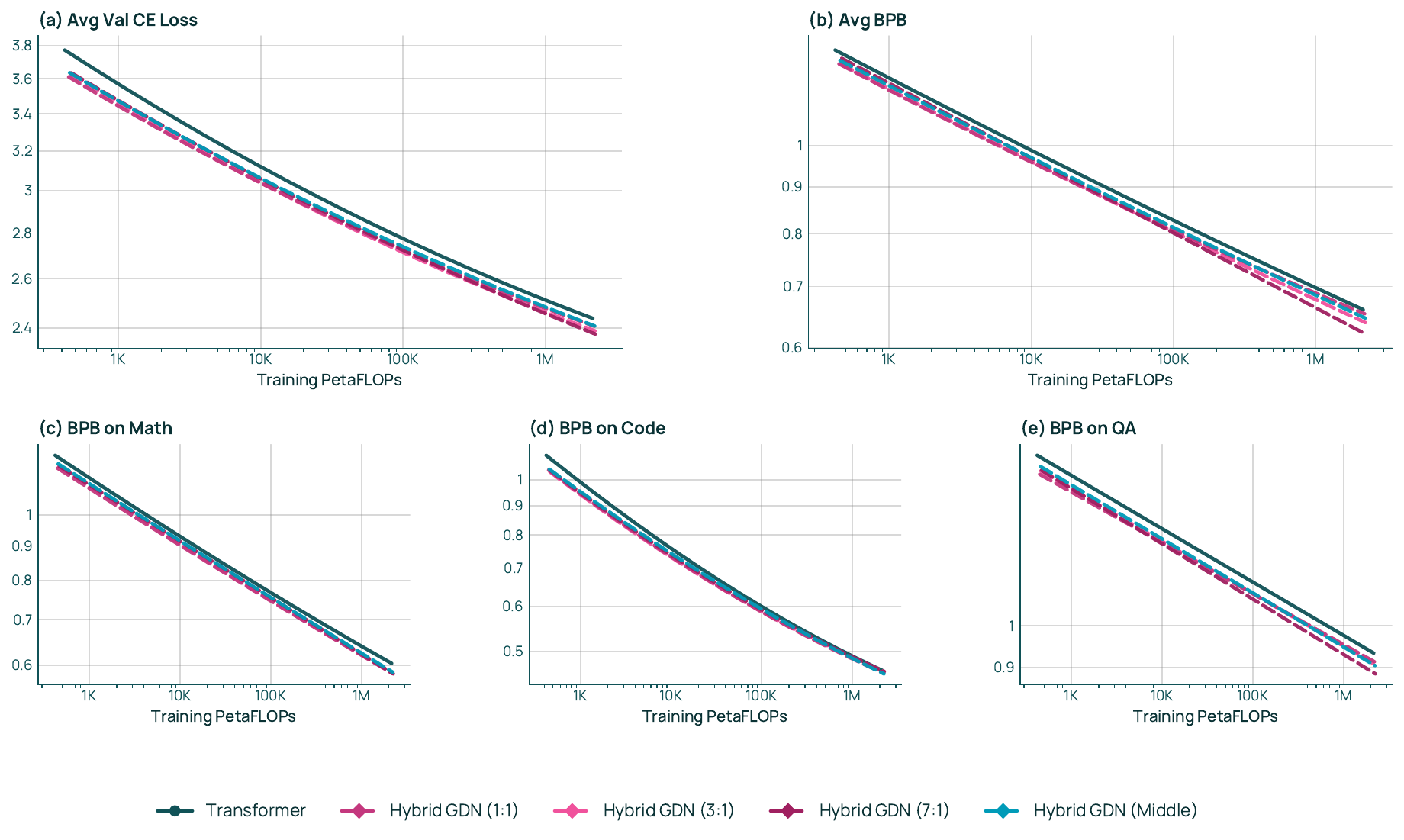}
\caption{Evaluation metrics vs.\ training FLOPs for GDN hybrid models with varying linear-to-attention ratios (1:1, 3:1, 7:1). Panels follow the same layout as \Cref{fig:base_easy_avg}: (a) average validation cross-entropy loss, (b) Base Easy Suite average BPB, and (c)--(e) per-cluster BPB for Math, Code, and QA. The thick line is the fitted compute-optimal frontier.}
\label{fig:ablations_hybrid_ratios}
\end{figure*}

\begin{figure*}[htbp]
\centering
\includegraphics[width=\textwidth]{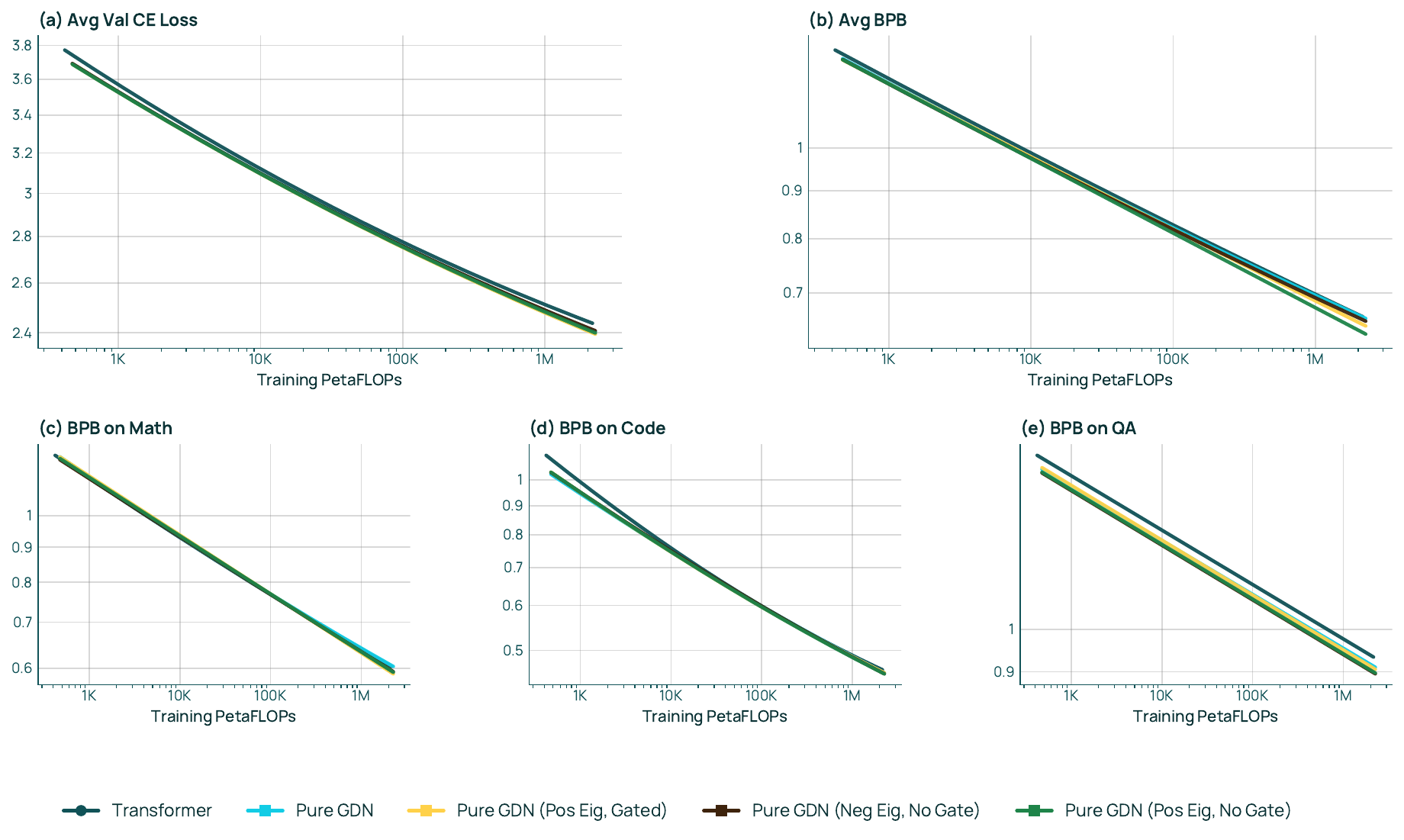}
\caption{Evaluation metrics vs.\ training FLOPs for GDN ablations (e.g., GDN with and without negative eigenvalues). Panels follow the same layout as \Cref{fig:base_easy_avg}: (a) average validation cross-entropy loss, (b) Base Easy Suite average BPB, and (c)--(e) per-cluster BPB for Math, Code, and QA. The thick line is the fitted compute-optimal frontier.}
\label{fig:ablations_pure_gdns}
\end{figure*}

\begin{figure*}[htbp]
\centering
\includegraphics[width=\textwidth]{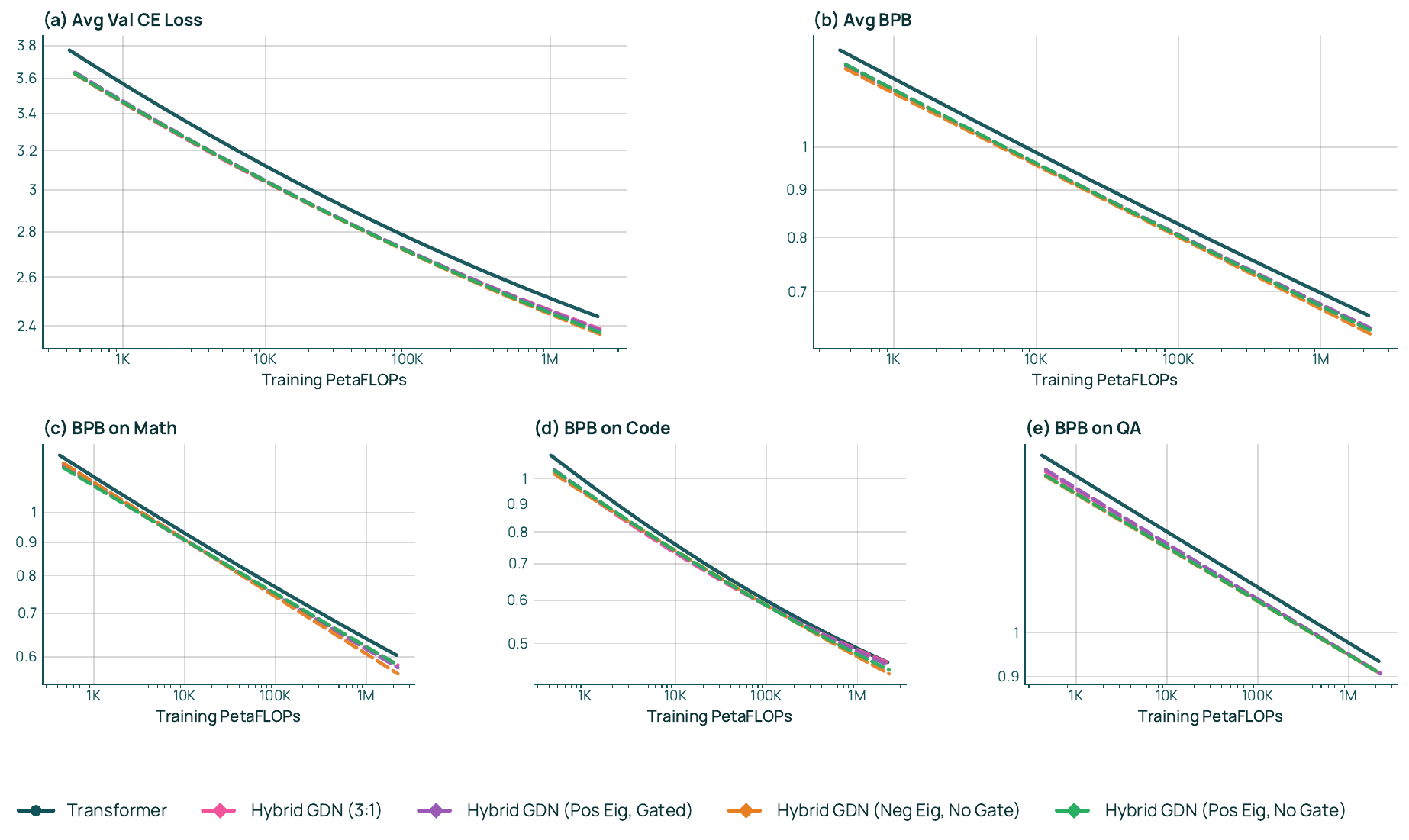}
\caption{Evaluation metrics vs.\ training FLOPs for GDN hybrid ablations (e.g., hybrid GDN with and without negative eigenvalues). Panels follow the same layout as \Cref{fig:base_easy_avg}: (a) average validation cross-entropy loss, (b) Base Easy Suite average BPB, and (c)--(e) per-cluster BPB for Math, Code, and QA. The thick line is the fitted compute-optimal frontier.}
\label{fig:ablations_hybrid_gdns}
\end{figure*}

\paragraph{Per-domain Breakdown.}
\Cref{tab:ablation-evals-detail-a} extends \Cref{tab:ablation-evals-avg} and provides the per-domain BPB breakdown for all ablation configurations evaluated in \Cref{sec:architecture-ablations}.
These verify that the aggregate loss improvements are consistent across domains.

\begin{table*}[t]
    \centering
    \caption{
    \textbf{Architecture ablation results --- per-domain breakdown.}
    Per-domain \olmothreeeval BPB for pure and hybrid architectures.
    \textbf{Bold} indicates best; \underline{underline} second best; $^\dagger$ best within section.
    $\bigstar$ marks our selected architecture.
    Note that per-size comparisons should be interpreted with care, as architectures at the same nominal scale differ in actual parameter count (see \Cref{tab:ablation-configs}).
    }
    \label{tab:ablation-evals-detail-a}
    \small
    \textit{Small scales (60M--370M parameters):}\\[2pt]
    \resizebox{\textwidth}{!}{%
    \begin{tabular}{@{}l c ccc ccc ccc ccc@{}}
        \toprule
        &  & \multicolumn{3}{c}{\textbf{60M}} & \multicolumn{3}{c}{\textbf{100M}} & \multicolumn{3}{c}{\textbf{190M}} & \multicolumn{3}{c}{\textbf{370M}} \\
        \cmidrule(lr){3-5} \cmidrule(lr){6-8} \cmidrule(lr){9-11} \cmidrule(lr){12-14}
        \textbf{Architecture} & \textbf{Attn \%} & Math & Code & QA & Math & Code & QA & Math & Code & QA & Math & Code & QA \\
        \midrule
        \multicolumn{14}{l}{\textit{Pure Architectures}} \\
        \quad Transformer & 100\% & 1.102 & 0.931 & 1.433 & 0.981 & 0.804 & 1.326 & 0.887 & 0.707 & 1.257 & 0.782 & 0.616 & 1.119 \\
        \quad GDN & 0\% & 1.040$^\dagger$ & \underline{0.837} & 1.363$^\dagger$ & 0.949$^\dagger$ & 0.752$^\dagger$ & 1.270$^\dagger$ & 0.839$^\dagger$ & 0.671$^\dagger$ & 1.176$^\dagger$ & 0.746$^\dagger$ & \textbf{0.573} & \textbf{1.066} \\
        \quad Mamba2 & 0\% & 1.110 & 0.917 & 1.411 & 0.994 & 0.798 & 1.316 & 0.896 & 0.711 & 1.215 & 0.790 & 0.616 & 1.122 \\
        \midrule
        \multicolumn{14}{l}{\textit{Hybrid: Interleaved Attention}} \\
        \quad GDN (1:1) & 50\% & 1.044 & 0.861 & 1.374 & 0.936 & 0.765 & 1.273 & 0.839 & 0.673 & 1.178 & 0.744 & 0.585 & 1.083 \\
        \quad GDN (3:1)$^\bigstar$ & 25\% & 1.034 & 0.849$^\dagger$ & \textbf{1.349} & 0.934 & 0.752$^\dagger$ & 1.271 & \textbf{0.835} & \textbf{0.663} & 1.177 & 0.741$^\dagger$ & 0.588 & 1.069$^\dagger$ \\
        \quad GDN (7:1) & 12.5\% & 1.032$^\dagger$ & 0.850 & 1.365 & 0.932$^\dagger$ & 0.764 & 1.269$^\dagger$ & 0.840 & 0.668 & \underline{1.169} & 0.745 & 0.584$^\dagger$ & 1.080 \\
        \quad Mamba2 (3:1) & 25\% & 1.080 & 0.905 & 1.404 & 0.957 & 0.786 & 1.292 & 0.863 & 0.698 & 1.202 & 0.770 & 0.603 & 1.110 \\
        \midrule
        \multicolumn{14}{l}{\textit{Hybrid GDN (3:1): Middle Placement}} \\
        \quad Interleaved$^\bigstar$ & 25\% & 1.034$^\dagger$ & 0.849$^\dagger$ & \underline{1.349} & 0.934$^\dagger$ & 0.752$^\dagger$ & 1.271$^\dagger$ & \underline{0.835} & \underline{0.663} & 1.177$^\dagger$ & 0.741$^\dagger$ & 0.588 & 1.069$^\dagger$ \\
        \quad Middle & 25\% & 1.041 & 0.868 & 1.353 & 0.934 & 0.754 & 1.275 & 0.842 & 0.671 & 1.183 & 0.743 & 0.576$^\dagger$ & 1.081 \\
        \midrule
        \multicolumn{14}{l}{\textit{Hybrid GDN (3:1): Gate/Eigenvalue Ablations}} \\
        \quad Neg EV, gate$^\bigstar$ & 25\% & 1.034 & 0.849 & 1.349$^\dagger$ & 0.934 & 0.752 & 1.271 & 0.835$^\dagger$ & 0.663$^\dagger$ & 1.177 & 0.741 & 0.588 & 1.069 \\
        \quad Neg EV, no gate & 25\% & \underline{1.026} & 0.838$^\dagger$ & 1.378 & 0.930 & \textbf{0.733} & \textbf{1.264} & 0.837 & 0.673 & 1.171$^\dagger$ & \underline{0.739} & 0.595 & 1.070 \\
        \quad Pos EV, gate & 25\% & 1.033 & 0.849 & 1.376 & \textbf{0.928} & 0.746 & \underline{1.267} & 0.839 & 0.675 & 1.176 & 0.764 & 0.598 & 1.122 \\
        \quad Pos EV, no gate & 25\% & 1.037 & 0.843 & 1.366 & \underline{0.929} & 0.757 & 1.278 & 0.839 & 0.681 & 1.183 & \textbf{0.738} & 0.579$^\dagger$ & 1.067$^\dagger$ \\
        \midrule
        \multicolumn{14}{l}{\textit{Pure GDN: Gate/Eigenvalue Ablations}} \\
        \quad Neg EV, gate & 0\% & 1.040 & 0.837 & 1.363 & 0.949 & 0.752 & 1.270$^\dagger$ & 0.839$^\dagger$ & 0.671 & 1.176 & 0.746 & \underline{0.573} & \underline{1.066} \\
        \quad Pos EV, gate & 0\% & \textbf{1.026} & \textbf{0.834} & 1.352$^\dagger$ & 0.936$^\dagger$ & \underline{0.740} & 1.271 & 0.848 & 0.664$^\dagger$ & \textbf{1.167} & 0.745$^\dagger$ & 0.577 & 1.073 \\
        \quad Neg EV, no gate & 0\% & 1.056 & 0.846 & 1.360 & 0.958 & 0.762 & 1.276 & 0.857 & 0.678 & 1.176 & 0.757 & 0.599 & 1.071 \\
        \quad Pos EV, no gate & 0\% & 1.052 & 0.856 & 1.369 & 0.963 & 0.751 & 1.292 & 0.851 & 0.683 & 1.170 & 0.758 & 0.585 & 1.075 \\
        \bottomrule
    \end{tabular}%
    }
    \medskip
    \phantomsection\label{tab:ablation-evals-detail-b}
    \textit{Large scales (600M--1B parameters):}\\[2pt]
    \resizebox{\textwidth}{!}{%
    \begin{tabular}{@{}l c ccc ccc ccc@{}}
        \toprule
        &  & \multicolumn{3}{c}{\textbf{600M}} & \multicolumn{3}{c}{\textbf{760M}} & \multicolumn{3}{c}{\textbf{1B}} \\
        \cmidrule(lr){3-5} \cmidrule(lr){6-8} \cmidrule(lr){9-11}
        \textbf{Architecture} & \textbf{Attn \%} & Math & Code & QA & Math & Code & QA & Math & Code & QA \\
        \midrule
        \multicolumn{11}{l}{\textit{Pure Architectures}} \\
        \quad Transformer & 100\% & 0.724 & 0.564 & 1.055 & 0.686 & 0.538 & 1.015 & 0.622$^\dagger$ & 0.492 & 0.933 \\
        \quad GDN & 0\% & 0.705$^\dagger$ & 0.551$^\dagger$ & 1.026$^\dagger$ & 0.673$^\dagger$ & 0.522$^\dagger$ & 0.970$^\dagger$ & 0.624 & 0.485$^\dagger$ & 0.921$^\dagger$ \\
        \quad Mamba2 & 0\% & 0.736 & 0.577 & 1.047 & 0.699 & 0.547 & 1.005 & 0.662 & 0.524 & 0.969 \\
        \midrule
        \multicolumn{11}{l}{\textit{Hybrid: Interleaved Attention}} \\
        \quad GDN (1:1) & 50\% & 0.725 & 0.571 & 1.065 & 0.663 & 0.527 & 0.983 & 0.610 & 0.484$^\dagger$ & 0.928 \\
        \quad GDN (3:1)$^\bigstar$ & 25\% & 0.694 & 0.546 & 1.022 & 0.658$^\dagger$ & 0.523$^\dagger$ & \textbf{0.969} & \textbf{0.604} & 0.488 & 0.914$^\dagger$ \\
        \quad GDN (7:1) & 12.5\% & \textbf{0.693} & \underline{0.543} & \underline{1.009} & 0.660 & 0.524 & 0.978 & 0.612 & 0.486 & 0.927 \\
        \quad Mamba2 (3:1) & 25\% & 0.714 & 0.561 & 1.042 & 0.673 & 0.531 & 0.992 & 0.638 & 0.509 & 0.946 \\
        \midrule
        \multicolumn{11}{l}{\textit{Hybrid GDN (3:1): Middle Placement}} \\
        \quad Interleaved$^\bigstar$ & 25\% & 0.694$^\dagger$ & 0.546$^\dagger$ & 1.022 & 0.658$^\dagger$ & 0.523 & \underline{0.969} & \underline{0.604} & 0.488 & 0.914$^\dagger$ \\
        \quad Middle & 25\% & 0.698 & 0.547 & 1.015$^\dagger$ & 0.661 & 0.521$^\dagger$ & 0.982 & 0.608 & 0.487$^\dagger$ & 0.922 \\
        \midrule
        \multicolumn{11}{l}{\textit{Hybrid GDN (3:1): Gate/Eigenvalue Ablations}} \\
        \quad Neg EV, gate$^\bigstar$ & 25\% & 0.694 & 0.546 & 1.022 & 0.658 & 0.523 & 0.969$^\dagger$ & 0.604$^\dagger$ & 0.488 & 0.914 \\
        \quad Neg EV, no gate & 25\% & 0.696 & 0.548 & 1.012$^\dagger$ & \textbf{0.654} & 0.526 & 0.983 & 0.610 & \textbf{0.483} & \textbf{0.904} \\
        \quad Pos EV, gate & 25\% & 0.695 & 0.546 & 1.030 & 0.659 & \textbf{0.514} & 0.970 & 0.605 & \underline{0.484} & \underline{0.912} \\
        \quad Pos EV, no gate & 25\% & \underline{0.694} & 0.545$^\dagger$ & 1.021 & \underline{0.656} & 0.518 & 0.977 & 0.605 & 0.486 & 0.918 \\
        \midrule
        \multicolumn{11}{l}{\textit{Pure GDN: Gate/Eigenvalue Ablations}} \\
        \quad Neg EV, gate & 0\% & 0.705 & 0.551 & 1.026 & 0.673 & 0.522 & 0.970$^\dagger$ & 0.624 & 0.485$^\dagger$ & 0.921 \\
        \quad Pos EV, gate & 0\% & 0.705$^\dagger$ & \textbf{0.542} & \textbf{1.004} & 0.669$^\dagger$ & \underline{0.516} & 0.975 & 0.614$^\dagger$ & 0.492 & 0.917$^\dagger$ \\
        \quad Neg EV, no gate & 0\% & 0.710 & 0.546 & 1.026 & 0.679 & 0.527 & 0.986 & 0.623 & 0.489 & 0.929 \\
        \quad Pos EV, no gate & 0\% & 0.712 & 0.559 & 1.024 & 0.677 & 0.526 & 0.977 & 0.626 & 0.488 & 0.919 \\
        \bottomrule
    \end{tabular}%
    }
\end{table*}

\subsection{Details on the Scaling Law Fits} \label{sec:fig-table-details}

\paragraph{Data Used for Fitting.}
Each architecture's scaling ladder consists of models trained at seven sizes and five Chinchilla multiples per size ($D/N \in \{10, 20, 40, 80, 160\}$), using post-decay checkpoints only (i.e., checkpoints after the learning-rate decay completes).
The loss signal is the \emph{average} of 11 held-out validation cross-entropy losses spanning diverse domains (C4, books, Common Crawl, Pes2o, Reddit, Stack, Wikipedia, ICE, S2ORC, Pile, and Wikitext-103); this averaging reduces noise from any single domain.
Non-embedding parameter counts are used throughout (embedding and language-model head parameters are excluded), and architecture-specific FLOPs are computed analytically using the parallel (chunkwise) formulations described in \Cref{sec:flop-computation}.

\paragraph{Scaling Law Fitting.}
We fit the Chinchilla parametric form $L(N, D) = E + A/N^\alpha + B/D^\beta$ by minimizing a Huber loss over the log-residuals, using a multi-start grid optimization (with six slices of the parameter space) to avoid local minima.
Bootstrap confidence intervals (95\%, $n = 1\text{,}000$ resamples) are computed by resampling data points with replacement and re-fitting.

\paragraph{Construction of \Cref{fig:scaling-law-fit-log-opt}.}
In each panel, thin curves connect the five post-decay checkpoints (one per Chinchilla multiple) for a given model size, and a bold curve shows the fitted scaling law $L(N, D)$ evaluated at the \emph{compute-optimal frontier}: for each $C$, $N$ and $D$ are allocated optimally and the predicted loss is computed.
Whereas we normally compute the FLOPs as $C = 3 \times F_\text{fwd} \times D$, where $F_\text{fwd}$ is the architecture-specific forward-pass FLOPs per token (see \Cref{sec:flop-computation}), in \Cref{fig:scaling-law-fit-log-opt}(a) we use the simplification $C = 6ND$ for all architectures.
In panel~(b), thin curves connect checkpoints of different model sizes at the same Chinchilla multiple, and in panel~(c), they connect checkpoints of the same size at different Chinchilla factors.
In both cases, the bold curve is evaluated at the largest Chinchilla multiple ($D/N = 160$).

\paragraph{Construction of \Cref{fig:combined-savings}.}
Both panels are derived from the fitted scaling laws by analytically inverting the loss function.
Panel~(a) solves $L(N, D) = L_\text{target}$ for $D$ as a function of $N$:
\begin{equation*}
  D(N) = \left(\frac{B}{L_\text{target} - E - A/N^\alpha}\right)^{1/\beta},
\end{equation*}
where $L_\text{target}$ is set to the minimum observed training loss.
This gives the projected number of training tokens each architecture needs to reach that target as a function of model size.
Panel~(b) plots the savings factor, defined as the ratio of the transformer's token requirement to each other architecture's token requirement at matched model size, i.e., $D_\text{Transformer}(N) / D_\text{arch}(N)$.

\paragraph{Construction of \Cref{tab:compute-equivalent-loss}.}
Here, we again use the simplification $C = 6ND$.
For each compute budget $C = 6ND$, the compute-optimal allocation is derived from the first-order condition $\alpha A / N^\alpha = \beta B / D^\beta$, which gives
\begin{equation*}
  N^* = \left(\frac{C}{6G}\right)^{a_\text{opt}}, \quad D^* = \frac{C}{6 N^*},
\end{equation*}
where $G = (\beta B / \alpha A)^{1/\beta}$ and $a_\text{opt} = \beta / (\alpha + \beta)$ is each architecture's fitted compute-optimal exponent.
The predicted loss $L(N^*, D^*)$ and the difference $\Delta$ relative to the transformer are reported with 95\% bootstrap confidence intervals.

\subsection{Architectural Details} \label{sec:arch-details}

All models in our scaling ladder and ablation experiments share the same base Olmo 3 hyperparameters at each size (see \Cref{tab:ablation-configs}): model dimension~$d$, number of attention heads~$h$, number of layers~$l$, and vocabulary size $V = 100{,}352$ (the Dolma~2 tokenizer, padded for efficient embedding lookups).
We provide some additional details on the architecture below.

\paragraph{Olmo 3 (Transformer Baseline).}
Each Olmo 3 layer consists of a multi-head self-attention sub-layer followed by a SwiGLU MLP \citep{shazeer2020glu}, with RMSNorm \citep{zhang2019rmsnorm} applied before each sub-layer (pre-norm).
Attention uses rotary position embeddings (RoPE; \citealp{jianlin2024rope}) and QK-norm.
No grouped-query attention is used: the number of key--value heads equals the number of query heads.
The MLP hidden dimension is $\lfloor \frac{3}{2} \cdot \frac{8d}{3} \rceil_{256}$, where $\lceil \cdot \rceil_{256}$ rounds up to the nearest multiple of 256.
The embedding and language-model head matrices are untied.

\paragraph{Hybrid (GDN--Transformer).}
Hybrid GDN models use the same $d$, $h$, and $l$ as Olmo 3 but replace a fraction of the attention sub-layers with GDN sub-layers \citep{yang2025gdn}.
In our default configuration, every $r$-th layer is an attention layer and the remaining layers are GDN layers, where $r$ is the \emph{transformer ratio} (default $r{=}4$, i.e., a 3:1 linear layer to attention ratio).
We additionally enforce the final layer be an attention layer; if it is not already selected by the every-$r$th rule, it is added.

Each GDN sub-layer computes query, key, and value projections with head dimension $h_{\text{GDN}} = \lceil 0.75 \cdot d / h \rceil_{128}$, yielding a key dimension of $k = h \cdot h_{\text{GDN}}$ and value dimension $v = h \cdot 2 h_{\text{GDN}}$ (i.e., the value dimension is expanded by a factor of~2).
The GDN sub-layer further includes two scalar per-head parameters ($A_{\log}$ and $\text{dt}_\text{bias}$), short depthwise convolutions (kernel size~4) over the $q$, $k$, and $v$ streams, a gate projection, and an output projection.
Each GDN layer retains the same SwiGLU MLP and two RMSNorms as in the Olmo 3 block.

\paragraph{Pure GDN.}
Pure GDN models replace \emph{all} attention layers with GDN layers (i.e., $r{=}0$) and do not force a final attention layer.
All other hyperparameters are identical to the hybrid variant.

\paragraph{Hybrid Mamba2 (Mamba2--Transformer).}
Hybrid Mamba2 models follow the same layer interleaving scheme as Hybrid GDN ($r{=}4$ by default with a forced final attention layer), but substitute Mamba2 \citep{dao2024mamba2} for the non-attention layers.
The Mamba2 sub-layer uses an expansion factor of~2 (intermediate size $= 2d$), state size $n = 128$, $n_{\mathrm{groups}} = 1$, and a depthwise convolution with kernel size~4.
In our hybrid Mamba2 configuration, the Mamba2 layers retain the full MLP from the attention blocks.

\begin{table*}[t]
\centering
\caption{
    \textbf{Architecture configurations and non-embedding parameter counts (millions).}
    $d$: model dimension, $h$: number of attention heads, $l$: number of layers.
    Architectures at the same nominal scale can differ substantially in parameter count due to differences in layer composition.
}
\label{tab:ablation-configs}
\small
\begin{tabular}{@{}l c r r r r r r r@{}}
\toprule
\textbf{Architecture} & \textbf{Attn \%} & \textbf{60M} & \textbf{100M} & \textbf{190M} & \textbf{370M} & \textbf{600M} & \textbf{760M} & \textbf{1B} \\
\midrule
$d$ &  & 384 & 512 & 768 & 1024 & 1280 & 1536 & 2048 \\
$h$ &  & 8 & 8 & 12 & 16 & 16 & 16 & 16 \\
$l$ &  & 8 & 12 & 12 & 16 & 16 & 16 & 16 \\
\midrule
\multicolumn{9}{l}{\textit{Pure Architectures}} \\
\quad Transformer & 100\% & 57 & 102 & 190 & 371 & 548 & 758 & 1279 \\
\quad GDN & 0\% & 78 & 140 & 276 & 574 & 780 & 1011 & 1549 \\
\quad Mamba2 & 0\% & 61 & 110 & 207 & 410 & 606 & 841 & 1423 \\
\midrule
\multicolumn{9}{l}{\textit{Hybrid: Interleaved Attention}} \\
\quad GDN (1:1) & 50\% & 68 & 121 & 233 & 472 & 664 & 885 & 1414 \\
\quad GDN (3:1)$^\bigstar$ & 25\% & 73 & 130 & 254 & 523 & 722 & 948 & 1482 \\
\quad GDN (7:1) & 12.5\% & 75 & 133 & 262 & 548 & 751 & 980 & 1516 \\
\quad Mamba2 (3:1) & 25\% & 60 & 108 & 203 & 400 & 592 & 820 & 1387 \\
\midrule
\multicolumn{9}{l}{\textit{Hybrid GDN (3:1): Middle Placement}} \\
\quad Interleaved (reference)$^\bigstar$ & 25\% & 73 & 130 & 254 & 523 & 722 & 948 & 1482 \\
\quad Middle & 25\% & 70 & 127 & 247 & 510 & 707 & 932 & 1465 \\
\midrule
\multicolumn{9}{l}{\textit{Hybrid GDN (3:1): Gate/Eigenvalue Ablations}} \\
\quad Neg EV, gate$^\bigstar$ & 25\% & 73 & 130 & 254 & 523 & 722 & 948 & 1482 \\
\quad Neg EV, no gate & 25\% & 73 & 130 & 254 & 523 & 722 & 948 & 1482 \\
\quad Pos EV, gate & 25\% & 73 & 130 & 254 & 523 & 722 & 948 & 1482 \\
\quad Pos EV, no gate & 25\% & 73 & 130 & 254 & 523 & 722 & 948 & 1482 \\
\midrule
\multicolumn{9}{l}{\textit{Pure GDN: Gate/Eigenvalue Ablations}} \\
\quad Neg EV, gate & 0\% & 78 & 140 & 276 & 574 & 780 & 1011 & 1549 \\
\quad Pos EV, gate & 0\% & 78 & 140 & 276 & 574 & 780 & 1011 & 1549 \\
\quad Neg EV, no gate & 0\% & 78 & 140 & 276 & 574 & 780 & 1011 & 1549 \\
\quad Pos EV, no gate & 0\% & 78 & 140 & 276 & 574 & 780 & 1011 & 1549 \\
\bottomrule
\end{tabular}
\end{table*}

\subsection{Parameter Count and FLOP Computations} \label{sec:flop-computation}

We report total (non-embedding) parameter counts and forward-pass floating-point operations (FLOPs) per token for each model.
Both quantities are computed analytically from the architecture specification; the formulas are detailed below.

\paragraph{Parameter Counting.}
The total parameter count is the sum of embedding, language-model head, per-layer, and final layer-norm parameters:
\begin{equation}
    N_{\text{total}} = \underbrace{V \cdot d}_{\text{embed}} + \underbrace{V \cdot d}_{\text{LM head}} + \sum_{\ell=1}^{l} N_{\text{layer}}^{(\ell)} + \underbrace{d}_{\text{final norm}},
\end{equation}
where $N_{\text{layer}}^{(\ell)}$ depends on the layer type.
Non-embedding parameters are $N_{\text{non-emb}} = N_{\text{total}} - V \cdot d$.

For an \textbf{attention layer}, the sub-layer parameter count is:
\begin{equation}
    N_{\text{attn}} = \underbrace{4d^2}_{\text{Q,K,V,O projections}} + \underbrace{2d}_{\text{QK-norm}} + \underbrace{3d \cdot d_{\text{MLP}}}_{\text{SwiGLU MLP}} + \underbrace{2d}_{\text{2 RMSNorms}},
\end{equation}
where $d_{\text{MLP}} = \lceil \frac{3}{2} \cdot \frac{8d}{3} \rceil_{256}$ is the MLP hidden size.

For a \textbf{GDN layer}, letting $k = h \cdot h_{\text{GDN}}$ and $v = 2k$ denote the total key and value dimensions:
\begin{equation}
    N_{\text{GDN}} = \underbrace{d(2k + v + 2h)}_{\substack{\text{q,k,v,a,b} \\ \text{projections}}} + \underbrace{2h}_{\text{$A_{\log}$, dt\textsubscript{bias}}} + \underbrace{(2k + v) \cdot 4}_{\text{short convolutions}} + \underbrace{dv}_{\substack{\text{gate} \\ \text{projection}}} + \underbrace{v}_{\text{norm}} + \underbrace{vd}_{\substack{\text{output} \\ \text{projection}}} + \underbrace{3d \cdot d_{\text{MLP}}}_{\text{MLP}} + \underbrace{2d}_{\text{norms}}.
\end{equation}

For a \textbf{Mamba2 layer}, letting $e = 2d$ (intermediate size), $c = e + 2 n_{\mathrm{groups}} \cdot n$ (convolution dimension), and $p = e + c + h$ (projection size):
\begin{equation}
    N_{\text{mamba2}} = \underbrace{d \cdot p + p}_{\text{in projection}} + \underbrace{c \cdot 4}_{\text{conv1d}} + \underbrace{3h}_{\text{dt,A,D}} + \underbrace{e}_{\text{norm}} + \underbrace{e \cdot d + d}_{\text{out projection}} + \underbrace{3d \cdot d_{\text{MLP}}}_{\text{SwiGLU MLP}} + \underbrace{2d}_{\text{layer norms}}.
\end{equation}

\subsubsection{FLOP Counting}
We estimate forward-pass FLOPs per token using the standard convention $\text{FLOPs} = 2 \times \text{MACs}$ (multiply-accumulate operations).
The total FLOPs per token are:
\begin{equation}
    F = 2 \left( \underbrace{d \cdot V}_{\text{LM head}} + \sum_{\ell=1}^{l} M_{\text{layer}}^{(\ell)} \right),
\end{equation}
where $M_{\text{layer}}^{(\ell)}$ is the per-token MAC count for layer $\ell$.
Embedding lookups are excluded as they involve no arithmetic.
For \textbf{attention layers}, the MACs per token include both the projections and the sequence-length-dependent attention computation. For causal attention, we use the average context length $s_{\text{eff}} = s/2$:
\begin{equation}
    M_{\text{attn}} = \underbrace{4d^2}_{\text{Q,K,V,O}} + \underbrace{h \cdot \frac{d}{h} \cdot s_\text{eff}}_{\text{scores } (\text{QK}^{\top})} + \underbrace{h \cdot \frac{d}{h} \cdot s_\text{eff}}_{\text{output } (\text{attn} \cdot V)} + \underbrace{\frac{5}{2} \cdot h \cdot s_\text{eff}}_{\text{softmax}} + \underbrace{3d \cdot d_{\text{MLP}}}_{\text{MLP}} = 4d^2 + ds + 3d \cdot d_{\text{MLP}},
\end{equation}
where $s$ is the sequence length.
The factor $\frac{5}{2}$ in front of the softmax contribution accounts for the $5 s_\text{eff}$ FLOPs usually attributed to softmax normalization \citep{beck2026xlstm} divided by $2$ since the expression computes MACs.

For the linear recurrent layers (GDN and Mamba2), we first present per-token MACs in the recurrent formulation, and then describe the chunkwise parallel formulation used during training.

\paragraph{Recurrent Formulation.}
For \textbf{GDN layers}, the per-token MACs include projections, convolutions, the recurrence, and the MLP.
The delta rule requires three state interactions (retention projection, state update, output projection):
\begin{equation}
    M_{\text{GDN}} = \underbrace{d(2k + v + 2h)}_{\text{linear projs}} + \underbrace{4(2k + v)}_{\text{depthwise convs}} + \underbrace{dv}_{\text{gate proj}} + \underbrace{vd}_{\text{out proj}} + \underbrace{3 \cdot h \cdot h_{\text{GDN}} \cdot 2h_{\text{GDN}}}_{\substack{\text{recurrence} \\ (Sk, uv^T, Sq)}} + \underbrace{3d \cdot d_{\text{MLP}}}_{\text{MLP}}.
\end{equation}

For \textbf{Mamba2 layers}, the recurrence involves two state interactions (state update and output projection):
\begin{equation}
    M_{\text{mamba2}} = \underbrace{d \cdot p}_{\text{in proj}} + \underbrace{4c}_{\text{conv1d}} + \underbrace{e \cdot d}_{\text{out proj}} + \underbrace{2 \cdot h \cdot \frac{e}{h} \cdot n}_{\substack{\text{SSM recurrence} \\ (vk^T, Sq)}} + \underbrace{3d \cdot d_{\text{MLP}}}_{\text{MLP}}.
\end{equation}

\paragraph{Parallel (Chunkwise) Formulation.}
While the recurrent formulation above describes inference-time FLOPs, training utilizes hardware-efficient chunkwise parallel algorithms.
For Gated DeltaNet, we utilize the \textbf{Extended WY Representation} algorithm \citep{yang2025gdn}, which folds the data-dependent decay terms directly into the update matrices, avoiding the need for log-space computations or sub-tiling required by previous gated linear attention methods \citep{yang2024gatedlinearattention}.
Training FLOPs account for the overhead of intra-chunk local attention and inter-chunk state passing.

For \textbf{GDN layers}, the Extended WY algorithm introduces overheads proportional to the chunk size $L_{\text{chunk}}$ (set to 256) for constructing the kernel, update matrices ($W, U$), and computing local attention.
\begin{equation}
    M_{\text{GDN,train}} = \underbrace{M_{\text{proj,conv,MLP}}}_{\text{sequence-independent}} + \underbrace{L_{\text{chunk}}(3k + 2v)}_{\substack{\text{intra-chunk overhead} \\ (\text{Kernel, } W/U, \text{ Attn})}} + \underbrace{3 \cdot \frac{kv}{h}}_{\substack{\text{inter-chunk} \\ \text{state passing}}},
\end{equation}
where $M_{\text{proj,conv,MLP}}$ represents the cost of projections, convolutions, and MLPs (identical to the recurrent formulation). The intra-chunk coefficient $(3k+2v)$ accounts for the construction of $KK^\top$, $W$, $U$, $QK^\top$, and the final output computation. The inter-chunk cost ($3kv/h$) represents the three matrix multiplications required to propagate the hidden state ($QS^\top$, $U^\top K$, $WS^\top$).

For \textbf{Mamba2 layers}, utilizing the sequential chunkwise algorithm (instead of parallel scan) reduces the state-passing overhead to simple unidirectional updates:
\begin{equation}
    M_{\text{mamba2,train}} = \underbrace{M_{\text{proj,conv,MLP}}}_{\text{sequence-independent}} + \underbrace{2 \cdot L_{\text{chunk}} \cdot e}_{\substack{\text{intra-chunk} \\ \text{SSD mixing}}} + \underbrace{2 \cdot e \cdot n}_{\substack{\text{inter-chunk} \\ \text{state passing}}},
\end{equation}
where $e$ is the intermediate size and $n$ is the state size.
The inter-chunk cost accounts for the unidirectional state passing (state update $S + V K^\top$ and output computation $Q S^\top$) required by the sequential algorithm.

For the scaling law fits, we use the analytically computed training FLOPs per token $F$ (using the parallel formulation) to obtain the compute estimate $C = 3FD$ (accounting for forward and backward passes).

\section{Proofs: Increased Expressive Power Improves Scaling} \label{sec:scaling-laws-proofs}

In the main text, we have theoretically explained how increasing an architecture's expressivity can lead to better scaling on a language modeling objective, working in the idealized setup of the quantization model for neural scaling laws.
We now clarify details of the quantization model and provide deferred proofs.

\citet{michaud2023quantization} introduce the quantization model of neural scaling laws, which we will augment to take into account expressivity limitations of different architectures.
First, we fully define the quantization model, starting with the notion of a learnable task:

\begin{definition}[Learnable Task] \label{def:learnable-task}
    Task $k$ is learnable with $D$ tokens if $D \geq \frac{1}{p_k} T_k$, i.e., we can expect to observe enough relevant tokens leveraging task $k$ within the $D$ tokens.
\end{definition}

In the quantization model, an idealized LM learner proceeds by allocating parameters to learn learnable tasks, ordered by their frequency:

\begin{definition}[Quantization Model] \label{def:quantization-model}
    The learner is given $N$ parameters and $D$ samples.
    It proceeds by ranking all tasks $k$ learnable with $D$ samples by their frequency $p_k$ in the data, greedily spending $C_k$ samples to learn task $k$ in this order until the parameter budget $N$ is exceeded.
\end{definition}

Let $X_k$ be a random variable representing the loss achieved on task $k$. We will analyze $\trueloss$, the expected loss across all tasks, which can be defined as
\begin{equation*}
    \trueloss = \mathbb E \left[ \sum_{k=1}^\infty X_k p_k \right ] .
\end{equation*}

We incorporate expressivity into the quantization model via the following assumptions, restated from the main text:

\assmExpressibility*

\assmLossReduction*

\assmResources*

It follows from these assumptions that, for an unlearned task, $X_k = b$, but, for a learned task,
\begin{equation*}
    X_k =
    \begin{cases}
        a & \textrm{w.p.~$1 - \epsilon$} \\
        a' & \textrm{otherwise.}
    \end{cases}
\end{equation*}
In contrast, for any task $k$, we have the following sample complexity and parameter requirements:
\begin{align*}
    T_k &=
    \begin{cases}
        T & \textrm{w.p.~$1 - \epsilon$} \\
        T' & \textrm{otherwise.}
    \end{cases} \\
    C_k &=
    \begin{cases}
        C & \textrm{w.p.~$1 - \epsilon$} \\
        C' & \textrm{otherwise.}
    \end{cases}
\end{align*}
We will assume $T' \geq T$, that is, inexpressible tasks require at least as many samples to approximate as expressible tasks.
It is natural to imagine $C' \geq C$, though we may also consider the case where $C' = 0$ and $\Delta' = 0$, i.e., inexpressible tasks are completely ignored by the learner.

\subsection{Scaling Law with Tasks Learned}

We first slightly generalize the original scaling law from \citet{michaud2023quantization} in the following way:

\begin{lemma} \label{lem:task-scaling}
    For every learned task $k$, let $X_k$ be an independent random variable with mean $a$.
    Similarly, for every unlearned task $k$, let $X_k$ be an independent random variable with mean $b$.
    Then the expected loss $\trueloss_n$ after learning $n$ tasks is
    \begin{equation*}
        \trueloss_n
        = a + \frac{b - a}{\zeta(\alpha+1)} \sum_{k=n+1}^\infty k^{-(\alpha+1)} .
    \end{equation*}
    Further, the following $L_n$ closely approximates $\tilde L_n$:
    \begin{equation*}
        L_n
        = a + \frac{b - a}{\alpha\, \zeta(\alpha+1)} n^{-\alpha} ,
    \end{equation*}
    in the sense that $L_{n+1} \leq \trueloss_n \leq L_n$, and thus the true loss $\trueloss_n$ rapidly converges to the power law $L_n$.
\end{lemma}

\begin{proof}
    The expression for the first part, $\trueloss_n$, is a straightforward extension to the original derivation from \citet{michaud2023quantization}.
    For the approximation, we observe that the terms in the summation $\sum_{n+1}^\infty n^{-(\alpha+1)}$ represent a step function, which is naturally lower and upper bounded by the corresponding continuous ``envelopes'' of the step function as follows:
    \begin{equation*}
        \int_{n+1}^\infty k^{-(\alpha+1)} \, dk \ \leq \ \sum_{k=n+1}^\infty k^{-(\alpha+1)} \ \leq \ \int_n^\infty k^{-(\alpha+1)} \, dk
    \end{equation*}
    which simplifies to
    \begin{equation*}
        \frac{(n+1)^{-\alpha}}{\alpha} \ \leq \ \sum_{k=n+1}^\infty k^{-(\alpha+1)} \ \leq \ \frac{n^{-\alpha}}{\alpha} .
    \end{equation*}
    It follows that $\trueloss_n$ is sandwiched between $L_{n+1}$ and $L_n$, and rapidly approaches $L_n$ in the following sense: \AS{@will: added this for concreteness; have a look when you get a chance; could revisit making the claim in the lemma statement of relative error being only $O(1/n)$.}
    \begin{align*}
        \frac{L_n - \trueloss_n}{\trueloss_n} & = \frac{L_n}{\trueloss_n} - 1 \leq \frac{L_n}{L_{n+1}} - 1 \\\
        & = \frac{a + \frac{b - a}{\alpha\, \zeta(\alpha+1)}\, n^{-\alpha}}{a + \frac{b - a}{\alpha\, \zeta(\alpha+1)}\, (n+1)^{-\alpha}} \, - \, 1 \\
        & \leq \frac{n^{-\alpha}}{(n+1)^{-\alpha}} - 1 \\
        & = \frac{(n+1)^\alpha - n^\alpha}{n^\alpha} \\
        & = \frac{\alpha n^{\alpha-1} + \frac{\alpha (\alpha-1)}{2!} n^{\alpha-2} + \ldots}{n^\alpha} & \text{by Taylor series expansion}\\
        & = O(1/n)
    \end{align*}
    Thus, the relative error $(L_n - \trueloss_n)/\trueloss_n$ is only $O(1/n)$.
\end{proof}

In the upcoming results, we will follow \citet{michaud2023quantization} in using the approximation $L_n$ from \Cref{lem:task-scaling}.
Having established a general result about analyzing the quantization model, we consider its instantiation in the expressivity-aware case:

\begin{lemcorollary} \label{lem:task-scaling-expressivity}
    Adopt \Cref{assm:expressible,assm:loss-incurred}, i.e., tasks are expressible w.p.\ $1-\epsilon$, achieving loss $L_0 - \Delta$ when learned, or inexpressible w.p.\ $\epsilon$, achieving loss $L_0 - \Delta'$ when learned.
    Then the expected loss $\trueloss^\epsilon_n$ after learning $n$ tasks is closely approximated by $L^\epsilon_n$ where:
    \begin{equation*}
        L^\epsilon_n = L^\epsilon_\infty + \frac{L_0 - L^\epsilon_\infty}{\alpha\, \zeta(\alpha+1)} \cdot n^{-\alpha} ,
    \end{equation*}
    and $L^\epsilon_\infty = L_0 - (1 - \epsilon) \Delta - \epsilon \Delta'$.
\end{lemcorollary}

The irreducible loss term $L^\epsilon_\infty$ will be the same in this scaling law as in the upcoming scaling laws for parameters and tokens.
Quite naturally, it interpolates between the reduced loss of expressible and inexpressible tasks.
Assuming $\Delta' < \Delta$, then, for any $\tilde{\epsilon} > \epsilon$, we will have $L^{\tilde{\epsilon}}_\infty > L^\epsilon_\infty$.
Furthermore, the following shows that this reduction in irreducible loss translates to a reduction in loss everywhere along the loss curve:

\begin{lemma} \label{lem:irreducible-loss-lowers-loss-everywhere}
    Fix $f, f' : \mathbb N \to \mathbb (0, 1)$ with $f'(n) \geq f(n)$ for all $n$. Define
    \begin{align*}
        L_n &= a + (b - a) f(n) \\
        L'_n &= a' + (b - a') f'(n) .
    \end{align*}
    Then, if $a' > a$, it holds that $L'_n > L_n$ for all $n$.
\end{lemma}

\begin{proof}
    Define $\Delta L_n = L'_n - L_n$. We show that, if $a' > a$, it holds that $\Delta L_n > 0$ for all $n$. By definition,
    \begin{align*}
        \Delta L_n
        &= a' - a + (b - a') f'(n) - (b - a) f(n) \\
        &\geq a' - a + (b - a') f(n) - (b - a) f(n) \\
        &= (a' - a) \left( 1 - f(n) \right) .
    \end{align*}
    Since $f(n) \in (0, 1)$, we have $\Delta L_n > 0$ as long as $a' > a$.
\end{proof}

Combining \Cref{lem:task-scaling-expressivity} and \Cref{lem:irreducible-loss-lowers-loss-everywhere}, the fact that we have $L^{\tilde{\epsilon}}_\infty > L^\epsilon_\infty$ implies that $L^{\tilde \epsilon}_n > L^\epsilon_n$ for any $n$.

\subsection{Parameter Scaling Law}

Define $\tilde L(N) = \mathbb E \left[ L_{n_N} \right]$, where $n_N$ is the expected number of tasks we can learn with $N$ parameters.

\begin{lemma}
    Adopt the assumptions of \Cref{lem:task-scaling}.
    Further, assume the parameter cost of task $k$ is a random variable $C_k$, independent across $k$ with mean $c$.
    Then, as a function of the number of parameters $N$, the loss $\trueloss(N)$ in the quantization model is closely approximated by a power law, i.e., $\trueloss(N) \approx L(N)$ where:
    \begin{equation*}
        L(N) = a + c^\alpha \cdot \frac{b - a}{\alpha\, \zeta(\alpha+1)} \cdot N^{-\alpha} .
    \end{equation*}
\end{lemma}

\begin{proof}
    The expected number of tasks $n_N$ we can learn with $N$ parameters satisfies
    \begin{equation*}
        N \geq \sum_{k=1}^{n_N} C_k
          = n_N \cdot c ,
    \end{equation*}
    which implies $n_N = \floor{ N / c } \approx N / c$.
    Plugging this into the expression for $L_n$ from \Cref{lem:task-scaling} yields the following approximation for the expected loss:
    \begin{align*}
        L(N)
        &= a + \frac{b - a}{\alpha\, \zeta(\alpha+1)} \left( N / c \right)^{-\alpha} \\
        &= a + c^\alpha \cdot \frac{b - a}{\alpha\, \zeta(\alpha+1)} \cdot N^{-\alpha} . \qedhere
    \end{align*}
\end{proof}

We now consider the form of $L(N)$ under the assumptions about expressivity awareness.
By \Cref{assm:expressible,assm:less-succint}, we have $C_n = C$ with probability $1 - \epsilon$ and $C_n = C'$ with probability $\epsilon$.
Thus, we obtain:

\begin{lemcorollary} \label{lem:param-scaling-expressivity}
    Adopt \Cref{assm:expressible,assm:loss-incurred,assm:less-succint}, i.e., tasks are expressible w.p.\ $1-\epsilon$, requiring $C$ parameters or inexpressible w.p.\ $\epsilon$, requiring $C'$ parameters.
    Then, as a function of the number of parameters $N$, the loss $\trueloss(N)$ in the quantization model is closely approximated by a power law, i.e., $\trueloss(N) \approx L(N)$ where:
    \begin{equation*}
        L(N) = L^\epsilon_\infty + A_\epsilon \cdot \frac{L_0 - L^\epsilon_\infty}{\alpha\, \zeta(\alpha+1)} \cdot N^{-\alpha} ,
    \end{equation*}
    where $L^\epsilon_\infty$ and $L_0$ are defined as in \Cref{lem:task-scaling-expressivity} and
    \begin{equation*}
        A_\epsilon = (C + \epsilon(C' - C))^\alpha .
    \end{equation*}
\end{lemcorollary}

The cost $C'$ does not change the irreducible loss, though it does impact the slope of the scaling trend.
In particular, with $C' = C$, we obtain the coefficient $c^\alpha = C^\alpha$, and with $C' = 0$, we recover the coefficient $c^\alpha = (1 - \epsilon)^\alpha C^\alpha$.
Further, under any nontrivial instantiation of the model (i.e., where inexpressible tasks either reduce loss less or cost more), increasing expressivity will decrease $L(N)$ for all $N$:

\begin{lemcorollary} \label{lem:param-monotonic}
    Assume $C' \geq C$.
    If $\Delta' < \Delta$ or $C' > C$, then
    $L(N)$ strictly decreases as $\epsilon$ decreases, elementwise for all $N$.
\end{lemcorollary}

\begin{proof}
    If $\Delta' = \Delta$ but $C' > C$, then $L^\epsilon_\infty = L_0 - \Delta$ independent of $\epsilon$, but $A_\epsilon$ is increasing with $\epsilon$.
    Thus, $L(N)$ will increase with $\epsilon$ for all $N$.

    If we have $\Delta' < \Delta$, then $L^\epsilon_\infty$ is increasing with $\epsilon$.
    It also holds that, for $\epsilon' > \epsilon$, $A_{\epsilon'} \geq A_\epsilon$.
    We apply \Cref{lem:irreducible-loss-lowers-loss-everywhere} to conclude that $L(N)$ is increasing with $\epsilon$ elementwise for all $N$.
\end{proof}

\subsection{Data Scaling Law} \label{sec:data-scaling}

We first generalize \Cref{lem:task-scaling} in the following way:

\begin{lemma} \label{lem:task-scaling-phases}
    Let $n_0, \ldots, n_m$. be a finite sequence of phases, with $n_0 = 0$ and $n_m = \infty$.
    For each $1 \leq \ell \leq m$, assume that, for all $k$ such that $n_{\ell - 1} < k \leq n_\ell$, all $X_k$ are independent with mean $\mu_{\ell - 1}$.
    Then the expected loss $\trueloss$ across all tasks is closely approximated by a power law, i.e., $\trueloss \approx L$ where:
    \begin{equation*}
        L = \mu_0 + \frac{1}{\alpha\, \zeta(\alpha+1)} \sum_{\ell = 1}^{m-1} (\mu_\ell - \mu_{\ell - 1}) \cdot n_\ell^{-\alpha} .
    \end{equation*}
\end{lemma}

\begin{proof}
    We can first write the loss as a weighted sum over losses coming from different phases:
    \begin{align*}
        \trueloss
        &= \mathbb E \left[ \sum_{k=1}^\infty X_k p_k \right] \\
        &= \mathbb E \left[ \sum_{\ell = 1}^m\sum_{k=n_{\ell - 1} + 1}^{n_\ell} X_k p_k \right] \\
        &= \sum_{\ell = 1}^m \mu_{\ell - 1} \sum_{k=n_{\ell - 1} + 1}^{n_\ell} p_k .
    \end{align*}
    At this point, for each series $\ell < m$, we add the missing terms $\sum_{k = n_{\ell} + 1}^\infty p_k$ to the series and subtract it from the remaining series.
    This yields:
    \begin{align*}
        \trueloss &= \sum_{\ell = 1}^m \mu_{\ell - 1} \left( \sum_{k=n_{\ell - 1} + 1}^\infty p_k - \sum_{k=n_\ell + 1}^\infty p_k \right) \\
        &= \left( \sum_{\ell = 1}^m \mu_{\ell - 1} \sum_{k=n_{\ell - 1} + 1}^\infty p_k \right) - \left( \sum_{\ell = 1}^m \mu_{\ell - 1} \sum_{k=n_\ell + 1}^\infty p_k \right) \\
        &= \left( \sum_{\ell = 1}^m \mu_{\ell - 1} \sum_{k=n_{\ell - 1} + 1}^\infty p_k \right) - \left( \sum_{\ell = 2}^{m+1} \mu_{\ell - 2} \sum_{k=n_{\ell-1} + 1}^\infty p_k \right) \\
        &= \mu_0 \sum_{k=n_0 + 1}^\infty p_k + \left( \sum_{\ell = 2}^m (\mu_{\ell - 1} - \mu_{\ell - 2}) \sum_{k=n_{\ell - 1} + 1}^\infty p_k \right) - \mu_{m-1 } \sum_{k={n_m + 1}}^\infty p_k \\
        &= \mu_0 \sum_{k=1}^\infty p_k + \left( \sum_{\ell = 2}^m (\mu_{\ell - 1} - \mu_{\ell - 2}) \sum_{k=n_{\ell - 1} + 1}^\infty p_k \right) - \mu_{m-1 } \sum_{k=\infty}^\infty p_k \\
        &= \mu_0  + \sum_{\ell = 2}^m (\mu_{\ell - 1} - \mu_{\ell - 2}) \sum_{k=n_{\ell - 1} + 1}^\infty p_k .
    \end{align*}
    At this point, we can simplify the series in a similar way to \citet{michaud2023quantization} to conclude:
    \begin{align*}
        \trueloss \approx L
        &= \mu_0 + \sum_{\ell = 2}^m \frac{\mu_{\ell - 1} - \mu_{\ell - 2}}{\alpha\, \zeta(\alpha+1)} \cdot n_{\ell - 1}^{-\alpha} \\
        &= \mu_0 + \frac{1}{\alpha\, \zeta(\alpha+1)} \sum_{\ell = 1}^{m-1} (\mu_\ell - \mu_{\ell - 1}) \cdot n_\ell^{-\alpha} . \qedhere
    \end{align*}
\end{proof}

\begin{lemma} \label{lem:data-scaling-expressivity}
    Adopt \Cref{assm:expressible,assm:loss-incurred,assm:less-succint}, i.e., tasks are expressible w.p.\ $1-\epsilon$, requiring $T$ relevant tokens to learn, or inexpressible w.p.\ $\epsilon$, requiring $T'$ relevant tokens to learn.
    Then, as a function of the token budget $D$, the expected loss $\trueloss(D)$ under the quantization model is closely approximated by a power law, i.e., $\trueloss(D) \approx L(D)$ where
    \begin{equation*}
        L(D) = L^\epsilon_\infty + \frac{B_\epsilon}{\alpha\, \zeta(\alpha+1)^{1/(\alpha+1)}} \cdot D^{-\alpha/(\alpha + 1)} ,
    \end{equation*}
    $L^\epsilon_\infty$ is defined as in \Cref{lem:task-scaling-expressivity}, and
    \begin{equation*}
        B_\epsilon = (1 - \epsilon) \Delta T^{\alpha/(\alpha + 1)} + \epsilon \Delta' T'^{\alpha/(\alpha + 1)} .
    \end{equation*}
\end{lemma}

\begin{proof}
    We consider the tasks that are learned with $D$ tokens.
    By \Cref{def:learnable-task}, any expressible task $k$ learnable with $D$ tokens satisfies
    \begin{align*}
        D
        &\geq \frac{1}{p_k} T \\
        &= \zeta(\alpha+1) k^{\alpha + 1} T
    \end{align*}
    Let $k_\textrm{e}$ be the largest index representing a learnable expressible task. It holds that
    \begin{align*}
        k_\textrm{e}
        &= \floor{ \left( \frac{1}{\zeta(\alpha+1)} \cdot D / T \right)^{1/(\alpha + 1)} } \\
        &\approx \left( \frac{1}{\zeta(\alpha+1)} \cdot D / T \right)^{1/(\alpha + 1)} .
    \end{align*}
    By the same reasoning and \Cref{assm:expressible}, we have that the maximum index $k_\textrm{i} \leq k_\textrm{e}$ representing an inexpressible task learnable with $D$ tokens satisfies
    \begin{equation*}
        k_\textrm{i} \approx \left( \frac{1}{\zeta(\alpha+1)} \cdot D / T' \right)^{1/(\alpha + 1)} .
    \end{equation*}

    Now, we will define three phases of tasks, during each of which all task losses are independent and have a different mean:
    Thus:
    \begin{enumerate}
        \item Phase from $1 \leq k \leq k_\textrm{i}$: all tasks are learned, with expressible tasks obtaining loss $L_0 - \Delta$ and inexpressible tasks obtaining loss $L_0 - \Delta'$. Thus, the expected loss is $L^\epsilon_\infty = L_0 - (1 - \epsilon) \Delta - \epsilon \Delta'$.
        \item Phase from $k_\textrm{i} + 1 \leq k \leq k_\textrm{e}$: only expressible tasks are learned obtaining loss $\Delta$. Thus, the expected loss is $L_0 - (1 - \epsilon) \Delta$.
        \item Phase from $k_\textrm{e} + 1 \leq k < \infty$: no tasks are learned. Thus, the expected loss is $L_0$.
    \end{enumerate}

    Invoking \Cref{lem:task-scaling-phases} with these three phases, the loss $\trueloss(D)$ is closely approximated by the following:
    \begin{align*}
        & L^\epsilon_\infty + \frac{1}{\alpha\, \zeta(\alpha+1)} \left( (1 - \epsilon) \Delta k_\textrm{e}^{-\alpha} + \epsilon \Delta' k_\textrm{i}^{-\alpha} \right) \\
        &\approx L^\epsilon_\infty + \frac{ (1 - \epsilon) \Delta \left( D / T \right)^{-\alpha/(\alpha + 1)} + \epsilon \Delta' \left( D / T' \right)^{-\alpha/(\alpha + 1)} }{\alpha\, \zeta(\alpha+1)^{1/(\alpha+1)}} \\
        &= L^\epsilon_\infty +
        \frac{(1 - \epsilon) \Delta T^{\alpha/(\alpha + 1)} + \epsilon \Delta' T'^{\alpha/(\alpha + 1)}}{\alpha\, \zeta(\alpha+1)^{1/(\alpha+1)}} \cdot D^{-\alpha/(\alpha + 1)} \\
        &= L^\epsilon_\infty + \frac{B_\epsilon}{\alpha\, \zeta(\alpha+1)^{1/(\alpha+1)}} \cdot D^{-\alpha/(\alpha + 1)} .
    \end{align*}
    This final expression is $L(D)$, our desired approximation to the expected loss.
\end{proof}

The irreducible loss here is the same as that of the task scaling law.
The scaling coefficient changes from the expressivity-unaware case by a factor of $B_\epsilon / B_0$.
Moreover, we see that increasing expressivity improves the loss curve:

\begin{lemcorollary} \label{lem:tokens-monotonic}
    If $\Delta' < \Delta$, or $T' > T$, then $L(D)$ strictly decreases as $\epsilon$ decreases, elementwise for all $D$.
\end{lemcorollary}

\begin{proof}


    We will prove this by showing that the derivative of the loss w.r.t.\ $\epsilon$ is strictly positive.
    Let $Z = \alpha\, \zeta(\alpha+1)^{1/(\alpha+1)}$.
    Recall that the loss $L(D)$ is
    \begin{align*}
        L(D) &= L^\epsilon_\infty + \frac{B_\epsilon}{Z} \cdot D^{-\alpha/(\alpha + 1)} \\
        & = L^\epsilon_\infty + \frac{(1 - \epsilon) \Delta T^{\alpha/(\alpha + 1)} + \epsilon \Delta' T'^{\alpha/(\alpha + 1)}}{Z} \cdot D^{-\alpha/(\alpha + 1)} \\
        & = L_0 - (1-\epsilon) \Delta - \epsilon \Delta' + \frac{1}{Z} \cdot \left( (1 - \epsilon) \Delta \left(\frac{T}{D}\right)^{\alpha/(\alpha + 1)} + \epsilon \Delta' \left(\frac{T'}{D}\right)^{\alpha/(\alpha + 1)} \right) .
    \end{align*}
    The derivative of the loss w.r.t.\ $\epsilon$ is therefore
    \begin{align*}
        \frac{\partial}{\partial \epsilon} L(D) &= (\Delta - \Delta') + \frac{1}{Z} \cdot \left( - \Delta \left(\frac{T}{D}\right)^{\alpha/(\alpha + 1)} + \Delta' \left(\frac{T'}{D}\right)^{\alpha/(\alpha + 1)} \right) \\
        & \geq (\Delta - \Delta') + \frac{1}{Z} \cdot \left( - \Delta \left(\frac{T}{D}\right)^{\alpha/(\alpha + 1)} + \Delta' \left(\frac{T}{D}\right)^{\alpha/(\alpha + 1)} \right) & \text{since\ } T' \geq T \\
        & = (\Delta - \Delta') + \frac{1}{Z} \cdot \left( - (\Delta - \Delta') \left(\frac{T}{D}\right)^{\alpha/(\alpha + 1)} \right) \\
        & = (\Delta - \Delta') \cdot \left( 1 - \frac{1}{Z} \cdot \left(\frac{T}{D}\right)^{\alpha/(\alpha + 1)} \right) \\
        & \geq (\Delta - \Delta') \cdot \left( 1 - \frac{1}{Z} \right) & \text{since\ } D \geq T ,
    \end{align*}
    which is strictly positive as long as $\Delta' < \Delta$. It follows that, for $\Delta' < \Delta$, $L(D)$ is an increasing function of $\epsilon$, i.e., it increases (pointwise) as $\epsilon$ increases and expressivity decreases.

    When $\Delta' = \Delta$, the expression for the derivative simplifies to
    \begin{align*}
        \frac{\partial}{\partial \epsilon} L(D)
        &= \frac{1}{Z} \cdot \left( - \Delta \left(\frac{T}{D}\right)^{\alpha/(\alpha + 1)} + \Delta \left(\frac{T'}{D}\right)^{\alpha/(\alpha + 1)} \right) \\
        &= \frac{\Delta}{Z} \cdot \left( \left(\frac{T'}{D}\right)^{\alpha/(\alpha + 1)} -\left(\frac{T}{D}\right)^{\alpha/(\alpha + 1)} \right)
    \end{align*}
    which again is strictly positive as long as $T' > T$ and $\Delta > 0$, the latter of which is satisfied by construction.
\end{proof}

Overall, we see that in every non-trivial instantiation of the expressivity-aware quantization model, increasing expressivity improves $L(D)$, which closely approximates $\tilde L(D)$ for large enough $D$.

\subsection{Main Results} \label{sec:main-scaling-results}

Combining the results above on parameter (\Cref{lem:param-scaling-expressivity}) and data scaling (\Cref{lem:data-scaling-expressivity}), we obtain:

\scalingEfficiency*

Combining \Cref{lem:param-monotonic} and \Cref{lem:tokens-monotonic} yields:

\expressivityAlwaysImprovesScaling*

The following follows straightforwardly from assuming $\Delta' < \Delta$:

\irreducibleLoss*

\paragraph{Future Work.}
There are several extensions to this analysis worth pursuing.
First, we have assumed learnable tasks are prioritized in order of frequency, whereas a more optimal policy would order them by \emph{parameter efficiency} $p_n / C_n$.
It would also be interesting to consider joint parameter-data scaling $L(N, D)$ rather than scaling laws derived in isolation.
On the empirical side, one could test whether these scaling laws match real LMs in controlled settings where $\epsilon$ is known, or whether it is possible to isolate tasks in the data in practical settings.
All of these directions would close the gap between the idealized quantization model and practical training, and could inform choices about the interaction between data and architecture during pretraining.

\end{document}